\documentclass{article}

\PassOptionsToPackage{round,authoryear}{natbib}
\usepackage[preprint]{neurips_2025}

\usepackage[utf8]{inputenc}
\usepackage[T1]{fontenc}
\usepackage{amsmath,amssymb,amsthm,mathtools}
\usepackage{booktabs,array,graphicx}
\usepackage{mathptmx}
\usepackage{microtype}
\usepackage{xcolor}
\usepackage[colorlinks=true,linkcolor=blue,citecolor=blue,urlcolor=blue]{hyperref}
\hypersetup{%
  pdftitle={Optimal Recovery Meets Bayesian Learning: Where Worst-Case Bounds Pay Off},
  pdfauthor={Gordei Verbii}}

\newtheorem{theorem}{Theorem}[section]
\newtheorem{proposition}[theorem]{Proposition}
\newtheorem{lemma}[theorem]{Lemma}

\theoremstyle{remark}
\newtheorem{remark}[theorem]{Remark}

\newcommand{\R}{\mathbb{R}}
\newcommand{\Hi}{\mathcal{H}}
\newcommand{\Lam}{\Lambda}
\newcommand{\norm}[1]{\lVert #1\rVert}
\newcommand{\ip}[2]{\langle #1, #2\rangle}
\newcommand{\eps}{\varepsilon}
\newcommand{\E}{\mathbb{E}}
\DeclareMathOperator*{\argmin}{arg\,min}

\title{Optimal Recovery Meets Bayesian Learning:\\
Where Worst-Case Bounds Pay Off}

\author{%
  Gordei Verbii\\
  Independent Researcher\\
  \texttt{%
scigverbii@gmail.com%
}\\
}

\begin{document}

\maketitle

\begin{abstract}
Worst-case Optimal Recovery (OR) and Bayesian learning describe the same
Gaussian--quadratic--Hilbert problems in two vocabularies. We sharpen the
correspondence -- the radius of information equals a nugget-optimized GP
posterior variance and is \emph{attained} by the posterior mean at a
closed-form balance nugget -- and measure, inside three published Bayesian
systems, where the worst-case side pays. The ledger is two-sided: the losses
instruct as much as the wins. Morozov
calibration tracks a test-access oracle within $1.00$--$1.19\times$ where
$\sigma$-blind rules fail, is $4.9$--$6.3\times$ more reproducible across
noise draws ($p=0.002$--$0.004$), and is the only deployable rule whose
selection survives a change of backend ($1.36\times$ against
$12$--$30\times$ for the released weight, ML-II and GCV); tight certificates
cover at the information-theoretic floor with no numerical slack. But
on exchangeable data split-conformal beats the OR head on interval score, a
water-filling prior adds nothing without an oracle noise hint, and under
covariate shift the OR band keeps coverage on every dataset yet loses
interval score to split-conformal, and to a feature-free constant band, on
most cells; what pays is not shift but shift on a \emph{learnable} target,
which a training-free audit statistic predicts before any model is fitted. In
Bayesian optimization the certified width is a validity floor whose scalar
inflation we prove inert under a checkable margin condition and check at
every step. Inertness is graded, not binary, and in the size of the
inflation as much as in the objective:
$\kappa{=}2$ is inert wherever $\kappa{=}5$ is and on more cells besides,
while $\kappa{=}5$ moves half the Ackley seeds and every Griewank seed.
Exploration is a shape problem, not a scale one. The design rule: match
the guarantee tool to the data regime, and audit the regime first.
\end{abstract}

\section{Introduction}

Bayesian learning systems ship with uncertainty, but rarely with
\emph{guarantees}: credible bands miscalibrate, regularization weights are tuned
on clean data, and exploration bonuses inherit whatever the posterior believes.
Optimal Recovery
\citep{GolombWeinberger1959,MicchelliRivlin1977,MelkmanMicchelli1979,%
TWW1988,Donoho1994} offers the complementary currency -- exact worst-case
statements over explicit model sets -- and in Hilbert spaces it is not a rival
theory but the \emph{same} theory in minimax vocabulary
\citep{KimeldorfWahba1970,Wahba1990,Kanagawa2018}. This paper makes that
identification load-bearing and measures where the worst-case side pays inside
three recent open-source Bayesian systems.

\paragraph{Contributions.}
\begin{itemize}\setlength{\itemsep}{1pt}
\item \textbf{Radius $=$ nugget-optimized posterior variance, with attainment
(Thm.~\ref{thm:radius}).} The two-ellipsoid radius of information for a point
functional is $E_\star(x)^2=\min_{\nu>0}V_\nu(x)(\eps^2+\eta^2/\nu)$, and the GP
posterior mean at the balance nugget attains it, so the per-algorithm
certificate \emph{equals} the radius -- an implementation invariant we report as
what it measurably is, a five-significant-digit invariant and not an identity
(Sec.~\ref{sec:caseA}). Table~\ref{tab:novel} in App.~\ref{app:proofs}
separates classical from new.
\item \textbf{Conditional Chebyshev intervals, Occam degeneracy, conformal radii
(Thm.~\ref{thm:local}--Thm.~\ref{thm:conformal}).} The $y$-conditional band is
the total-ambiguity robust-Bayes credible interval -- elementary in hindsight
\citep{Berger1985,Vidakovic2000} but licensing a certified credible-band
replacement; at the Occam radius the consistent set is a single point (why
data-driven prior scales collapse); split-conformal inflation restores
finite-sample validity on exchangeable observations, at a price: on a proper
score that floor is a net loss in both tabular regimes, so we publish the
no-floor band beside the delivered one (Sec.~\ref{sec:caseB}).
\item \textbf{A three-sided empirical ledger on three published systems}
\citep{TGPS2026,DAK2025,ALAS2026}. \emph{Wins:} Morozov-calibrated weights track
a test-access oracle where $\sigma$-blind ML-II/GCV/L-curve fail, and are the
reproducible choice on both seed dispersion and portability across a change of
backend; tight certificates cover with probability one at the
information-theoretic floor (Sec.~\ref{sec:caseA}). \emph{Losses, with paired
statistics:} in-distribution split-conformal dominates the OR head on interval
score; the water-filling prior's advantage does not survive using the model's
own noise estimate; and under covariate shift the OR head wins coverage on every
dataset but loses the interval score to split-conformal and to a feature-free
constant band on most cells (Sec.~\ref{sec:caseB}). \emph{What separates them}
is not the shift but the target: a training-free random-forest audit statistic,
computed before any model is fitted.
\item \textbf{A directly tested exploration principle.} Prop.~\ref{prop:inv}
gives a checkable sufficient condition for $\kappa$-scaling the certified floor
to change nothing, tested against what the acquisition actually did at every BO
iteration of every arm and seed -- $9{,}018$ records, $0$ counterexamples -- and
markedly conservative, which we measure rather than hide. Inertness is graded,
not binary, and graded in the size of the inflation as well as in the objective:
exploration needs the width's geometry reshaped, not its magnitude rescaled
(Sec.~\ref{sec:caseC}).
\item \textbf{Reproducibility measured, not assumed.} At $\sigma=0$ the Case-A
solver is not run-to-run deterministic, which puts a floor under every
dispersion number in the paper, while Case-C reproduces bitwise across machines
on all $80$ shared trajectories; reporting both is what makes the Case-A
stability ordering believable and the Case-C bitwise-identity claims meaningful.
Tightening the bounds also exposed five defects invisible under loose ones: a
train/test feature inconsistency, an all-rows delivery rule that is a
\emph{validity} rather than a width failure, a silently substituted six-level
regression target, the Occam collapse and a $\kappa$-grid saturation mode.
\end{itemize}

\section{Related work}

Optimal recovery from inaccurate data descends from
\citet{GolombWeinberger1959,MicchelliRivlin1977,MelkmanMicchelli1979} and
information-based complexity \citep{TWW1988}; the two-hyperellipsoid radius,
optimal parameters and S-procedure exactness are
\citet{FoucartLiao2023,FoucartLiao2024,FoucartLiaoSproc} (see also
\citealp{PolikTerlaky2007}). Parameter choice for ill-posed problems is classical
-- Morozov \citep{Morozov1966,EnglHankeNeubauer1996}, GCV
\citep{GolubHeathWahba1979}, L-curve \citep{Hansen1992} -- and Case A is a
head-to-head among them inside a modern GP--PDE stack
\citep{Chen2021,TGPS2026}. Statistical counterparts: minimax
linear theory \citep{Donoho1994,DLM1990}, Pinsker--Osipenko filters
\citep{Pinsker1980,Osipenko2024}, GP--RKHS equivalences
\citep{KimeldorfWahba1970,Wahba1990,Kanagawa2018}, robust Bayes
\citep{Berger1985,Vidakovic2000}, optimal UQ \citep{Owhadi2013} and frequentist
coverage of Bayesian credible sets
\citep{KnapikVdVaart2011,SzaboVdVaart2015}. Conformal prediction
\citep{Vovk2005,Lei2018,GibbsCandes2021} supplies our strongest in-distribution
baseline and the wrapper we attach to the model radius; feature-free split
conformal around a constant predictor is a standard sanity check there
\citep{Lei2018}, and we report one throughout Case B because it changes the sign
of a conclusion. Numerical reproducibility across hardware and repeated runs is
studied in its own right; our Case-A measurements find the selection \emph{rule}
to be the dominant factor. Bandit widths:
\citet{Srinivas2010,ChowdhuryGopalan2017} schedule posterior widths with growing
$\beta_t$; Prop.~\ref{prop:inv} and Sec.~\ref{sec:caseC} give a checkable
condition under which scheduling the \emph{certified} width is provably vacuous,
and measure how conservative it is. Nonlinear OR and GP--PDE machinery:
\citet{LinDu2025,Chen2021}.

\section{Setting, notation, and the classical bridge}

$\Hi$ is a real separable Hilbert space -- an RKHS with kernel $k$ and point
representers $\varphi_x=k(x,\cdot)$, or a feature space $\R^d$ with
$\varphi_x=\phi(x)$. A bounded linear map $\Lam u=(\ip{\lambda_i}{u})_i$
produces $y=\Lam u+e$, $\norm e_2\le\eta$; $G_{ij}=\ip{\lambda_i}{\lambda_j}$,
$b(x)_i=\ip{\lambda_i}{\varphi_x}$, $k(x,x)=\norm{\varphi_x}^2$. For
$\ell(u)=\ip{c}{u}$, $c\in\Hi$ is its representer and
$\ell=\ip{\varphi_x}{\cdot}$ is point evaluation. The \emph{radius of
information} is $E_\star(\ell)=\inf_A\sup_{\norm u\le\eps,\norm e\le\eta}
|\ell(u)-A(\Lam u+e)|$; the \emph{data-consistent set} is
$C_y(\rho,\eta)=\{u:\norm u\le\rho,\norm{\Lam u-y}\le\eta\}$ with conditional
interval $I_\ell(y)=[\inf_{C_y}\ell,\sup_{C_y}\ell]$. The conjugate model
$u\sim\mathrm{GP}(0,s^2k)$, $y|u\sim\mathcal N(\Lam u,\sigma^2I)$ has
\begin{equation}\label{eq:gp}
m_\nu(x)=b(x)^\top(G+\nu I)^{-1}y,\quad
V_\nu(x)=k(x,x)-b(x)^\top(G+\nu I)^{-1}b(x),\quad \nu=\sigma^2/s^2,
\end{equation}
and we set $w_\mu(x)=(G+\mu^{-1}I)^{-1}b(x)$,
$\mathrm{pf}_\mu(x)=\norm{\varphi_x-\Lam^*w_\mu(x)}$,
$\mathrm{wn}_\mu(x)=\norm{w_\mu(x)}_2$, with the two-line identity
$V_{1/\mu}=\mathrm{pf}_\mu^2+\mu^{-1}\mathrm{wn}_\mu^2$.
The classical dictionary -- regularization $=$ posterior mean
\citep{KimeldorfWahba1970}, noiseless posterior sd $=$ power function
\citep{RW2006,Wahba1990,Wendland2005,Kanagawa2018}, minimax--Bayes duality
\citep{Pinsker1980,Donoho1994,Osipenko2024}, robust Bayes
\citep{Berger1985,Vidakovic2000}, posterior contraction
\citep{KnapikVdVaart2011,SzaboVdVaart2015} -- frames everything below;
Table~\ref{tab:dict} in App.~\ref{app:proofs} gives the working translations.

\section{Tight certificates: theory, and what exactly is new}
\label{sec:theory}

\begin{theorem}[Radius; attainment by adaptive-nugget GP regression]
\label{thm:radius}
For $\eps,\eta>0$ and $\ell=\ip{\varphi_x}{\cdot}$:
\emph{(a)}
$E_\star(x)=\sup\{h(x):\norm h\le\eps,\norm{\Lam h}\le\eta\}$ and
\begin{equation}\label{eq:radius}
E_\star(x)^2=\min_{\nu>0} V_\nu(x)\,\big(\eps^2+\eta^2/\nu\big);
\end{equation}
\emph{(b)} for fixed $\mu=1/\nu$ the rule $A_\mu(y)=w_\mu(x)^\top y$ has
exact worst-case error
$\eps\,\mathrm{pf}_\mu(x)+\eta\,\mathrm{wn}_\mu(x)$;
\emph{(c)} $\min_\mu[\eps\,\mathrm{pf}_\mu+\eta\,\mathrm{wn}_\mu]
=E_\star(x)$, the minimizer solving
$\mu_\star=\eps\,\mathrm{wn}_{\mu_\star}/(\eta\,\mathrm{pf}_{\mu_\star})$:
the GP posterior mean at the balance nugget is a globally optimal recovery
and its certificate equals the radius. Proof and provenance:
App.~\ref{app:proofs}, Table~\ref{tab:novel}; the value problem is classical
\citep{MicchelliRivlin1977,FoucartLiao2024}.
\end{theorem}

\begin{theorem}[Conditional Chebyshev intervals; robust-Bayes reading]
\label{thm:local}
Let $C_y(\rho,\eta)$ have nonempty interior and let $\ell=\ip{c}{\cdot}$ with
representer $c\in\Hi$. Then
\emph{(a)} the midpoint/half-width of $I_\ell(y)$ are the Chebyshev center
and radius of $\ell(C_y)$; \emph{(b)} they solve
$\inf_a\sup_{\pi\in\Gamma}\E_\pi(\ell(u)-a)^2$ for $\Gamma$ = all priors on
$C_y$ -- the known collapse of $\Gamma$-minimax to worst case under total
ambiguity \citep{Berger1985}; its value is that it \emph{licenses} reading the
certified interval as a credible band, and it is exactly computable by a
certified two-parameter dual (part \emph{(c)}, App.~\ref{app:proofs});
\emph{(d)} half-width $\le E_\star(\ell)$ for every $y$.
Restricted prior classes $\Gamma$ (moment or smoothness constraints), where
the collapse fails, are open here.
\end{theorem}

\begin{lemma}[Occam degeneracy]\label{lem:occam}
With $\rho_{\mathrm{oc}}(y,\eta)=\min\{\norm u:\norm{\Lam u-y}\le\eta\}$
(unique minimizer), $\rho=\rho_{\mathrm{oc}}$ makes $C_y$ a singleton and
every $I_\ell(y)$ zero-width; interiority requires
$\rho>\rho_{\mathrm{oc}}$.
\end{lemma}

\begin{theorem}[Conformalized radius; saturation fallback]
\label{thm:conformal}
Split into fit $F$ / calibration $C$ ($|C|=n_c$); build bands from $F$ only
at $\rho=\kappa\rho^F_{\mathrm{oc}}$; scores
$s_i=\inf\{\kappa\ge1:y_i\in\mathrm{band}_i(\kappa)\}$ (bands are nested in
$\kappa$), $\kappa_\star=\max_{i\in C}s_i$. For exchangeable
$(x_i,y_i)_{i\in C\cup\{\mathrm{new}\}}$ independent of $F$,
$\Pr[y_{\mathrm{new}}\in\mathrm{band}_{\mathrm{new}}(\kappa_\star)]
\ge1-\tfrac1{n_c+1}$ -- a guarantee on \emph{observations}, at the
achieved level $n_c/(n_c{+}1)$, which we always report. If the search grid
saturates before covering all of $C$, validity is restored by the additive
floor $\hat q=\max_i(|y_i-\mathrm{mid}_i|-\mathrm{half}_i)_+$, which our
implementation always adds; coverage of the \emph{truth} off-distribution
is empirical, never guaranteed (the noise-bounded transfer is
Prop.~\ref{prop:truth}, App.~\ref{app:proofs}).
\end{theorem}

\begin{proposition}[$\kappa$-invariance of certified LCB]\label{prop:inv}
Let $c:\mathcal X_{\mathrm{cand}}\to(0,\infty)$,
$m:\mathcal X_{\mathrm{cand}}\to\R$,
$x_\kappa=\argmin_{x}[m(x)-\kappa c(x)]$, and let
$\Delta_\kappa>0$ be the runner-up gap of $m-\kappa c$. If
$|\kappa'-\kappa|\cdot\mathrm{osc}(c)<\Delta_\kappa$, where
$\mathrm{osc}(c)=\max c-\min c$ over the candidates, then
$x_{\kappa'}=x_\kappa$. \emph{Proof.} For any $x$,
$(m-\kappa'c)(x)-(m-\kappa'c)(x_\kappa)\ge\Delta_\kappa-
|\kappa'-\kappa|\,|c(x_\kappa)-c(x)|>0$.\hfill$\square$
\end{proposition}

Any $\kappa E_\star$ with $\kappa\ge1$ keeps worst-case validity, so $E_\star$ is
the certified \emph{floor} of exploration widths, but nothing makes
$\kappa\cdot$floor a useful schedule (Rem.~\ref{rem:floor}). A flatness lemma
(Lem.~\ref{lem:flat}, App.~\ref{app:proofs}) predicts inertness precisely when
every candidate is far from the data on the kernel's length scale -- the
small-budget large-domain regime Sec.~\ref{sec:caseC} measures.

\section{OR-adjusted Bayesian learning: routes, tools, hypotheses}

Given a released learner (\textbf{R0}), an adjustment is \textbf{R1} (post-hoc
OR head on the frozen representation, noise budget from the learner's
\emph{own} estimate), \textbf{R2} (OR inside training: Morozov-calibrated
weights; water-filling prior variances), or
\textbf{R1{+}2}. Five algorithms implement it -- DiscrepancyCalibrate
\citep{Morozov1966,EnglHankeNeubauer1996}, TightGlobalCertify, LocalInterval,
ConformalRadius and GuardedDelivery (App.~\ref{app:algos}): the delivered local
half-width is a \emph{dual} value, hence an upper bound on the true supremum by
weak duality, and pathological duals fall back to a feasible ridge value with
the finite global certificate. A width is \emph{consumed} as (i) a promise,
(ii) a calibration signal, or (iii) an exploration bonus.
Pre-registered hypotheses and dispositions: \textbf{H1} (OR calibration matches
oracle tuning without test access, mode ii) -- \emph{supported, and sharpened
twice: the advantage is Morozov's noise-level awareness, and it extends from
accuracy to reproducibility on two independent axes}; \textbf{H2} (tight
certificates give valid, informative uncertainty, mode i) -- \emph{supported in
the well-specified regime, the numerical slack not exercised on this run's grid;
in the exchangeable regime split-conformal is the better band, and the conformal
floor we attach to the OR head is itself a net loss on a proper score};
\textbf{H3} (R1/R2 complementary in supervised BNNs) -- \emph{refuted: the
accuracy gain does not survive using the model's own noise estimate};
\textbf{H4} (transfer to BO, mode iii) -- \emph{refuted as a scheduling
principle and replaced by a checkable sufficient condition
(Prop.~\ref{prop:inv}), evaluated against the realised argmin at every step of
every run together with a measurement of how conservative it is}. A fifth is
forced by the data rather than pre-registered: \textbf{H5} (the OR shape is the
right tool under covariate shift) -- \emph{partially refuted and re-scoped: the
OR band keeps coverage everywhere but loses the proper score to conformal
baselines on most cells, and the condition separating the cells where it wins is
target learnability, measurable before training}.

\section{Case A: GP--PDE solver (TGPS)}\label{sec:caseA}

TGPS \citep{TGPS2026} solves $-\Delta u+u^3=f$ on $[0,1]^2$ (Dirichlet) by a
rank-10 tensor GP with sequential linearization; the release fixes the
data-fit weight $\lambda_2=1.49\times10^4$, tuned on noise-free data with
test-RMSE epoch selection. \textbf{Protocol} (App.~\ref{app:exp}): five noise
seeds per level; \emph{all} rules select from one shared $\lambda_2$ grid -- the
released value, Morozov (misfit closest to $\eta=\sigma\sqrt{N_c}$,
noise-level-aware), $\sigma$-blind ML-II and GCV \citep{GolubHeathWahba1979},
the L-curve corner \citep{Hansen1992}, and a test-access oracle; two collocation
grids, $24^2$ (the released configuration) and $16^2$, at the same five seeds
and $120$ epochs, giving $60$ cells, $420$ solver runs and no degenerate cell.
The second grid is not optional: the paired unit of the stability test is the
(grid, noise level) group, so one grid cannot clear $0.05$
(App.~\ref{app:caseA}).

\begin{table}[t]\centering\footnotesize
\caption{Case A: median test RMSE over 5 noise seeds per rule at $24^2$
collocation (shared $\lambda_2$ grid); gain $=$ released$/$Morozov, median
[min,max] of the per-cell ratio; last column the median per-cell
Morozov$/$oracle ratio. The $\sigma=0$ row is the median of five re-runs (the
solver is not run-to-run deterministic, Table~\ref{tab:Afloor},
App.~\ref{app:caseA}), where the
Morozov branch reduces to $\arg\min$ misfit and returns the released weight
in $5/5$ runs, so its gain is exactly $1$.}
\label{tab:A}
\setlength{\tabcolsep}{4pt}
\begin{tabular}{@{}lcccccccc@{}}
\toprule
rel.\ noise & released & Morozov & ML-II & GCV & L-curve & oracle &
gain & M/o\\
\midrule
0 & $4.0\!\cdot\!10^{-4}$ & $4.0\!\cdot\!10^{-4}$ & $4.0\!\cdot\!10^{-4}$ & $4.0\!\cdot\!10^{-4}$ & $1.4\!\cdot\!10^{-3}$ & $4.0\!\cdot\!10^{-4}$ & $1.00$ & 1.00\\
$10^{-4}$ & $3.0\!\cdot\!10^{-2}$ & $7.9\!\cdot\!10^{-4}$ & $3.0\!\cdot\!10^{-2}$ & $3.0\!\cdot\!10^{-2}$ & $1.4\!\cdot\!10^{-3}$ & $7.4\!\cdot\!10^{-4}$ & $40\times$ [20,77] & 1.12\\
$3\!\cdot\!10^{-4}$ & $4.3\!\cdot\!10^{-2}$ & $1.1\!\cdot\!10^{-3}$ & $9.2\!\cdot\!10^{-3}$ & $3.9\!\cdot\!10^{-2}$ & $5.0\!\cdot\!10^{-3}$ & $9.5\!\cdot\!10^{-4}$ & $40\times$ [21,135] & 1.18\\
$10^{-3}$ & $5.3\!\cdot\!10^{-2}$ & $1.9\!\cdot\!10^{-3}$ & $3.8\!\cdot\!10^{-3}$ & $1.2\!\cdot\!10^{-2}$ & $1.3\!\cdot\!10^{-2}$ & $1.7\!\cdot\!10^{-3}$ & $27\times$ [14,34] & 1.19\\
$3\!\cdot\!10^{-3}$ & $1.0\!\cdot\!10^{-1}$ & $4.6\!\cdot\!10^{-3}$ & $4.7\!\cdot\!10^{-3}$ & $3.8\!\cdot\!10^{-3}$ & $4.7\!\cdot\!10^{-3}$ & $3.6\!\cdot\!10^{-3}$ & $23\times$ [17,94] & 1.15\\
$10^{-2}$ & $6.2\!\cdot\!10^{-1}$ & $1.2\!\cdot\!10^{-2}$ & $1.2\!\cdot\!10^{-2}$ & $1.2\!\cdot\!10^{-2}$ & $1.2\!\cdot\!10^{-2}$ & $1.2\!\cdot\!10^{-2}$ & $47\times$ [24,71] & 1.00\\
\bottomrule
\end{tabular}
\end{table}

\textbf{Accuracy.} At the released grid Morozov reduces noise-time RMSE by
$23$--$47\times$ in the median (per-cell $13.5$--$135\times$) and tracks the
oracle within $1.00$--$1.19\times$ at every level, using the known noise level
but no test data (Table~\ref{tab:A}); at $16^2$ the noisy levels track to
$1.00$--$1.05\times$, and the single exception in either grid is that grid's
noise-free cell, at $1.71\times$ (Table~\ref{tab:A16},
App.~\ref{app:caseA}). Paired over the $25$ noisy $24^2$ cells, Morozov beats
the released weight by a median $-1.536$ dex on $\log_{10}$RMSE ($25/25$,
$p=5.96\!\cdot\!10^{-8}$, the attainable floor), ML-II by $-0.370$ ($20/25$,
$p=3.8\!\cdot\!10^{-5}$), GCV by $-0.751$ ($p=4.5\!\cdot\!10^{-4}$) and the
L-curve by $-0.203$ ($21/25$, $p=1.5\!\cdot\!10^{-5}$), and loses to the
oracle by $+0.056$ ($p=2.0\!\cdot\!10^{-4}$); a Friedman test over the five
deployable rules pooled across grids gives $\chi^2=124.5$,
$p=5.8\!\cdot\!10^{-26}$ (App.~\ref{app:caseA}). The $\sigma$-blind rules are
not substitutes: ML-II and GCV are pinned to the released weight at low noise
(ML-II in $7/25$ noisy runs, GCV in $9/25$) and release it only at high noise,
which their negative fitted noise exponents record (Fig.~\ref{fig:Aprofile},
App.~\ref{app:caseA}). The honest claim is not ``OR beats classical rules'' but
that the rule OR theory singles out -- the discrepancy principle -- is the one
that works, and wiring it into a modern GP--PDE stack is a one-line change worth
up to two orders of magnitude.

\textbf{Reproducibility.} Pooled over both grids ($n=10$ paired groups, the
only grouping whose attainable $p$-floor of $0.00195$ can clear $0.05$)
Morozov's median seed-CV of test RMSE is $0.102$ against $0.505$ (released,
$p=0.0020$, $10/10$ groups), $0.513$ (GCV), $0.570$ (L-curve) and $0.639$
(ML-II), all at $p=0.0039$ -- $4.9$--$6.3\times$ more reproducible; at $24^2$
alone Morozov's median seed-CV is $0.186$, the band $2.4$--$3.4\times$, and it
cannot reach significance at $n=5$, whose floor is $0.0625$ (Table~\ref{tab:Astab}, App.~\ref{app:caseA}). A
\emph{constant} weight is steadier still -- $\lambda_2=0.005$ has median seed-CV
$0.125$, a factor $1.48$ below Morozov -- but costs up to $3.03\times$ Morozov's
median RMSE, and neither gap is significant. Morozov remains the
lowest-dispersion deployable rule under four further dispersion measures at the
pooled grouping, though two of them fail at $24^2$ alone. Every dispersion
number here sits on a floor: at $\sigma=0$ the five ``noise seeds'' receive
bit-identical data with a fixed initialization seed, yet no two of the five RMSE
curves coincide, with seed-CV up to $0.110$ at $24^2$ and $2.102$ at $16^2$.
That $0.110$ is $59\%$ of Morozov's $0.186$ at the same grid, so the CVs are
upper bounds on seed variability, not estimates of it; normalizing each rule by
the floor at its own operating point leaves the ordering intact
(Table~\ref{tab:Afloor}, App.~\ref{app:caseA}).

\textbf{Portability.} Across a change of backend -- the \emph{same} $25$ noisy
$24^2$ cells at matched noise seed and level recomputed on different hardware,
so this is a replication of identical draws and not a deeper study -- only
$6.6\%$ of the $182$ matched (cell, $\lambda_2$) RMSE pairs agree to $1\%$, and
the rules separate by an order of magnitude: Morozov's selected model moves by
at most $1.36\times$ and the oracle's by $1.35\times$, against $12.3\times$
(released), $30.4\times$ (ML-II), $12.3\times$ (GCV) and $4.9\times$ (L-curve).
Individual entries of the stability table are \emph{not} portable; the ordering,
the ratio band and Morozov's own value are (Table~\ref{tab:Aport},
App.~\ref{app:caseA}).

\textbf{Certificates.} Coverage of the true solution is $1.0000$ at both
delivered noise levels and at all five rungs of the $\rho$-inflation ladder --
$0$ violations in $432$ evaluation points per level, a $95\%$ Clopper--Pearson
lower bound of $0.993$ on the pooled $0/432$ -- with the floating-point
allowance $\delta_{\mathrm{num}}$ never exercised on this run's grid and the
binding point still holding $31$--$33\%$ of the median certificate in reserve
(Table~\ref{tab:Acert}, Rem.~\ref{rem:slack}, App.~\ref{app:caseA}). The
certificate equals the radius
as a five-significant-digit \emph{invariant} and not as an identity: over all
$864$ points the ratio $\mathrm{cert}/E_\star$ has median $0.999989$, never
exceeds $1$, and is worst at $0.997273$, the shortfall being the certifier's
$97$-point $\mu$ grid rather than a failure of Thm.~\ref{thm:radius}c. At the
delivered $\rho=1.05\rho_{\mathrm{oc}}$ the $y$-conditional band is $104.6\times$
tighter than the global certificate at rel.\ $10^{-3}$; the $\sigma=0$ figure of
$3.10\times$ carries no information, because at $\eta=0$ the tightening is a
known function of the inflation alone and falls below $1$ for $m>\sqrt2$
(Table~\ref{tab:Aladder}, App.~\ref{app:caseA}). The Bayesian comparator is
two-sided: the
misfit-matched GP credible band is $3.3\times$ narrower at $\sigma=0$, but at
rel.\ $10^{-3}$ it undercovers its own posterior-mean error ($0.833$,
$120/144$) at $14.8\times$ the OR width, failing at exactly the same $24$ points
in all three seeds, every one on the two rings nearest the Dirichlet boundary.
Conformal calibration does not apply here -- collocation points are not
exchangeable draws -- which is precisely the regime where the worst-case ball is
the only guarantee available. Two scope limits: this block delivered two of
three planned noise levels, and with three seeds the smallest attainable
two-sided $p$ of any seed-level test is $0.25$, so every across-seed statement
here is descriptive; the conditional band's width$/$error premium is
$21.5\times$ in the median with a right tail reaching $3.8\!\cdot\!10^{3}$, so
only medians are quotable.

\section{Case B: Bayesian deep-kernel network (DAK)}\label{sec:caseB}

DAK \citep{DAK2025} composes a feature extractor, a variational linear
embedding and GP activations with a learned noise layer $\hat\sigma$.
\textbf{Protocol} (App.~\ref{app:exp}): five tabular targets
(\texttt{diabetes}, \texttt{concrete}, \texttt{energy}, \texttt{yacht},
\texttt{california}), each pinned by \texttt{data\_id} and passed through a
hard-failing target audit before training; two regimes, i.i.d.\ ($80/20$) and
covariate shift (a rank split along a random unit direction, $70/30$), plus one
legacy axis-mode anchor cell; seeds $\{0,1\}$; label noise
$\mathrm{rel}\in\{0,.25\}$ everywhere and $\{0,.1,.25,.5\}$ on
\texttt{diabetes}. Nine arms per cell: DAK as released and KL-repaired; split-CP
and normalized split-CP (scores $|r|/\hat\sigma(x)$) on a fit-split model the
network never calibrates on; a \emph{trivial} split-CP band around a constant
predictor, using no features and no training; the OR head (R1) on the
\emph{same} fit-split model and calibration indices, plus its no-floor and
pure-interval ablations; and the water-filling prior (R2) with the model's
\emph{own} $\hat\sigma$. Coverage is a mean over cells; width (within-cell
median half-width unless stated) and the Winkler interval score at $95\%$
\citep{GneitingRaftery2007} are median [IQR] over cells. Targets are
standardized, so a half-width of $1$ is one training standard deviation. Audit,
conformal levels and nine-arm ledger: Tables~\ref{tab:Baudit},
\ref{tab:Blevels}, \ref{tab:Barms} in App.~\ref{app:caseB}.

\begin{table}[t]\centering\footnotesize
\caption{Case B under covariate shift, per dataset: the ledger the pooled
numbers hide. Coverage, width and IS are \emph{means} over that dataset's
shift cells -- width the within-cell median half-width, IS the Winkler score
at $\alpha=0.05$, lower better -- because the per-point tail lives in the
mean. $R^2$ is the training-free RandomForest 5-fold audit statistic of
Table~\ref{tab:Baudit} (App.~\ref{app:caseB}), computed before any model is
fitted; $d_{\mathrm{M}}$ is the \emph{median} Mahalanobis shift severity; $n$ is
the number of shift cells the dataset contributes, and on \texttt{diabetes} the
two OR columns rest on $5$ of the $6$, the OR arms having refused on
\texttt{diabetes/shift/0/0.5}. Bold marks the lowest interval score \emph{among
the four arms shown}; on \texttt{energy} and \texttt{yacht} the full nine-arm
ledger contains arms that score lower still (Table~\ref{tab:Barms},
App.~\ref{app:caseB}). The
coverage the OR head wins over split-CP on all $5$ datasets is the only
universal win here; on
\texttt{california} its interval score is destroyed by a per-point tail
(Table~\ref{tab:Btail}) that costs \texttt{or1\_nf} almost as much.}
\label{tab:Bshift}
\setlength{\tabcolsep}{3pt}
\begin{tabular}{@{}l r r r ccc ccc ccc ccc@{}}
\toprule
& & & & \multicolumn{3}{c}{split-CP} & \multicolumn{3}{c}{triv (constant)}
& \multicolumn{3}{c}{OR head (R1)} & \multicolumn{3}{c}{OR, no floor}\\
\cmidrule(lr){5-7}\cmidrule(lr){8-10}\cmidrule(lr){11-13}\cmidrule(lr){14-16}
dataset & $R^2$ & $d_{\mathrm{M}}$ & $n$ & cov & wid & IS & cov & wid & IS & cov & wid & IS & cov & wid & IS\\
\midrule
\texttt{concrete} & 0.343 & 1.29 & 4 & 0.831 & 0.80 & \textbf{3.74} & 0.970 & 1.96 & 4.26 & 0.996 & 2.04 & 7.22 & 0.977 & 1.43 & 6.27\\
\texttt{diabetes} & 0.419 & 1.25 & 6 & 0.950 & 2.09 & 4.88 & 0.995 & 2.05 & \textbf{4.17} & 1.000 & 3.45 & 26.21 & 0.923 & 1.83 & 24.01\\
\texttt{california} & 0.740 & 1.39 & 3 & 0.962 & 1.35 & \textbf{3.35} & 0.969 & 2.26 & 4.96 & 1.000 & 3.29 & 111.0 & 0.954 & 1.55 & 108.2\\
\texttt{energy} & 0.965 & 3.04 & 3 & 0.784 & 0.33 & \textbf{2.05} & 0.986 & 1.98 & 4.03 & 1.000 & 1.41 & 3.44 & 1.000 & 1.29 & 3.20\\
\texttt{yacht} & 0.995 & 1.60 & 3 & 0.896 & 0.66 & 3.24 & 0.936 & 2.73 & 6.39 & 0.971 & 1.04 & 2.73 & 0.968 & 0.97 & \textbf{2.62}\\
\bottomrule
\end{tabular}
\end{table}
\begin{figure}[t]\centering
\includegraphics[width=\textwidth]{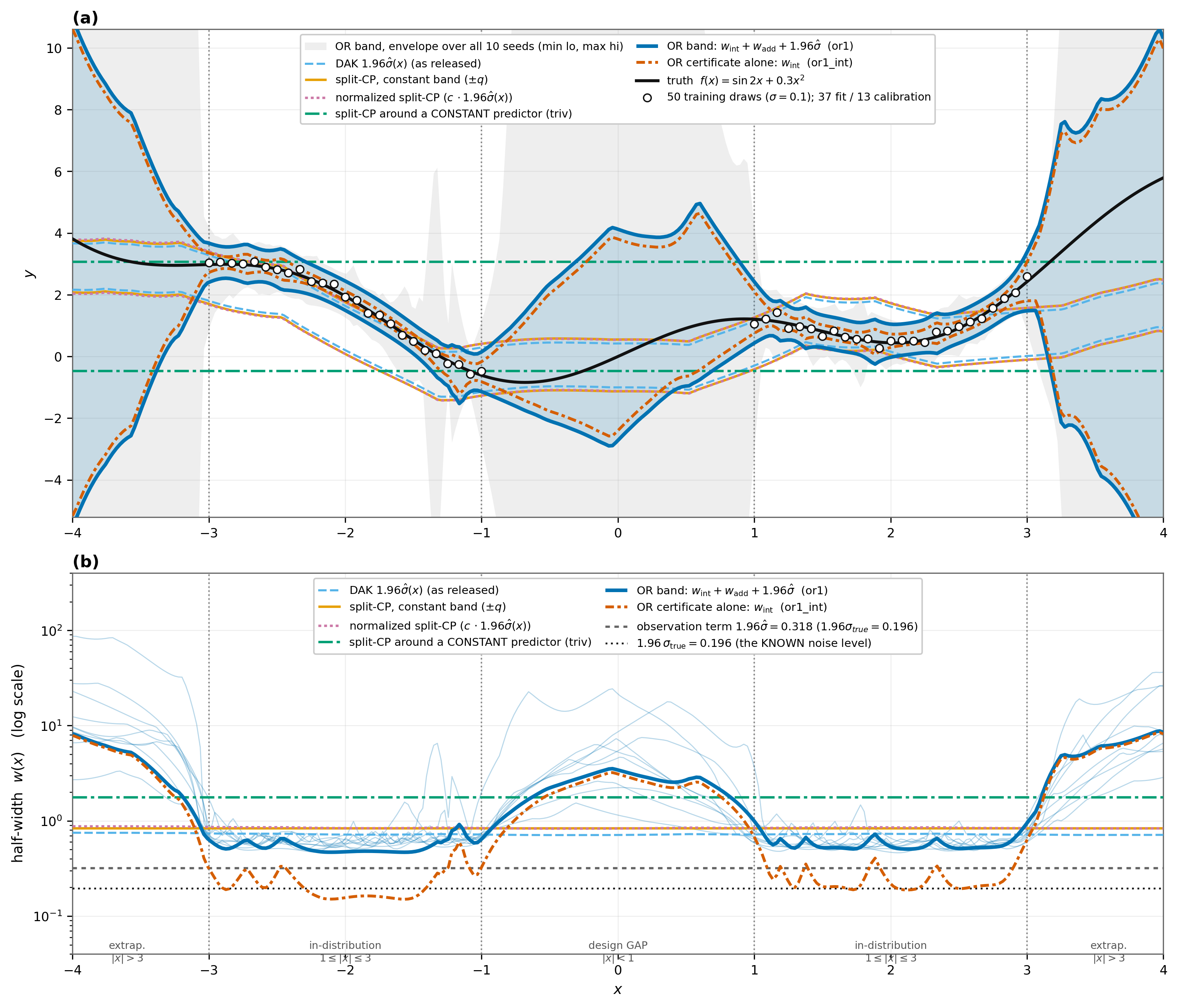}
\caption{Case B gap task, ten seeds. (a) the six bands on seed $7$ -- the median
seed of the ten on $\kappa_\star$, on $q_{cp}$ and on every region width, its
$q_{cp}=0.8395$ being the paper's single-seed $\pm0.84$ -- the OR band's
ten-seed envelope in grey, the vertical range truncated for legibility; the DAK
band, both conformal variants and the constant-predictor band are flat.
(b) half-width profiles on a log axis, all
ten seeds, every half-width in full: the OR band pinches on the data and
inflates through the gap and the
extrapolation, against the marked $1.96\hat\sigma$ and
$1.96\sigma_{\mathrm{true}}$. On the seed drawn the learner overestimates the
known noise by $1.62\times$ ($0.318$ against $0.196$), ten-seed median
$1.76\times$, which is where the in-distribution width comparison of
Table~\ref{tab:Bgapsurv} goes.}
\label{fig:B}
\end{figure}

\textbf{In-distribution, the concessions first.} In-distribution
split-conformal is the better band: it beats the OR head on interval score in
$18/19$ paired cells (median $-2.573$, bootstrap $95\%$ CI $[-3.36,-1.20]$,
$p=7.6\!\cdot\!10^{-6}$ at an attainable floor of $3.8\!\cdot\!10^{-6}$) and
is narrower in $19/19$ (median $-1.143$), while the OR head is valid ($1.000$
mean coverage) but wide (median half-width $2.06$ against $0.91$);
Fig.~\ref{fig:Bp} and Table~\ref{tab:Barms} in App.~\ref{app:caseB} give all
nine arms. The
water-filling prior yields no gain once its noise hint is the model's
\emph{own} $\hat\sigma$: on
\texttt{diabetes} its mean interval score is $6.95$ i.i.d.\ against $7.10$ and
$7.03$ for the two DAK arms (mean coverage $0.861$ and $0.873$ against nominal
$0.95$), and $8.13$ under shift against $8.02$ and
$6.77$, at coverage $0.685$ and $0.703$; on the median convention the
released-arm comparison flips, so the concession is that R2 buys nothing
detectable, not that it is uniformly worse. The conformal floor we bolt onto the
OR head is itself a net loss on a proper score: it buys $+0.010$ coverage for
$+0.96$ of median width and $+1.85$ of interval score i.i.d.\
($p=6.5\!\cdot\!10^{-4}$), and $+0.023$ for $+0.84$ and $+1.26$ under shift, so
we report the no-floor band beside the delivered one throughout
(Table~\ref{tab:Bfloor}, App.~\ref{app:caseB}).

\textbf{Under shift.} The one thing the OR head wins universally is coverage:
pooled it holds $0.9969$ $[0.9947,0.9984]$ of test points against $0.890$ for
split-CP and $0.891$ for normalized split-CP, and it beats split-CP on
coverage on $5/5$ datasets. It loses the proper score. Against split-CP it is
worse on interval score in $16$ of $18$ delivered shift cells (median
$+2.695$, $p=8.4\!\cdot\!10^{-4}$; blocked on the $10$ distinct splits,
$9/10$, $p=0.0098$); against normalized split-CP, $16/18$ (median $+2.719$,
$p=4.2\!\cdot\!10^{-4}$); and against a \emph{feature-free constant} band,
$12/18$ (median $+2.677$, $p=0.038$). Repeating it on the median-based
pointwise score, which is immune to the one catastrophic cell, does not rescue
it (better in $1/18$, $p=1.5\!\cdot\!10^{-5}$). The coverage it does win over
the constant band is worth $+0.0075$ in the median cell and costs $+1.374$ of
\emph{mean} half-width ($+0.366$ on medians, $p=0.52$). We therefore withdraw
``the OR shape covers everything at smaller in-distribution width''.

\textbf{The regime that pays is shift on a learnable target.} The per-dataset
ledger splits along a training-free audit statistic (Table~\ref{tab:Bshift}):
on the two targets a random forest recovers at CV $R^2\ge0.965$ the OR band is
competitive on the proper score, and on the three at $R^2\le0.74$ its mean
interval score is $2$--$33\times$ worse ($1.7$--$9.6\times$ on within-cell
medians). Even on the learnable pair the win is bounded: on \texttt{energy} the
OR band's $3.44$ and the no-floor band's $3.20$ both lose to split-CP's $2.05$
and normalized split-CP's $1.73$, so what it wins there is coverage, not score;
on \texttt{yacht} it beats every conformal arm ($2.73$ and $2.62$ against
$3.24$ and $6.39$) but the two DAK arms themselves score $2.38$ and $2.41$. The
same statistic predicts, at Spearman $\rho=-0.62$ ($p=4.5\!\cdot\!10^{-5}$ over
$37$ cells), which cells let the ball geometry rather than the conformal floor
carry the band. Holding the centre and the mean width fixed and asking whether
the OR geometry covers more than a flat band of that width would, the skill is
positive on $5/6$ \texttt{energy} and $6/6$ \texttt{yacht} cells against only
$1/12$ and $3/13$ elsewhere (Fisher $p=0.0039$ and $0.0031$), with a best case
of $+0.719$ coverage at $0.41\times$ the trivial band's width.

\textbf{The failure under shift is the centre, not the width.} Over the $20$
shift cells the median coverage is $0.921$ for split-CP, $1.000$ for the same
shift-blind centre given the constant band's wider quantile
($q_{\mathrm{triv}}=2.127$ against $q_{\mathrm{cp}}=1.159$), $0.830$ for a
constant centre with $q_{\mathrm{cp}}$, and $0.9925$ for the constant band
itself. No shape is required, and the OR head does not repair the centre either:
its RMSE under shift is worse than the KL-repaired posterior mean's ($4/18$,
$p=0.0034$) and no better than split-CP's ($7/18$, $p=0.18$). The smallest
constant inflation of split-CP matching the OR head's coverage is $1.68\times$
in the median, and there the shape-free band reaches OR coverage at $1/1.79$ of
its mean half-width, narrower in $16/18$ cells (App.~\ref{app:caseB}).

\textbf{Tail, refusals, defects, and scope.} The worst-case semantics survive
the repair as a per-\emph{point} tail: $95$ of $7129$ test points ($1.33\%$)
in $13$ of $37$ cells carry an OR interval above
$20\,\mathrm{sd}(y_{\mathrm{fit}})$, with a maximum of $61896$ on a
standardized target, $92$ of them under shift; $55$ delivered half-widths in
$16$ cells are guarded substitutions rather than tight duals (largest relative
primal--dual gap $0.55$); and on two cells the head \emph{refused} outright,
the noise budget implied by the network's own $\hat\sigma$ being smaller than
the best achievable fit-split misfit (margins $-0.487$, $-1.553$)
(Table~\ref{tab:Btail}, App.~\ref{app:caseB}). Tightening the bounds exposed
the defects: a feature-scale clamp that multiplies the within-cell mean half-width
by $235$ and the maximum by $2681$ while moving the median by only $1.02$, and
an all-rows delivery rule that is a \emph{validity} rather than a width defect,
taking the refusal rate from $3$--$4$ of $30$ cells to $18$--$19$ (McNemar
$p=6.1\!\cdot\!10^{-5}$). A separate 1D design-gap task (S6; ten seeds re-drawing
the noise and the $37/13$ fit/calibration split on a fixed design,
Fig.~\ref{fig:B}) probes the same head in the gap, in-distribution and in
extrapolation: coverage of the truth is $1.000$ in all three regions, at or
above nominal on $10/10$, $10/10$ and $8/10$ seeds, and no other arm reaches
nominal in extrapolation on any seed; but the $1.7\times$ in-distribution width
advantage reported for the single-run version does \emph{not} survive -- the
ratio is $1.21$ $[0.92,1.49]$ with the OR band narrower on only $6/10$ seeds
($p=0.16$) -- and what does survive at the attainable floor is $2.67\times$
against the OR certificate itself and $2.99\times$ against split-conformal
around a constant predictor, the only baseline that attains nominal
in-distribution ($10/10$, $p=0.002$). In the gap alone the interval-score
comparison is a tie ($5.61$ against $8.03$, $6/10$, $p=0.77$). Scope: the
session delivered $39$ of $73$ planned cells; the $19$ random-direction shift
cells rest on $10$ distinct splits, so cell-level $p$-values are optimistic and
we give the split-blocked version wherever it changes a conclusion; and the
shift is a weak test of blindness, because the rank split \emph{narrows} the
target distribution and the feature-free baseline's coverage therefore
\emph{rises} under it (median $+0.0216$, higher in $14/19$) while split-CP's
falls ($-0.0451$, $p=0.0046$). The axis-mode \texttt{shiftcol} cell is one cell
on which the OR arm refused and is not a third regime.

\section{Case C: Bayesian optimization (ALAS-BO)}\label{sec:caseC}

ALAS-BO \citep{ALAS2026} couples a learnable $\alpha$-stable mixture kernel with
UCB ($\beta{=}0.2$) or EI \citep{Srinivas2010,ChowdhuryGopalan2017}.
\textbf{Protocol} (App.~\ref{app:exp}): \texttt{botorch 0.18.1}, $Q{=}3$;
Branin-2D, Hartmann-6, Griewank-5D, Ackley-5D, Levy-10D; $167$ trajectories from
one session, all of length $19$, $6$ or $7$ seeds per (benchmark, arm), every
paired test pairing only on shared seeds. All $80$
(benchmark, arm, seed) trajectories this corpus shares with an earlier CPU
corpus are \emph{bitwise identical} at all $19$ iterations -- the one place in
this paper where a cross-machine reproduction is exact. Writing $m(x)$ for the
posterior mean ($\mu$ is reserved for the dual parameter), \textbf{OR-LCB} picks
$\argmin_x[m(x)-c_\star(x)]$ with $c_\star$ the information radius
($\eps=\rho_{\mathrm{oc}}$, $\eta=\hat\sigma\sqrt t$); the direct floor-law test
\textbf{$\kappa{\times}$floor} uses $\kappa c_\star(x)$, $\kappa\in\{2,5\}$.

\begin{figure}[t]\centering
\includegraphics[width=0.85\textwidth]{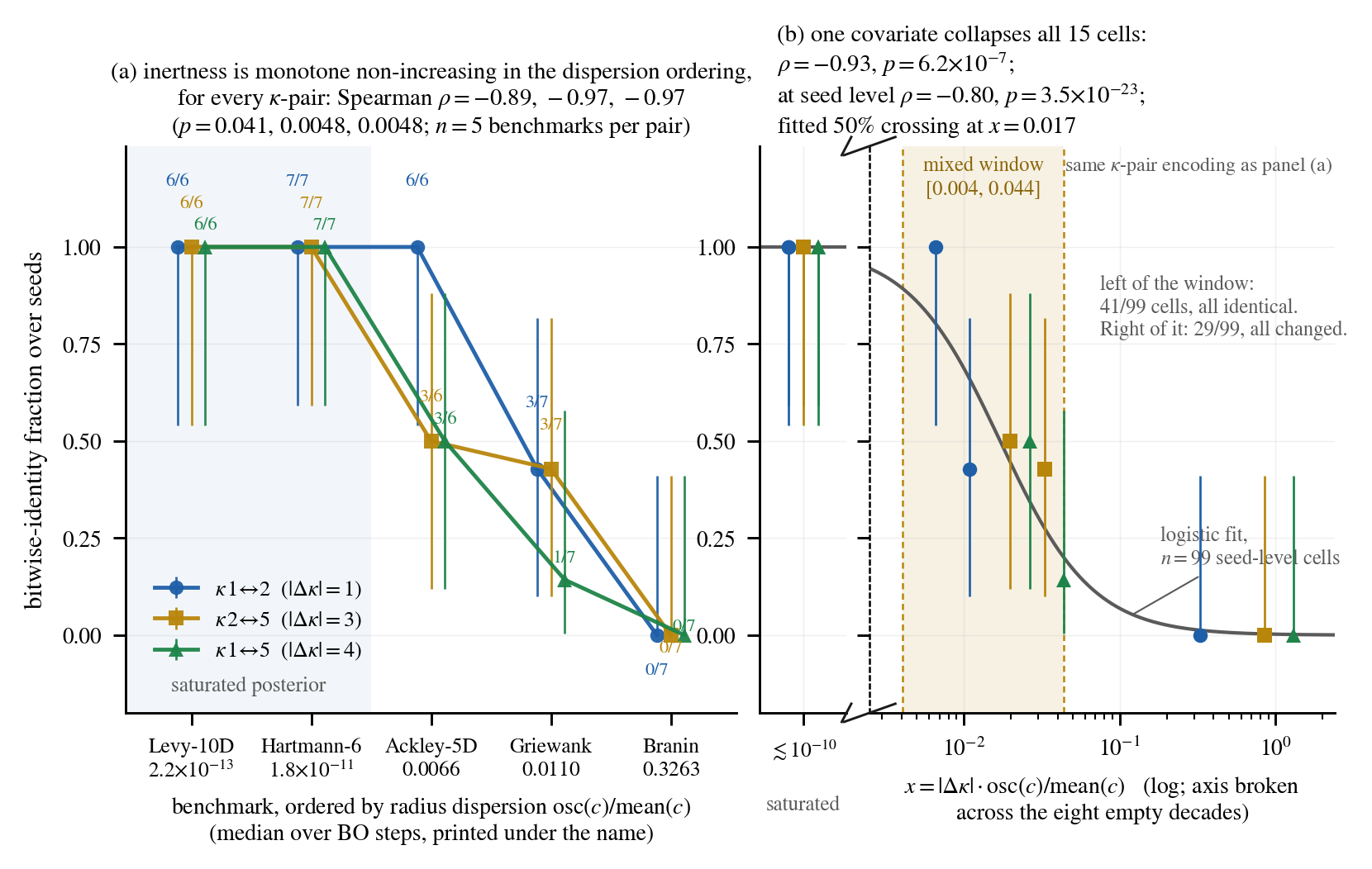}
\caption{$\kappa$-inertness is graded rather than binary; one point per
benchmark $\times$ $\kappa$-pair. (a) Ordering the
benchmarks by radius dispersion, the identity fraction is monotone
non-increasing for every $\kappa$ pair. (b) Both dependences -- on dispersion
and on $|\Delta\kappa|$ -- collapse onto the single dimensionless covariate
$x=|\Delta\kappa|\cdot\mathrm{osc}(\hat c)/\mathrm{mean}(\hat c)$, the axis
broken across the eight empty decades separating the saturated benchmarks from
the rest. Bars are Clopper--Pearson $95\%$ intervals. Two readings need care:
the counts drawn are regret-trajectory identity, marginally weaker than the
action-level counts of Table~\ref{tab:Ckappa} (App.~\ref{app:caseC}); and the
two left-most dispersion values, printed under the benchmark names inside the
shaded saturated band, are $O(\mu_{\mathrm{lo}})$ grid artifacts and not
measurements, shown only to place those benchmarks at the left end of the
ordering and never to be compared with each other (see text and
Rem.~\ref{rem:gridfloor}).}
\label{fig:Cmech}
\end{figure}

\textbf{Regret.} Paired against EI on shared seeds over $33$ complete blocks,
OR-LCB is worse by $+0.075$ dex in the pooled normalized $\log_{10}$ ratio
($p=0.0011$); EI is the stronger default, and a Friedman test separates the
arms ($\chi^2=19.33$, $p=6.8\!\cdot\!10^{-4}$, mean ranks EI $2.41$, UCB $2.64$,
$5{\times}$floor $3.05$, $\kappa{=}1$ $3.42$, $2{\times}$floor $3.49$)
(Table~\ref{tab:C}, App.~\ref{app:caseC}). A second ordering by AUC disagrees
with the final-regret
ranking on three of five benchmarks (App.~\ref{app:caseC}).

\textbf{The $\kappa$ ladder.} $\kappa{=}2$ is inert wherever $\kappa{=}1$ is
-- same candidate at every step on $6/6$ Levy-10D, $7/7$ Hartmann-6 and $6/6$
Ackley-5D seeds -- while $\kappa{=}5$ breaks away on $3$ of $6$ Ackley-5D and
$7$ of $7$ Griewank-5D seeds, and where it acts it helps, moving the medians
$7.92\to7.34$ and $12.40\to9.92$ without reaching significance at six or seven
seeds (Fig.~\ref{fig:Cmech}; exact action-level counts, dispersions and
escape steps in Table~\ref{tab:Ckappa}, App.~\ref{app:caseC}). The claim of an
earlier revision that ``$\kappa{=}5$ does not help'' rested on a two-seed
Griewank median and is withdrawn. Ordered by radius dispersion
$\mathrm{osc}(\hat c)/\mathrm{mean}(\hat c)$ -- Levy-10D and Hartmann-6 below the
certifier's $\mu$-grid floor, then $0.0066$ Ackley-5D, $0.0110$ Griewank-5D,
$0.3263$ Branin-2D -- the action-level inertness of the full ladder is
monotone non-increasing, and within every benchmark it is monotone
non-increasing in $|\Delta\kappa|$: the shape of a graded law rather than a
binary property of two benchmarks.

\emph{This is not a resolved five-level ordering.} Pairwise Fisher exact tests
with Holm correction over the ten benchmark pairs resolve exactly two groups,
$\{$Branin, Griewank$\}$ and $\{$Hartmann-6, Levy$\}$ (four comparisons at
Holm $p=0.006$--$0.037$); Ackley-5D is indistinguishable from both (Holm
$p\ge0.42$); and the two comparisons the word ``graded'' would most need --
Branin against Griewank at the bottom and Hartmann-6 against Levy at the top
-- are both $p=1.00$. Radius dispersion is moreover confounded with
\emph{dimension} across these five objectives (exact-permutation Spearman
against inertness: dispersion $-0.97$, $p=0.033$; dimension $+0.92$,
$p=0.067$), the one dimension-controlled comparison available (Ackley against
Griewank, same $d=5$, verified bitwise-identical initial designs) is Fisher
$p=0.070$, and we cannot separate the two explanations from this corpus. And
the ordering is a budget snapshot: recomputing inertness at a truncated budget
of six iterations gives Griewank $7/7$ and Ackley $6/6$ and destroys it
entirely. The two saturated benchmarks at the left of Fig.~\ref{fig:Cmech}(a)
sit below the certifier's
$\mu$-grid floor, so their dispersions are bounds that must not be compared
with each other, and we quote no dispersion below about $10^{-6}$
(Rem.~\ref{rem:gridfloor}, App.~\ref{app:proofs}).

\textbf{The proposition, tested at every step.} At every BO iteration of every
arm and seed we recorded the criterion of Prop.~\ref{prop:inv} alongside whether
the argmin actually coincided: $9{,}018$ records, $0$ counterexamples, $0$ of
$99$ at whole-trajectory resolution, and $0$ within every benchmark and every
$\kappa$ pair separately (Table~\ref{tab:Cinv}, App.~\ref{app:caseC}). It is
not a tie artifact -- the
minimum runner-up gap anywhere in the corpus is $1.02\!\cdot\!10^{-7}$ -- and not
vacuous: the largest margin ratio on a firing record is $0.998$, so the
criterion came within $0.2\%$ of binding and did not fail. What it is, is
conservative: it fires on $65.8\%$ of steps, covers $84.7\%$ of the steps on
which invariance in fact occurred, and above its threshold the argmin still
coincided on $1{,}076$ of $3{,}082$ steps ($34.9\%$). A usable rule follows --
when the criterion fires on at least $80\%$ of the $18$ steps the trajectory was
bitwise identical in $48/48$ cells, and when it fires on fewer than half, in
$0/28$ (App.~\ref{app:caseC}).

\textbf{What the inert end is, and is not.} It is where every acquisition is
inert. On Levy-10D five of six seeds produce one action path across all five
arms, and the UCB-versus-OR-LCB action-identity rate across the five benchmarks
($0/7$, $2/7$, $2/6$, $5/7$, $6/6$) reproduces the same ordering with no
$\kappa$ anywhere in it; the criterion is vacuous there by \emph{always} firing
($1620/1620$ Levy, $1863/1890$ Hartmann-6 records), supplying $3{,}483$ of the
$5{,}936$ firings and none of the information.

\section{Discussion and limitations}

\textbf{The synthesis} (Table~\ref{tab:regime}, App.~\ref{app:exp}). Without
exchangeable
calibration data the worst-case ball is the only guarantee available, and
Morozov, its calibration rule, is the accuracy winner, the reproducibility
winner across noise draws, and the only deployable rule whose choice survives a
change of machine. With such data, in-distribution, split-conformal is the
better band, in $18$ of $19$ paired cells, and the honest addition is that the
conformal floor we bolt onto the OR head is itself a net loss on a proper score,
so the OR head should be reported without it. Under covariate shift an earlier
revision's claim was too strong and we retract it: the OR band does keep
coverage where both conformal variants lose it, on every dataset we tried, but
it pays for that with width, and on a proper score it loses to split-conformal
on $16$ of $18$ cells and to a band built from no features at all on $12$ of
$18$; the shortfall is a failure of the model \emph{centre}, not of the constant
width. What decides between the two covariate-shift regimes is not the regime
label but the
target: a random forest run before any model is fitted separates the two
datasets on which the OR band recovers the coverage split-conformal loses from
the three on which its per-point tail costs $2$--$33\times$ more in the mean
than its coverage buys. ``Put OR where the guarantee is consumed'' therefore
refines to \emph{match the tool to the data regime and audit the regime first}
-- and the audit is cheap. For exploration the certified floor anchors validity
and Prop.~\ref{prop:inv} says when scaling it can do nothing, checked $9{,}018$
times without a counterexample while we measured how often it is silent when
nothing happens anyway.

\textbf{What reproducibility cost us, and bought us.} Two conclusions changed
because we measured the same computation twice. At $\sigma=0$ the Case-A solver
returns five different answers to one deterministic problem, putting a floor of
up to $0.110$ under every seed-CV we report; across a change of backend,
individual entries of that table move by factors $0.19$ to $1.49$ and the
low-noise gain moves from $26\times$ to $40\times$, and what survives is the
ordering, the ratio band and Morozov's own value. Case C, conversely, reproduces
\emph{bitwise} across machines on all $80$ shared trajectories.

\textbf{Limitations.} Case A rests on one PDE forcing and one solver
initialization seed, so nothing here claims instance-independence, and its
certificate block delivered two of three planned noise levels, so the
high-noise end of the certificate ladder is untested; with three seeds there,
no seed-level comparison can beat $p=0.25$. The $13.6$--$34.4\times$ Case-A
width$/$error premium is the price of norm-ball semantics, and its pointwise
distribution has a right tail reaching $3.8\!\cdot\!10^{3}$, so only medians
are quotable. The two radii that can anchor the conditional ladder differ by
$26\%$ at $\eta>0$ because they are the exact radii of two different
consistent sets, and the global certificate is consequently certified over a
$26\%$ smaller model ball than the conditional band; we present them as two
stages, not one. The Case-A cross-backend arm is a \emph{replication} of the
same noise draws on other hardware, not a deeper study. Case B delivered $39$
of $73$ planned cells, its $19$ shift cells rest on $10$ distinct splits, its
random-direction shift is a weak test of shift-blindness because it
\emph{narrows} the target distribution and so lifts the feature-free
baseline's coverage, its axis-mode cell is one refused cell and not a third
regime, and a per-point tail, $55$ guarded substitutions and two refusals
remain (Sec.~\ref{sec:caseB}). Truth-coverage under shift is empirical
by necessity (Thm.~\ref{thm:conformal} scope). Case-C conclusions are budget-
and benchmark-bound: at $23$ evaluations the flatness mechanism is exactly
what small-budget BO in large domains produces, two of the inertness cells
diverge only at $t=15$ and a $17$-iteration budget would move two further
Griewank cells, so a shorter budget would report a different table, and
recomputing at a six-iteration budget destroys the ordering entirely. The
ordering itself resolves only two groups under Holm-corrected pairwise tests,
with Ackley-5D indistinguishable from both; radius dispersion is confounded
with benchmark dimension across our five objectives, and the single
dimension-controlled comparison available does not clear $0.05$; the two fully
inert benchmarks are configurations in which no acquisition function acts at
all, so they cannot falsify anything about the OR radius specifically; and
Case C contains no within-session determinism control, so its bitwise-identity
results are conditional on the GP fit being reproducible across calls, which
we did not verify. The $Q{=}7$ Branin probe and the or\_full$/$or\_ei diagnostics come from an
earlier corpus and support no claim here.

\textbf{Open.} Shape-aware certified exploration, with the margin condition of
Prop.~\ref{prop:inv} as the thing a schedule must be designed to violate; a
deep-budget Case-C sweep, to show where the inertness grading converges rather
than where a $23$-evaluation snapshot leaves it; a dimension-controlled
benchmark family, to separate radius dispersion from dimension;
restricted-$\Gamma$ credible intervals where the minimax collapse fails;
conformal--OR hybrids with partial shift guarantees \citep{GibbsCandes2021}; and
a shift-aware conformal \emph{centre}, where our four-way counterfactual
suggests the real gap is.

\clearpage

{\small
\bibliographystyle{plainnat}
\bibliography{references}
}

\clearpage

\appendix

\section{Theory: deferred statements and full proofs}\label{app:proofs}

\subsection{Deferred statements and remarks}

\begin{table}[t]\centering\footnotesize
\caption{Working dictionary.}\label{tab:dict}
\begin{tabular}{@{}l>{\raggedright\arraybackslash}p{3.0in}@{}}
\toprule
Optimal Recovery & Bayesian learning\\
\midrule
model ball $\norm u\le\eps$; noise ball $\norm e\le\eta$ &
prior scale; Gaussian likelihood\\
regularization weight / nugget $\nu$ & prior-to-noise ratio
$\sigma^2/s^2$\\
power function (exact data) & GP posterior sd\\
radius of information $E_\star(x)$, Eq.~\eqref{eq:radius} &
nugget-optimized posterior variance (new, Thm.~\ref{thm:radius})\\
conditional Chebyshev interval $I_{\varphi_x}(y)$ &
total-ambiguity robust-Bayes credible interval (Thm.~\ref{thm:local})\\
Occam radius $\rho_{\mathrm{oc}}$; conformal inflation
$\kappa_\star\rho_{\mathrm{oc}}$ & degenerate / repaired data-driven prior
scale (Lem.~\ref{lem:occam}, Thm.~\ref{thm:conformal})\\
\bottomrule
\end{tabular}
\end{table}

\begin{table}[t]\centering\footnotesize
\caption{Provenance ledger for Sec.~\ref{sec:theory}.}\label{tab:novel}
\begin{tabular}{@{}>{\raggedright\arraybackslash}p{2.45in}>{\raggedright\arraybackslash}p{2.6in}@{}}
\toprule
Known & Here\\
\midrule
Two-ellipsoid value problem; linearity of optimal recovery; S-procedure
exactness \citep{MicchelliRivlin1977,MelkmanMicchelli1979,FoucartLiao2024,%
FoucartLiaoSproc}; optimal parameter within the regularization family
\citep{FoucartLiao2023}; Morozov's principle \citep{Morozov1966,%
EnglHankeNeubauer1996} &
Bayesian packaging \eqref{eq:radius} as nugget-optimized posterior
variance; the identity $V_{1/\mu}=\mathrm{pf}^2+\mathrm{wn}^2/\mu$ and the
6-line AM--GM attainment proof with the balance equation
$\mu_\star=\eps\,\mathrm{wn}/(\eta\,\mathrm{pf})$; per-functional delivered
certificate $=$ radius as an \emph{implementation invariant}\\
\addlinespace[2pt]
$\Gamma$-minimax collapses to minimax under total ambiguity
\citep{Berger1985,Vidakovic2000}; Chebyshev-center computation
\citep{FoucartLiaoSproc}; OUQ programs \citep{Owhadi2013} &
The per-functional \emph{credible-interval} reading with certified
$2$-parameter dual, exact-row elimination and pencil-fast evaluation;
positioning as a drop-in band for learned features\\
\addlinespace[2pt]
Split conformal prediction \citep{Vovk2005,Lei2018,GibbsCandes2021} &
Occam-degeneracy lemma (why data-driven radii collapse); conformal
calibration \emph{of the model-ball radius} via nested bands, with the
saturation-to-additive-floor fallback\\
\bottomrule
\end{tabular}
\end{table}

Part \emph{(c)} of Thm.~\ref{thm:local}, deferred from the body: the
half-width is exactly computable as
$S_+(\ell)=\min_{a,b\ge0} u_{a,b}^\top M_{a,b}^{-1}u_{a,b}+a\rho^2+b\eta^2
-b\norm y^2$, $M_{a,b}=aI+b\Phi^\top\Phi$, $u_{a,b}=\tfrac c2+b\Phi^\top y$
(S-procedure exactness \citep{FoucartLiaoSproc,PolikTerlaky2007}), certified
a posteriori by primal recovery, with exact rows eliminated first.

\begin{remark}[Numerical slack]\label{rem:slack}
Delivered certificates carry an explicit floating-point allowance
$\delta_{\mathrm{num}}=10\sqrt{\eps_{\mathrm m}}\,\hat\eps$ (an
eigendecomposition backward-error budget through the $\hat\eps$-scaled
quadratic forms, $\eps_{\mathrm m}$ machine epsilon). It is an insurance
policy, not a correction. On the $12\times12$ interior evaluation grid
used here it is $3.73\!\cdot\!10^{-6}$ at $\sigma=0$ ($0.024\%$ of the
median certificate) and $2.82\!\cdot\!10^{-6}$ at rel.\ $10^{-3}$, while
raw coverage is already $1.0000$ at both levels -- $0$ violations in
$432$ evaluation points each -- and the tightest point still holds
$33\%$ ($\sigma=0$) and $31\%$ (rel.\ $10^{-3}$) of the median
certificate in reserve. On a $300$-point grid that included points
essentially on top of collocation nodes the same construction did
produce raw violations of order $10^{-6}$, which is the regime the
allowance is sized for; on the grid of this run it is never exercised.
\end{remark}

\begin{remark}[Floor, not schedule]\label{rem:floor}
Any $\kappa E_\star$, $\kappa\ge1$, keeps worst-case validity, so
$E_\star$ is the certified \emph{floor} of exploration widths -- but
nothing makes $\kappa\cdot$floor a useful schedule: if $E_\star(\cdot)$ is
(near-)constant over the candidate set, $\argmin_x[m(x)-\kappa E_\star(x)]$
is $\kappa$-invariant. Sec.~\ref{sec:caseC} measures how often this
happens and finds it graded rather than binary: the full $\kappa$ ladder
leaves the action path unchanged on $6/6$ Levy-10D and $6/7$ Hartmann-6
seeds, on $3/6$ Ackley-5D seeds, and on no Griewank-5D or Branin-2D
seed, and inertness is graded in the size of the inflation too --
$\kappa{=}2$ is inert wherever $\kappa{=}1$ is, while $\kappa{=}5$ is
not. The inert end is also, on this corpus, the end at which no
acquisition function acts at all, so it is evidence about the regime and
not about the OR radius specifically.
\end{remark}

\begin{lemma}[Radius flatness far from data]\label{lem:flat}
Let $k(x,x')=k_0\,\varphi(\norm{x-x'}/\ell)$ with $0\le\varphi\le1$
nonincreasing, and let
$\hat c(x)=\min_{\nu\in[\nu_-,\nu_+]}\sqrt{V_\nu(x)(\eps^2+\eta^2/\nu)}$
be the grid-computed radius from $t$ point evaluations. Then for every $x$
with $\min_i\norm{x-x_i}\ge D$,
\[
k_0\ \ge\ V_\nu(x)\ \ge\ k_0-\tfrac{t\,k_0^2}{\nu}\,\varphi(D/\ell)^2
\qquad\text{for all }\nu,
\]
hence over the far set $\{x:\min_i\norm{x-x_i}\ge D\}$,
\[
\mathrm{osc}(\hat c)\le
\sqrt{\eps^2+\eta^2/\nu_-}\,
\big(\sqrt{k_0}-\sqrt{(k_0-t k_0^2\varphi(D/\ell)^2/\nu_-)_+}\big),
\]
which vanishes as $\varphi(D/\ell)\to0$. (Proof: App.~\ref{app:proofs};
$b(x)_i\le k_0\varphi(D/\ell)$ and $(G+\nu I)^{-1}\preceq I/\nu$.)
\end{lemma}

Together, Prop.~\ref{prop:inv} and Lem.~\ref{lem:flat} predict that
$\kappa$-scaling is inert precisely when all candidates are far from the
data on the kernel's length scale -- the small-budget large-domain
regime. Measured at the final acquisition of Sec.~\ref{sec:caseC} ($22$
observations in hand, selecting the $23$rd), the median over seeds of
$\mathrm{osc}(\hat c)/\mathrm{mean}(\hat c)$ is $0.364$ $[0.349,0.403]$
on Branin-2D, where $\kappa$ changes every trajectory; $0.045$
$[0.032,0.107]$ on Griewank-5D and $0.015$ $[0.008,0.031]$ on Ackley-5D,
where it changes some; and below the resolution of the certifier's $\mu$
grid on Hartmann-6 and Levy-10D, where the whole $\kappa$ ladder produces
bitwise-identical regret trajectories (Rem.~\ref{rem:gridfloor}; that is
regret-trajectory identity -- on Hartmann-6 the \emph{action} path
diverges on one of seven seeds, Table~\ref{tab:Ckappa}). The
median normalized candidate-to-data distances are
$0.17/0.52/0.51/0.63/0.88$ in the order Branin, Griewank, Ackley,
Hartmann-6, Levy. Sec.~\ref{sec:caseC} turns this ordering into a
measured contrast, tests Prop.~\ref{prop:inv} directly at every step, and
states what the ordering is and is not confounded with.

\begin{remark}[The grid floor, and why we do not quote a dispersion
below it]\label{rem:gridfloor}
The certifier evaluates
$\hat c(x)=\min_{\mu\in[\mu_{\mathrm{lo}},\mu_{\mathrm{hi}}]}
\sqrt{V_{1/\mu}(x)(\eps^2+\mu\eta^2)}$ on a finite log grid. When no
candidate is informative -- $k(x,X)\approx0$ -- the objective is
increasing in $\mu$, the minimizer is pinned at the grid's lower end
$\mu_{\mathrm{lo}}$, and there $\hat c(x)\to\hat\eps\sqrt{k(x,x)}$, which
for a stationary base kernel is constant in $x$. The residual
oscillation recorded in that regime is the $O(\mu_{\mathrm{lo}})$
correction to that limit: re-running the same certifier code we find
$\mathrm{osc}(\hat c)/\mathrm{mean}(\hat c)$ proportional to
$\mu_{\mathrm{lo}}$ over six decades and exactly zero for
$\mu_{\mathrm{lo}}\le10^{-16}$, with the optimization unchanged. Values
of order $10^{-11}$ and $10^{-13}$ are therefore properties of the
solver's grid, not of the objective; we report them as ``below the grid
floor'', never as numbers to be compared with each other, and we quote no
dispersion below about $10^{-6}$. Note also that Lem.~\ref{lem:flat}
states its bound at the endpoint $\nu_-=1/\mu_{\mathrm{hi}}$, which is
vacuous at the grid actually used; the operative endpoint for the
measured oscillation is $\nu_+=1/\mu_{\mathrm{lo}}$, and the measured
value is proportional to $1/\nu_+$. The lemma is not contradicted, but it
is not what makes the measured number small.
\end{remark}

\begin{proposition}[Truth coverage from observation coverage]
\label{prop:truth}
In the setting of Thm.~\ref{thm:conformal}, if additionally
$y_{\mathrm{new}}=f(x_{\mathrm{new}})+e_{\mathrm{new}}$ with noise
independent of the band, then for any $z>0$,
$\Pr\big[|f(x_{\mathrm{new}})-\mathrm{mid}_{\mathrm{new}}|\le
\mathrm{half}_{\mathrm{new}}(\kappa_\star)+\hat q+z\big]\ge
1-\tfrac1{n_c+1}-\Pr[|e_{\mathrm{new}}|>z]$.
\emph{Proof.} $|f-\mathrm{mid}|\le|y-\mathrm{mid}|+|e|$; union bound over
the coverage event and $\{|e|\le z\}$.\hfill$\square$
\end{proposition}

\subsection{Proofs}

\textbf{Theorem~\ref{thm:radius}.} \emph{Identities:} with
$w=w_\mu(x)$: $\mathrm{pf}^2+\mu^{-1}\mathrm{wn}^2
=k(x,x)-2w^\top b+w^\top(G+\mu^{-1}I)w=k(x,x)-w^\top b=V_{1/\mu}$;
Woodbury gives $\varphi_x^\top(I+\mu\Lam^*\Lam)^{-1}\varphi_x=V_{1/\mu}$.
\emph{(a) lower bound:} $\pm h$ with $\norm h\le\eps,\norm{\Lam h}\le\eta$
are both consistent with $y=0$, so any $A$ errs $\ge h(x)$ on one:
$E_\star\ge\Omega=\sup h(x)$. \emph{Dual value:} the sup is linear over
two centered ellipsoids (Slater at $0$):
$\Omega=\min_{\lambda_1,\lambda_2\ge0}\tfrac14\varphi_x^\top(\lambda_1I+
\lambda_2\Lam^*\Lam)^{-1}\varphi_x+\lambda_1\eps^2+\lambda_2\eta^2
=\min_\mu\sqrt{V_{1/\mu}(\eps^2+\mu\eta^2)}$ after
$(\lambda_1,\lambda_2)=(s,s\mu)$ and closed-form $s$.
\emph{(b):} $\ell(u)-A_\mu(y)=\ip{\varphi_x-\Lam^*w}{u}-w^\top e$; both
Cauchy--Schwarz suprema are attained simultaneously ($u,e$ independent).
\emph{(c):} $V_{1/\mu}(\eps^2+\mu\eta^2)=\eps^2\mathrm{pf}^2
+\eta^2\mathrm{wn}^2+\mu\eta^2\mathrm{pf}^2+\mu^{-1}\eps^2\mathrm{wn}^2
\ge(\eps\,\mathrm{pf}+\eta\,\mathrm{wn})^2$ by AM--GM, equality iff
$\mu=\eps\,\mathrm{wn}/(\eta\,\mathrm{pf})$; sandwiching
$\Omega\ge\min_\mu[\eps\,\mathrm{pf}_\mu+\eta\,\mathrm{wn}_\mu]\ge
E_\star=\Omega$ forces equality. \hfill$\square$

\textbf{Theorem~\ref{thm:local}.} (a) $\ell(C_y)$ is a compact interval;
the scalar Chebyshev center is the midpoint. (b) Bayes actions under
priors on $C_y$ sweep exactly $\ell(C_y)$ (Dirac priors), so
$\sup_\pi\E_\pi(\ell-a)^2=\max_{t\in\ell(C_y)}(t-a)^2$, minimized at the
midpoint -- the total-ambiguity collapse of \citet{Berger1985}. (c)
Lagrangian duality for $\sup\{c^\top\theta:\theta^\top\theta\le\rho^2,
\norm{\Phi\theta-y}^2\le\eta^2\}$ gives
$\theta^\star=M_{a,b}^{-1}u_{a,b}$ and the displayed dual; strong duality
by interiority; two-quadratic exactness per
\citet{FoucartLiaoSproc,PolikTerlaky2007}. Exact rows are eliminated by
$\theta=\theta_p+N\zeta$. (d) The optimal global rule satisfies
$|\ell(u)-A_\star(y)|\le E_\star$ for all $u\in C_y$, so
$\ell(C_y)\subseteq[A_\star(y)\pm E_\star]$. \hfill$\square$

\textbf{Lemma~\ref{lem:occam}.} Strict convexity of $\norm\cdot^2$ on the
closed convex tube gives a unique minimizer; any element of
$C_y(\rho_{\mathrm{oc}},\eta)$ is a minimizer. For
$\rho>\rho_{\mathrm{oc}}$ a neighborhood of $u_{\mathrm{oc}}$ inside the
tube is feasible. \hfill$\square$

\textbf{Lemma~\ref{lem:flat}.} Each entry $b(x)_i=k(x,x_i)\le
k_0\varphi(D/\ell)$ for $x$ in the far set, so $\norm{b(x)}^2\le t
k_0^2\varphi(D/\ell)^2$ and $b(x)^\top(G+\nu
I)^{-1}b(x)\le\norm{b(x)}^2/\nu$; the oscillation bound follows by
evaluating the two-sided envelope of $\hat c$ at the shared minimizer of
the upper envelope and using $\sqrt{a}-\sqrt{a-b}\le b/(2\sqrt{a-b})$
monotonicity in $\nu\ge\nu_-$. Prop.~\ref{prop:truth} is proved in its
statement. \hfill$\square$

\textbf{Theorem~\ref{thm:conformal}.} Nestedness of
$C_{y_F}(\kappa\rho^F_{\mathrm{oc}})$ in $\kappa$ makes bands nested, so
$y_i\in\mathrm{band}_i(\kappa)\iff s_i\le\kappa$; scores are exchangeable
given $F$, hence $\Pr[s_{\mathrm{new}}>\max_C s_i]\le\tfrac1{n_c+1}$. The
additive variant is split conformal \citep{Vovk2005,Lei2018} on residual
scores; it applies verbatim when the $\kappa$-grid saturates.
\hfill$\square$

\section{Algorithms and complexity}\label{app:algos}

\emph{(1) DiscrepancyCalibrate}: verify misfit-monotonicity endpoints;
bisect $\log\lambda$ (60 steps) to $\norm{\Lam u_\lambda-y}=\eta$; return
$u_\lambda,\rho_{\mathrm{oc}}$. \emph{(2) TightGlobalCertify}: one
eigendecomposition of $G$; per $\mu$ on a 97-point log grid,
$\mathrm{pf},\mathrm{wn}$ spectrally ($O(n)$/point);
$\mathrm{cert}=\min_\mu[\eps\,\mathrm{pf}+\eta\,\mathrm{wn}]
+\delta_{\mathrm{num}}$ (Rem.~\ref{rem:slack}). \emph{(3) LocalInterval}:
eliminate exact rows (SVD); per point, diagonalize the pencil
$(C^\top C,B_{\mathrm{red}})$; minimize the dual over $\log(a,b)$ via
$9{\times}9$ grid pre-search $+$ L-BFGS with analytic gradients. The
returned half-width is the dual value, so weak duality makes it an upper
bound on the true supremum by construction, never optimistic. The
feature-space instantiation additionally returns a projected-primal point,
and the module self-test checks the relative primal--dual gap at tolerance
$10^{-6}$ and passes. The kernel-pencil path used for Case A logs no gap,
so we quote no magnitude there; the feature-space path used for Case B
does log one, and $55$ of its $7129$ delivered test points, in $16$ of
$37$ cells, exceeded the tolerance -- largest recorded relative gap
$0.55$ -- and were handed to \emph{(5)} below. \emph{(4) ConformalRadius}:
$\kappa$-grid $\{1.1,1.3,1.7,2.5,4,6,10,20,40\}$;
$\kappa_\star=$ first $\kappa$ covering all calibration responses;
\emph{always} add the residual floor $\hat q$, which alone guarantees the
achieved level if the grid saturates (saturation is recorded and
reported: $12/12$ \texttt{diabetes} cells in this run; $0/10$ gap
seeds). \emph{(5)
GuardedDelivery}: flag non-finite or
$|\mathrm{mid}-\mathrm{ridge}|>\mathrm{half}+5\max(1,\norm y_\infty)$,
and also flag a relative primal--dual gap above $10^{-6}$; on a flagged
point deliver the global certificate together with the estimate that
certificate belongs to, never the conditional dual value. Costs (single
vCPU, $n\le672$, reduced dim $\le577$): ms per global certificate,
$0.1$--$0.15$\,s per conditional interval.

\section{Experimental details}\label{app:exp}

\begin{table}[t]\centering\footnotesize
\caption{The design rule in one table: the regime label alone is not enough, and
the cheap audit that decides between the two shift rows costs one random forest
and no model fit. Every float the Evidence column cites lives in an appendix
except Tab.~\ref{tab:A}, \ref{tab:Bshift} and Fig.~\ref{fig:B},
\ref{fig:Cmech} (Case~A: App.~\ref{app:caseA}; Case~B: App.~\ref{app:caseB};
Case~C: App.~\ref{app:caseC}).}\label{tab:regime}
\setlength{\tabcolsep}{5pt}
\begin{tabular}{@{}p{1.80in}p{1.85in}p{1.50in}@{}}
\toprule
Data regime & Recommended tool & Evidence\\
\midrule
no exchangeable calibration data & OR ball $+$ Morozov & Sec.~\ref{sec:caseA},
Tab.~\ref{tab:A}, \ref{tab:Astab}, \ref{tab:Aport}\\
exchangeable, in-distribution & split-CP (constant or normalized); the OR
head is valid but wider and scores worse & Tab.~\ref{tab:Barms},
Fig.~\ref{fig:Bp}\\
covariate shift, \emph{learnable} target (RF CV $R^2\gtrsim0.95$) & OR
shape for coverage, and drop the conformal floor, which is a net loss on
interval score; on a proper score a conformal constant can still win, as
split-CP does on \texttt{energy} & Tab.~\ref{tab:Bshift}, \ref{tab:Bgap},
Fig.~\ref{fig:B}\\
covariate shift, poorly learnable target & a wider conformal
\emph{constant}; the OR band's per-point tail costs more than its
coverage buys & Tab.~\ref{tab:Bshift}, Fig.~\ref{fig:Btriv}\\
exploration bonus & floor for validity; shape, not scale -- and check
Prop.~\ref{prop:inv} before scaling & Prop.~\ref{prop:inv},
Fig.~\ref{fig:Cmech}, Tab.~\ref{tab:Ckappa}, \ref{tab:Cinv}\\
any of the above, before trusting a dispersion number & measure the
run-to-run floor first & Tab.~\ref{tab:Afloor}, \ref{tab:Aport}\\
\bottomrule
\end{tabular}
\end{table}

\textbf{Provenance.} All results in Sec.~\ref{sec:caseA}--\ref{sec:caseC}
come from one notebook executed as two sessions whose code cells are
byte-identical apart from four run switches: session 1 ran the invariant
suite, Case C and Case A; session 2 ran Case B. Both used the same
production configuration with zero recorded deviations, which we
confirmed independently from the arm records themselves -- all $39$
Case-B OR entries carry fit-split delivery, refusal on an infeasible
tube, and between $9$ and $23$ dropped dead feature columns, never a
clamp. Total measured compute: $23{,}506$\,s ($6.53$\,h) over $323$ timed
units, no unit retried and none failed. That total counts the ten
gap-task units once per session, because both sessions ran them; charging
them once gives $23{,}402$\,s ($6.50$\,h) over $313$ units. The one block both sessions ran,
the 1D gap task, is reproducible to the bit across the two machines: the
ten per-point archives are byte-identical, and the two result files
differ only in two process-cumulative solver counters. Where this run
supersedes an earlier corpus we say so in the text; the two places where
the two corpora appear side by side are Table~\ref{tab:Aport} and
Table~\ref{tab:Bgapsurv}, and both are labelled as cross-corpus
comparisons. Every other number we take from the earlier corpus is
marked as such at the point of use; we do not give a count, because the
category has two kinds of member and only one of them is a headline.
Some are ``before'' values inside an explicit withdrawal -- the
$13/38$ pre-v5 blow-up census, the $1.08/0.49/2.08$ single-seed gap
widths, the two-seed Griewank median of $14.85$, and the $4.40$ Griewank
regret of an intentionally loose fixed-nugget bound -- and each is read
only against the v5 value beside it. Two support no claim at all: the
$Q{=}7$ Branin fidelity probe (2 seeds) and the or\_full and
or\_ei diagnostic medians ($0.36/1.57/43.6$ on
Branin$/$Hartmann-6$/$Griewank for or\_ei; or\_full worst among plotted
OR variants everywhere).

\textbf{Case A.} TGPS at $24^2/30^2/60^2$ (collocation$/$inducing$/$test)
and at $16^2/22^2/60^2$, 120 epochs, paper hyperparameters; noise seeds
$1$--$5$ at every level including level $0$, which is five re-runs of one
deterministic computation and not five noise draws; ALS initialization
seed fixed at $0$ in all $60$ cells, so the seed-CV measures sensitivity
to the noise draw plus solver non-determinism and never to the
initialization; one PDE forcing only, so no Case-A conclusion here is a
claim about instance-independence; $\eta=\sigma\sqrt{N_c}$,
$\mathrm{std}(f)=678.4431$ at $24^2$ and $692.0062$ at $16^2$; rule
evaluation: Morozov and oracle from run outputs; ML-II and GCV from each
run's final linearization at the misfit-matched nugget (self-consistent,
no cross-run reference), with ML-II minimizing the negative log marginal
likelihood; L-curve from the discrete $(\log\mathrm{misfit},
\log\norm u)$ corner over the grid. Conditional certificates at $24^2$
only: bordered Occam $\rho$, inflation $1.05$; sensitivity over
$1.01$--$2.0$ at three seeds gives coverage $1.000$ at all $30$ (seed,
rung) cells, $0$ refused rungs, width $\le2.40\times$ and width$/$error
$13.6$--$34.4\times$, and at $\eta=0$ the median-width ratio between two
inflations matches the parameter-free constant
$\sqrt{(m_2^2-1)/(m_1^2-1)}$ to $4.1\!\cdot\!10^{-15}$ relative on all
$12$ consecutive pairs and to $5.4\!\cdot\!10^{-15}$ on all $4320$
pointwise pair tests. The certificate block was planned at three noise
levels and delivered two; the session ended inside its loop.

\textbf{Case B.} Five datasets pinned by \texttt{data\_id}, seeds
$\{0,1\}$, $39$ of $73$ planned cells delivered; $80/20$ train$/$test
i.i.d.\ and $70/30$ under shift; head$/$CP split $25\%$ of train, giving
$n_c$ from $53$ (\texttt{yacht} shift) to $300$ (\texttt{california}
i.i.d.); the recorded conformal level is $n_c/(n_c{+}1)$ and the achieved
level is the order statistic the code takes, $(k{+}1)/(n_c{+}1)$ with
$k=\min(n_c{-}1,\lceil0.95(n_c{+}1)\rceil{-}1)$ -- the two are not equal
and the difference is a quantile convention, not a defect
(Table~\ref{tab:Blevels}). Water-filling hint $=$ model's $\hat\sigma$ at
epoch 5; interval score at $\alpha=0.05$. The $2\times2$ ablation is
$30$ shift cells with four head variants built on one trained model per
cell, so it is fully paired and isolates the head. The gap task is ten
seeds, $37$ fit $/$ $13$ calibration on a fixed design, achieved level
$13/14=0.9286$ on $10/10$. That block ran in both sessions two hours
apart under the same driver hash: all ten seeds, all $10\times300$
per-point arrays and every recorded scalar are bit-identical between
them, and the only differing fields are two session-cumulative solver
counters offset by the work the longer session had already done.

\textbf{Case C.} $167$ trajectories from one session; $6$--$7$ seeds per
(benchmark, arm); shared $2500$-point Sobol candidates; $5{+}18$
evaluations; warm-started fits; \texttt{botorch 0.18.1}, $Q{=}3$. The
per-iteration diagnostic series records, at every step of every arm and
seed, $\mathrm{osc}(c)$, $\mathrm{mean}(c)$, the runner-up gap at each
$\kappa$, the argmin index at each $\kappa$, and both the predicted and
the observed invariance flags for all three $\kappa$ pairs; that file is
what Table~\ref{tab:Cinv} is computed from, and the action-level
inertness of Table~\ref{tab:Ckappa} is computed from the argmin indices
rather than from the regret trajectories. Runtimes: $\approx36$\,s per
Case-A \emph{cell} of seven solver runs, i.e.\ $\approx5$\,s per run over
the $420$ runs behind Table~\ref{tab:A} and Table~\ref{tab:A16} (per-cell
median $36.4$\,s, block total $2200.7$\,s); a median of $34.8$\,s and a mean of
$44.6$\,s per Case-B cell over the $39$ timed cells (block total
$1739.2$\,s); and a median of $93.4$\,s per Case-C method-seed.

\clearpage

\section{Case A: full results (TGPS)}\label{app:caseA}

This appendix carries the Case-A protocol in full, the six tables and ten
figures the body defers, and their prose, unchanged.

\subsection{Parameter-choice rules on both collocation grids}

TGPS \citep{TGPS2026} solves $-\Delta u+u^3=f$ on $[0,1]^2$ (Dirichlet) by
a rank-10 tensor GP with sequential linearization; the release fixes the
data-fit weight $\lambda_2=1.49\times10^4$, tuned on noise-free data with
test-RMSE epoch selection. \textbf{Protocol:} five noise seeds
per level, and \emph{all} parameter-choice rules select from the same
shared $\lambda_2$ grid $\{1.49\!\cdot\!10^4,500,50,5,0.5,0.05,0.005\}$:
the released value; Morozov (misfit closest to $\eta=\sigma\sqrt{N_c}$;
noise-level-aware); ML-II and GCV \citep{GolubHeathWahba1979} evaluated on
each run's own final linearization at the misfit-matched nugget
($\sigma$-blind); the L-curve corner \citep{Hansen1992}; and a test-access
oracle on the same grid. Two collocation grids are swept at the same five noise
seeds and the same $120$ epochs: $24^2$ collocation / $30^2$ inducing / $60^2$ test
(the released system's own configuration) and $16^2/22^2/60^2$. That is
$1\times2\times5\times6=60$ cells, $420$ solver runs, no degenerate cell, and a
median of $36.4$\,s per cell. The second grid is not an optional extra: within one
grid the paired unit of the stability test below is the (grid, noise level) group,
so $n=5$ and the smallest attainable two-sided Wilcoxon $p$ is $0.0625$ -- no
single-grid version of that test can reach $0.05$ however large the effect. Two
grids give $n=10$ and a floor of $0.00195$.

\textbf{Results} (Table~\ref{tab:A}, Fig.~\ref{fig:Aprofile}). At the
released system's own grid Morozov reduces noise-time RMSE by
$23$--$47\times$ in the median (per-cell range $13.5$--$135\times$; at
least $100\times$ in $1$ of $25$ cells) relative to the released weight,
and tracks the oracle within $1.00$--$1.19\times$ at every level -- with
no test access, but \emph{with} the known noise level, exactly the
information the released systems' own experiments assume. Every paired
comparison is decisive at $n=25$ noisy cells, where the attainable
minimum two-sided Wilcoxon $p$ is $5.96\!\cdot\!10^{-8}$: on
$\log_{10}$RMSE, Morozov beats the released weight by a median
$-1.536$ dex ($25/25$ cells, $p=5.96\!\cdot\!10^{-8}$, the floor),
ML-II by $-0.370$ ($20/25$, $p=3.8\!\cdot\!10^{-5}$), GCV by $-0.751$
($17$ wins, $3$ losses, $5$ exact ties, $p=4.5\!\cdot\!10^{-4}$) and the
L-curve by $-0.203$ ($21/25$, $p=1.5\!\cdot\!10^{-5}$), and loses to the
test-access oracle by $+0.056$ ($p=2.0\!\cdot\!10^{-4}$;
Fig.~\ref{fig:Aineff}). A Friedman test
over the five deployable rules gives $\chi^2=55.14$, $p=3.0\!\cdot\!10^{-11}$
at $24^2$ and $\chi^2=124.5$, $p=5.8\!\cdot\!10^{-26}$ pooled over both
grids, with Morozov's mean rank $1.58$ $[1.22,2.00]$ and $1.37$
$[1.18,1.60]$. The $\sigma$-blind rules are not substitutes: ML-II and
GCV stick to the released (overfitting) weight at low noise and only
converge to Morozov at the highest levels, and the L-curve is unreliable
throughout, including at $\sigma=0$, where it is the only rule that does
\emph{not} return the released weight and pays a median factor $3.5$ for
it. The honest claim is therefore not ``OR beats classical rules'' but
``the classical rule that OR theory singles out -- the discrepancy
principle -- is the one that works, and wiring it into a modern GP--PDE
stack is a one-line change worth up to two orders of magnitude.''

The noise-scaling exponents make the distinction quantitative. Fitting
$\log_{10}(\text{median RMSE})$ on $\log_{10}(\text{rel.\ noise})$ over
the five noisy levels at $24^2$, Morozov and the oracle share a clean
positive slope -- $\alpha=0.601$ (se $0.066$, $r^2=0.965$) and
$\alpha=0.602$ (se $0.083$, $r^2=0.946$) -- while ML-II and GCV have
\emph{negative} slopes with poor fit ($-0.216$, $r^2=0.227$ and $-0.359$,
$r^2=0.505$): their median RMSE is non-monotone in $\sigma$ because they
are pinned to the released weight at low noise and release it only at
high noise. Each fit has five points, so the load-bearing statement is
the contrast, not the third digit of $\alpha$.

\textbf{The mechanism is one grid step} (Fig.~\ref{fig:Alam2}). Morozov's
median $\lambda_2$ path at $24^2$ runs
$1.49\!\cdot\!10^{4}\!\to\!5\!\to\!0.5\!\to\!0.05\!\to\!0.005\!\to\!0.005$
against the oracle's
$1.49\!\cdot\!10^{4}\!\to\!50\!\to\!5\!\to\!0.5\!\to\!0.05\!\to\!0.005$.
Recall $\lambda_2$ is the \emph{data-fit} weight, so a smaller
$\lambda_2$ is more regularization. Morozov tracks the noise with the
right slope but sits at least one grid step -- one decade, except across
the first pair, where the step is $29.8\times$ -- below the oracle in
$18/25$ noisy cells -- exactly one step in $14$ and two steps in $4$ --
matches it exactly in $7/25$, and is never above it
(median signed offset $+1$ step, IQR $[0,+1]$, Wilcoxon against zero
offset $p=7.9\!\cdot\!10^{-5}$; at $16^2$ the split is $12$ cells at one
step, $4$ at two and $1$ at three, so $17/25$ at $p=1.6\!\cdot\!10^{-4}$).
The residual $1.00$--$1.19\times$ RMSE premium in the last
column of Table~\ref{tab:A} \emph{is} that offset: it is the
discrepancy principle's classical mild over-regularization, not a
failure to sense the noise level. Morozov cannot beat the oracle by
construction -- it equals it exactly in $7/25$ cells and is strictly
worse otherwise.

ML-II selected the released
weight in $7/25$ noisy runs, GCV in $9/25$ (criterion flat at low noise); the
L-curve never did, but its corner is erratic and is wrong even at
$\sigma=0$.

\begin{table}[t]\centering\footnotesize
\caption{Case A at $16^2$ collocation, same protocol and the same five noise
seeds (median test RMSE over 5 seeds). The coarse grid is
discretization-limited rather than noise-limited: the oracle RMSE barely
moves with $\sigma$ (exponent $0.055$) and the whole
$\lambda_2\le5$ half of the grid is nearly flat, which is why every ratio
against Morozov is larger here. The $\sigma=0$ row is the median of five
re-runs; at $16^2$ the released weight diverges on one of them
(RMSE $2.07$), which is why the noise-free gain is not $1$.}
\label{tab:A16}
\setlength{\tabcolsep}{4pt}
\begin{tabular}{@{}lcccccccc@{}}
\toprule
rel.\ noise & released & Morozov & ML-II & GCV & L-curve & oracle &
gain & M/o\\
\midrule
0 & $2.8\!\cdot\!10^{-2}$ & $2.8\!\cdot\!10^{-2}$ & $2.8\!\cdot\!10^{-2}$ & $2.8\!\cdot\!10^{-2}$ & $1.7\!\cdot\!10^{-2}$ & $1.6\!\cdot\!10^{-2}$ & $1.0\times$ [1,97] & 1.71\\
$10^{-4}$ & $1.8\!\cdot\!10^{-1}$ & $1.6\!\cdot\!10^{-2}$ & $1.6\!\cdot\!10^{-1}$ & $1.6\!\cdot\!10^{-1}$ & $2.0\!\cdot\!10^{-2}$ & $1.6\!\cdot\!10^{-2}$ & $11\times$ [7,18] & 1.00\\
$3\!\cdot\!10^{-4}$ & $1.1\!\cdot\!10^{-1}$ & $1.7\!\cdot\!10^{-2}$ & $1.1\!\cdot\!10^{-1}$ & $1.1\!\cdot\!10^{-1}$ & $3.7\!\cdot\!10^{-2}$ & $1.6\!\cdot\!10^{-2}$ & $6.7\times$ [5,14] & 1.04\\
$10^{-3}$ & $4.1\!\cdot\!10^{-1}$ & $1.7\!\cdot\!10^{-2}$ & $2.6\!\cdot\!10^{-1}$ & $2.6\!\cdot\!10^{-1}$ & $1.7\!\cdot\!10^{-2}$ & $1.6\!\cdot\!10^{-2}$ & $23\times$ [16,70] & 1.05\\
$3\!\cdot\!10^{-3}$ & $5.6\!\cdot\!10^{-1}$ & $1.8\!\cdot\!10^{-2}$ & $4.2\!\cdot\!10^{-2}$ & $5.6\!\cdot\!10^{-1}$ & $4.2\!\cdot\!10^{-2}$ & $1.7\!\cdot\!10^{-2}$ & $31\times$ [18,51] & 1.03\\
$10^{-2}$ & $1.0$ & $2.2\!\cdot\!10^{-2}$ & $4.5\!\cdot\!10^{-2}$ & $1.0$ & $1.3\!\cdot\!10^{-1}$ & $2.2\!\cdot\!10^{-2}$ & $42\times$ [37,57] & 1.00\\
\bottomrule
\end{tabular}
\end{table}

\begin{figure}[t]\centering
\includegraphics[width=\textwidth]{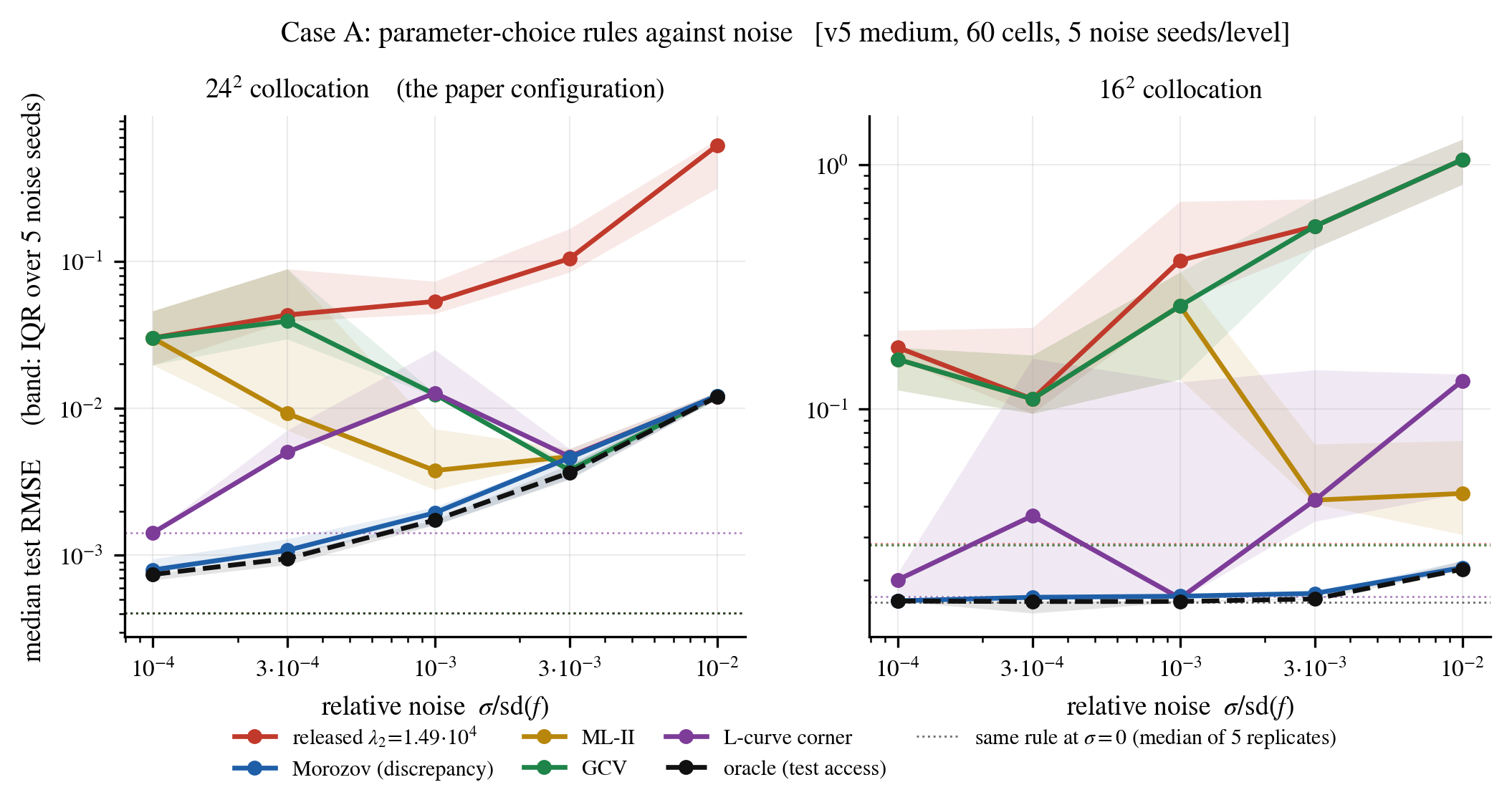}
\caption{Case A parameter-choice rules, both collocation grids. Median
test RMSE against relative noise for all six rules with the interquartile
band over the five noise seeds, log--log, and dotted horizontals marking
each rule's own $\sigma=0$ median. ML-II and GCV are pinned to the
released weight at low noise and release it only at high noise, which is
the non-monotonicity their negative fitted exponents record; Morozov and
the oracle share a clean positive slope.}
\label{fig:Aprofile}
\end{figure}

\begin{figure}[t]\centering
\includegraphics[width=\textwidth]{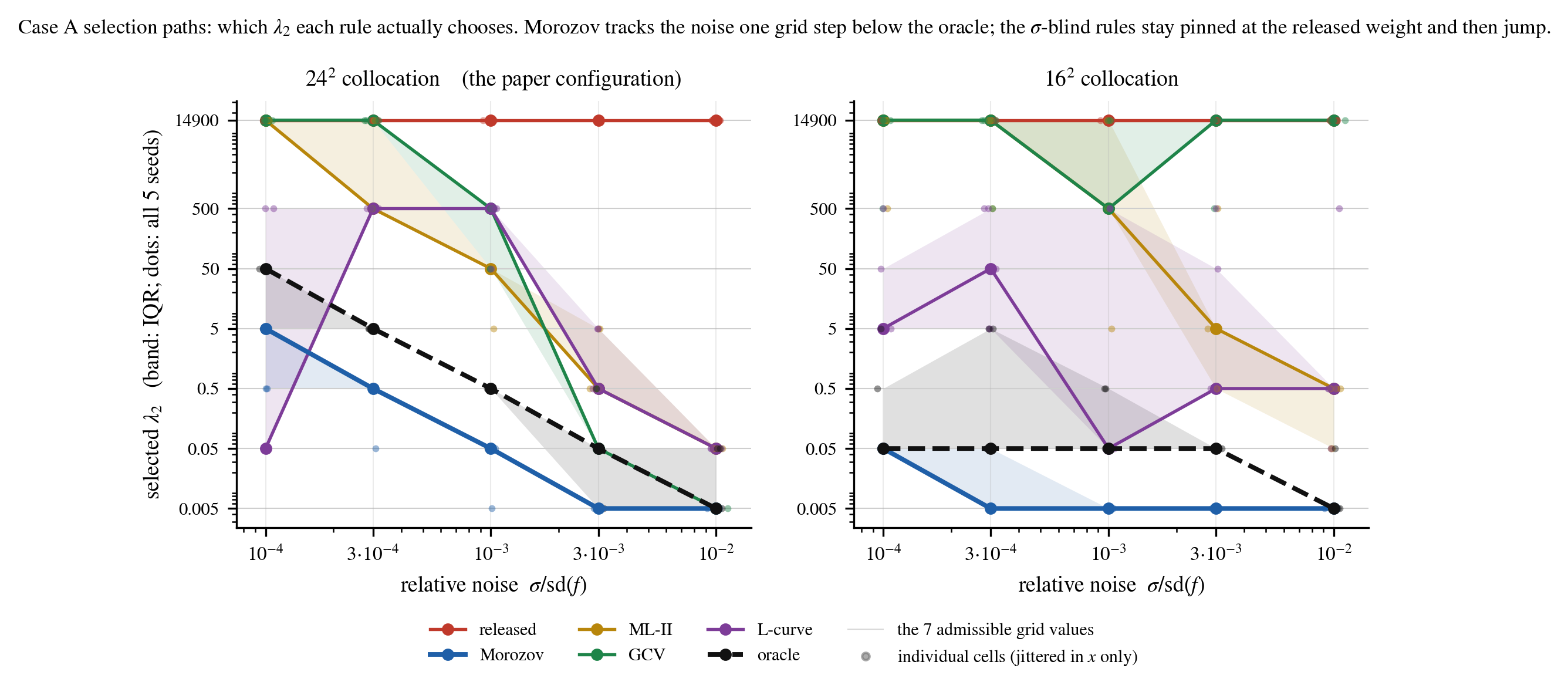}
\caption{Selected data-fit weight against noise level, median with
interquartile band and all five individual cells, log--log; the seven
admissible grid values are drawn as faint horizontals so the offset
between Morozov and the oracle can be counted directly. Morozov sits at
least one grid step below the oracle in $18/25$ noisy cells at $24^2$
(exactly one step in $14$, two in $4$), matches it in $7/25$, and is
never above it.}
\label{fig:Alam2}
\end{figure}

\begin{figure}[t]\centering
\includegraphics[width=\textwidth]{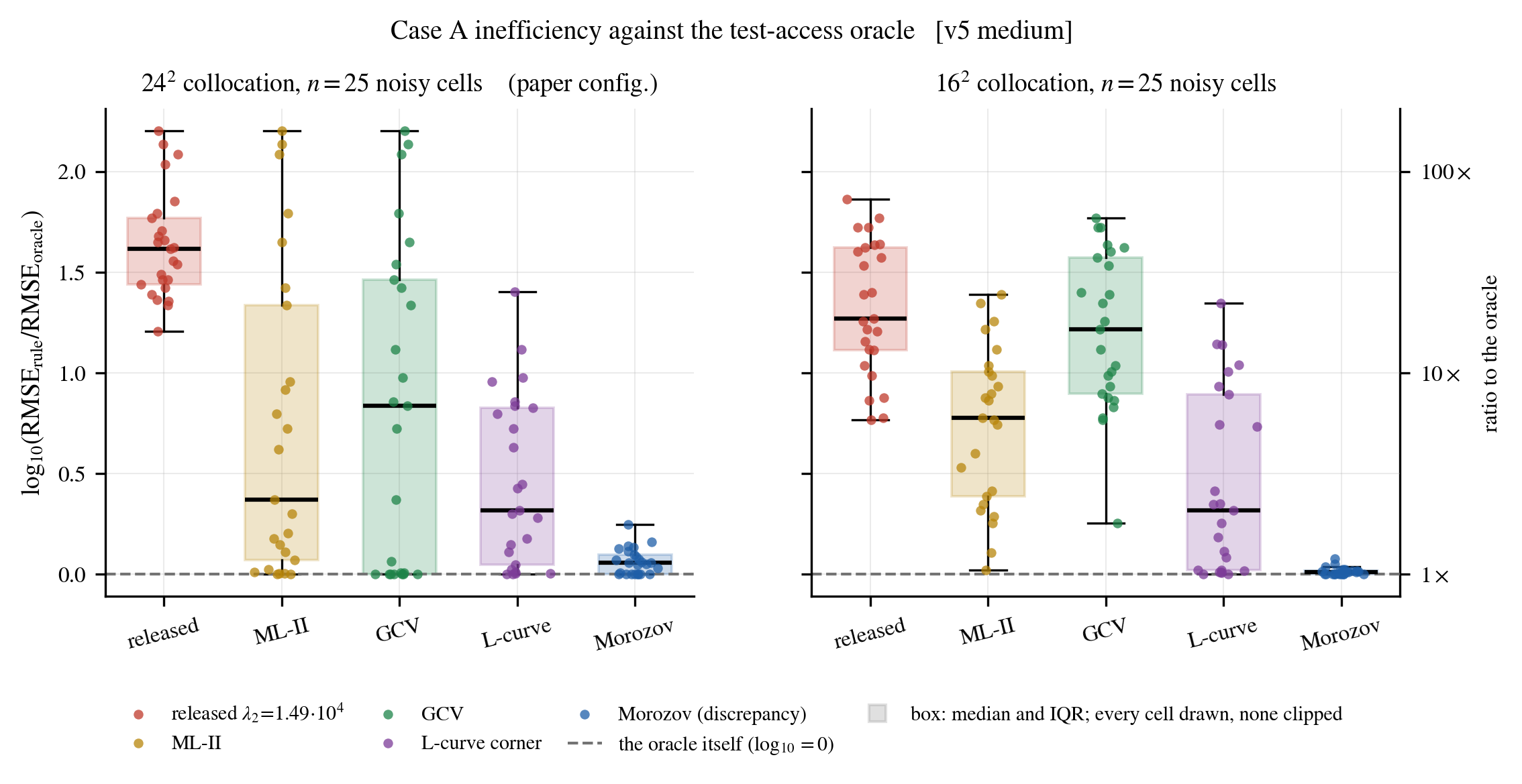}
\caption{$\log_{10}(\text{rule}/\text{oracle})$ per grid, with every one
of the $25$ cells per rule drawn as a jittered point -- no outlier hidden
or clipped -- and a secondary axis converting to a ratio. Median ratios
at $24^2$: released $41.2\times$, GCV $6.89\times$, ML-II $2.35\times$,
L-curve $2.08\times$, Morozov $1.14\times$.}
\label{fig:Aineff}
\end{figure}

\subsection{Run-to-run stability and the non-determinism floor}

\textbf{Run-to-run stability} (Table~\ref{tab:Astab},
Fig.~\ref{fig:Astab}). Accuracy is not the only thing a selection rule
owes a practitioner: rerun on a fresh noise draw, it should return a
comparable model. At the released system's own $24^2$ grid Morozov's
median seed-CV of test RMSE is $0.186$, against $0.447$ (L-curve),
$0.576$ (GCV) and $0.624$ (released weight and ML-II alike) --
$2.4$--$3.4\times$ more reproducible -- and it is strictly lowest among
the five deployable rules at four of the five noisy levels (at rel.\
$10^{-2}$ every rule has converged and Morozov sits $3.3\%$ above the
best). One grid gives only $n=5$ paired groups, whose attainable minimum
two-sided Wilcoxon $p$ is $0.0625$, so that comparison cannot reach
significance however large the effect; this is why the single-grid
version of this table, which is what an earlier revision reported, could
only be read as an effect size. Adding the $16^2$ sweep gives $n=10$
groups, a floor of $0.00195$, and the result then clears $0.05$ against
every alternative: median seed-CV $0.102$ (Morozov) against $0.505$
(released, $p=0.0020$, $10/10$ groups), $0.513$ (GCV, $p=0.0039$),
$0.570$ (L-curve, $p=0.0039$) and $0.639$ (ML-II, $p=0.0039$), i.e.\
$4.9$--$6.3\times$. This is the first properly powered version of the
claim. The two groupings disagree on the size of the band
($2.4$--$3.4\times$ against $4.9$--$6.3\times$), which is itself an
honest measure of how coarse a five-seed CV is; we quote the $24^2$
band in the text and the pooled band wherever significance is asserted.
The $16^2$ ratios ($22$--$55\times$) are an upper bound and we do not
quote them as a headline: that grid is discretization-limited, its
$\lambda_2\le5$ half is nearly flat (RMSE $\max/\min$ over that sub-grid
has median $1.64$ against $2.29$ at $24^2$), so any rule landing at
small $\lambda_2$ there has a near-degenerate CV denominator.

The test-access oracle's median seed-CV is $0.164$ at $24^2$ -- $12\%$
\emph{below} Morozov -- and $0.106$ pooled, $3\%$ above; the paired
difference against Morozov is not significant where it can be tested
($p=0.13$, $n=10$). The honest reading is that per-seed selection
\emph{with} test access buys no reliable reproducibility over the
discrepancy principle in either direction, not that Morozov matches it
to within $1\%$. We also do not read the oracle as a floor on
dispersion -- it minimizes per-seed RMSE, not spread, and a
\emph{constant} weight does better: the best fixed point of the grid at
$24^2$ is now $\lambda_2=0.005$ with median seed-CV $0.125$, a factor
$1.48$ below Morozov. That concession is weaker than it looks. It costs
up to $3.03\times$ Morozov's median RMSE (at rel.\ $10^{-4}$), so it
does not satisfy an accuracy qualifier of $1.7\times$; the next weight
up, $\lambda_2=0.05$, has seed-CV $0.150$ (factor $1.24$) at RMSE within
$1.77\times$; and neither gap is significant ($p=0.25$ at $n=5$,
$p=0.0625$ at $n=10$ with five wins and five ties). Part of the spread
that survives under Morozov is therefore selection slack rather than
solver seed sensitivity -- Morozov's own $\lambda_2$ pick varies across
seeds at rel.\ $10^{-4}$, $3\!\cdot\!10^{-4}$ and $10^{-3}$ and is
unanimous at the two highest levels -- so the reading is that Morozov
buys most of the available reproducibility without needing the
noise-blind rules' luck, not that no weight is steadier. Morozov remains
the lowest-dispersion deployable rule under IQR$/$median, MAD$/$median,
$\max/\min$ and the standard deviation of $\log_{10}$RMSE at the pooled
grouping (strictly lowest in $6$, $7$, $9$ and $8$ of the $10$ groups,
and lowest by median under all five measures). At $24^2$ alone two of
those four fail outright -- on IQR$/$median the medians are GCV $0.134$,
L-curve $0.149$, Morozov $0.254$, and on MAD$/$median the L-curve is
$0.064$ against Morozov's $0.066$ -- so that robustness statement
belongs to the pooled grouping and we scope it there.

\textbf{A floor on any such measurement} (Table~\ref{tab:Afloor}). At
$\sigma=0$ the harness draws no noise and the ALS initialization seed is
fixed at $0$, so the five level-$0$ cells at each grid are five
executions of one computation. They are all distinct: zero bitwise
identical pairs among the five seven-point RMSE curves and among the
five misfit curves, at either grid. The resulting seed-CV runs from
$0.031$ ($\lambda_2=50$) to $0.110$ (the released weight) at $24^2$ and
from $0.003$ to $2.102$ at $16^2$, and at $16^2$ the released-weight
RMSE across the five spans $1.5\!\cdot\!10^{-2}$ to $2.07$, a factor
$139$. That single divergent re-run is why the noise-free gain in
Table~\ref{tab:A16} is not exactly $1$ and why the coarse grid's
$\sigma=0$ Morozov$/$oracle entry there is $1.71$, the one cell of either
grid outside the $1.00$--$1.19\times$ band. The released weight's pure non-determinism CV, $0.110$, is
$59\%$ of Morozov's entire measured noisy seed-CV. Normalizing each
rule's Table~\ref{tab:Astab} value by the floor at its own median
operating point leaves the ordering intact (Morozov $5.16\times$ the
floor, released $5.70$, L-curve $7.02$, GCV $7.66$, ML-II $9.81$) but
shows the absolute CVs are floor-inflated: they are upper bounds on seed
variability, not estimates of it. We report the measurement and not a
diagnosis -- non-deterministic GPU reductions, autotuning and
iteration-count differences under the stop criterion are all consistent
with these data and this corpus cannot separate them. The floor itself
rests on five repetitions at each of seven weights: enough to establish
that it is nonzero and $\lambda_2$-dependent, not enough to pin its
value.

\begin{table}[t]\centering\footnotesize
\caption{Case A run-to-run stability: seed coefficient of variation (CV) of
test RMSE, per rule. Convention: within a (grid, noise level) group, the
sample standard deviation over the five noise seeds (ddof $=1$) divided by
their mean; the \emph{median} rows are medians over groups and the
\emph{ratio} rows divide by Morozov's. Lower is more reproducible; bold
marks the lowest value among the five deployable rules. The oracle column
reads the test set: it is a reference, not a competitor. Five seeds make
each CV a coarse dispersion estimate, so the claim rests on the ordering and
the ratios, not on the third digit. The $\sigma=0$ level is excluded from
the CV because it measures solver non-determinism rather than noise
sensitivity; it is reported separately in Table~\ref{tab:Afloor}.}
\label{tab:Astab}
\setlength{\tabcolsep}{5pt}
\begin{tabular}{@{}lcccccc@{}}
\toprule
rel.\ noise & released & Morozov & ML-II & GCV & L-curve & oracle\\
\midrule
\multicolumn{7}{@{}l}{\emph{$24^2$ collocation (the released system's own grid), $n=5$ groups}}\\
$10^{-4}$ & 0.624 & \textbf{0.124} & 0.624 & 0.624 & 0.534 & 0.132\\
$3\!\cdot\!10^{-4}$ & 0.668 & \textbf{0.150} & 1.223 & 0.882 & 0.447 & 0.164\\
$10^{-3}$ & 0.427 & \textbf{0.242} & 0.923 & 0.576 & 0.607 & 0.129\\
$3\!\cdot\!10^{-3}$ & 0.889 & \textbf{0.192} & 0.200 & 0.199 & 0.200 & 0.207\\
$10^{-2}$ & 0.531 & 0.186 & \textbf{0.180} & 0.186 & \textbf{0.180} & 0.186\\
median & 0.624 & \textbf{0.186} & 0.624 & 0.576 & 0.447 & 0.164\\
ratio to Morozov & $3.36$ & $1.00$ & $3.36$ & $3.10$ & $2.40$ & $0.88$\\
\midrule
\multicolumn{7}{@{}l}{\emph{$16^2$ collocation, $n=5$ groups}}\\
median & 0.379 & \textbf{0.017} & 0.654 & 0.507 & 0.933 & 0.040\\
ratio to Morozov & $22.5$ & $1.00$ & $38.8$ & $30.0$ & $55.3$ & $2.38$\\
\midrule
\multicolumn{7}{@{}l}{\emph{both grids pooled, $n=10$ groups -- the only grouping at which the test can reach $p<0.05$}}\\
median & 0.505 & \textbf{0.102} & 0.639 & 0.513 & 0.570 & 0.106\\
ratio to Morozov & $4.94$ & $1.00$ & $6.26$ & $5.02$ & $5.58$ & $1.03$\\
Wilcoxon $p$ vs Morozov & $0.0020$ & -- & $0.0039$ & $0.0039$ & $0.0039$ & $0.13$\\
groups Morozov lower & $10/10$ & -- & $9/10$ & $9/10$ & $9/10$ & $8/10$\\
\bottomrule
\end{tabular}
\end{table}

\begin{table}[t]\centering\footnotesize
\caption{The non-determinism floor. At $\sigma=0$ the five ``noise seeds''
receive bit-identical data and the ALS initialization seed is fixed at $0$,
so the five runs are repetitions of one computation; they nevertheless
disagree. Seed-CV of test RMSE at each fixed $\lambda_2$, $\sigma=0$, five
repetitions. This is a floor under every entry of Table~\ref{tab:Astab},
it is strongly $\lambda_2$-dependent, and it is largest exactly at the
released weight, the worst-conditioned point of the grid. The last row normalizes each rule's
median noisy seed-CV at $24^2$ by the floor at that rule's own median operating
point: the ordering of Table~\ref{tab:Astab} survives, its absolute magnitudes do
not. No two of the five seven-point RMSE curves coincide at either grid.}
\label{tab:Afloor}
\setlength{\tabcolsep}{5pt}
\begin{tabular}{@{}lccccccc@{}}
\toprule
$\lambda_2$ & $1.49\!\cdot\!10^{4}$ & $500$ & $50$ & $5$ & $0.5$ & $0.05$ & $0.005$\\
\midrule
$24^2$ & 0.110 & 0.075 & 0.031 & 0.042 & 0.033 & 0.064 & 0.036\\
$16^2$ & 2.102 & 1.089 & 0.192 & 0.021 & 0.006 & 0.003 & 0.006\\
\midrule
& released & Morozov & ML-II & GCV & L-curve & oracle &\\
Tab.~\ref{tab:Astab} median $/$ own floor & $5.70$ & $5.16$ & $9.81$ & $7.66$ & $7.02$ & $4.55$ &\\
\bottomrule
\end{tabular}
\end{table}

\begin{figure}[t]\centering
\includegraphics[width=\textwidth]{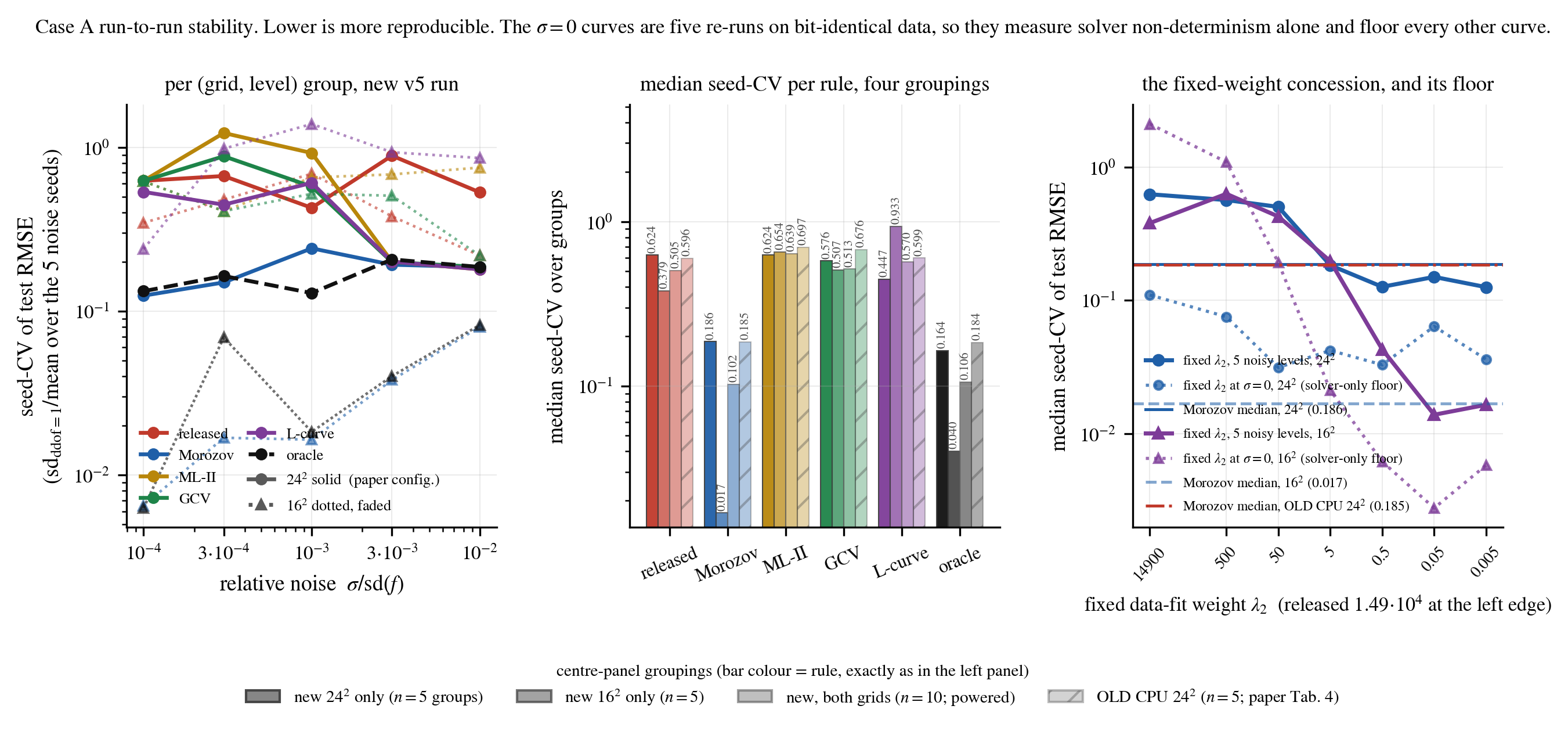}
\caption{Case A run-to-run stability. Left: per-(grid, level) seed-CV for
all six rules at both grids, log $y$. Centre: median seed-CV per rule
under all four groupings -- new $24^2$, new $16^2$, new pooled ($n=10$,
the powered one) and the earlier CPU corpus. Right: the fixed-weight
sweep against $\lambda_2$ with the $\sigma=0$ solver-only floor overlaid,
which is the only place the fixed-weight concession and the
non-determinism floor can be read together. The $\sigma=0$ curves are
five re-runs on bit-identical data, so they measure solver
non-determinism alone and floor every other curve.}
\label{fig:Astab}
\end{figure}

\subsection{Cross-backend portability}

\textbf{The rule that survives a change of machine}
(Table~\ref{tab:Aport}). The same $25$ noisy $24^2$ cells were computed
on a different backend in an earlier corpus, at matched noise seed and
level, with identical $\mathrm{std}(f)$, identical $\eta$ and an
identical $\lambda_2$ grid; the per-$\lambda_2$ misfit ratio between the
two has median $0.998$, so the data are the same and only the solve
differs. Over the $26$ shared cells -- the $25$ noisy ones plus the
shared $\sigma=0$ cell -- only $6.6\%$ of the $182$ matched
(cell, $\lambda_2$) RMSE pairs agree to $1\%$, and $63\%$ agree to
$20\%$. Against that
background the rules separate by an order of magnitude: the released
weight's per-cell RMSE moves by up to $12.3\times$ between machines,
ML-II by $30.4\times$, GCV by $12.3\times$ and the L-curve by
$4.9\times$, while \emph{Morozov moves by at most $1.36\times$} and the
test-access oracle by $1.35\times$. Decomposing the drift of the gain
column of Table~\ref{tab:A}, the released numerator carries $106\%$ of
the log-variance and Morozov's denominator $2\%$: the gain number is
unstable across machines precisely because the released weight is, and
this is the same ill-conditioning that Table~\ref{tab:Afloor} measures
at $\sigma=0$. Resampling each of the $25$ cells independently from
either backend and recomputing Table~\ref{tab:Astab} $4000$ times,
Morozov is the lowest-dispersion deployable rule in $4000/4000$ draws,
its own median seed-CV stays in $[0.174,0.195]$, and the
released$/$Morozov ratio is $[2.58,4.37]$. Individual entries of
Table~\ref{tab:Astab} are \emph{not} portable -- across the two corpora
they move by factors $0.19$ to $1.49$, $17\%$ of them by more than
$1.5\times$ -- and neither is the low-noise end of the gain column
($26\times\!\to\!40\times$ at rel.\ $10^{-4}$, minimum per-cell gain
$5.3\!\to\!13.6$; Fig.~\ref{fig:Again}). What is portable is the ordering, the bold cell at
every level, Morozov's own value ($0.1855\!\to\!0.1859$), the
released$/$Morozov ratio ($3.21\!\to\!3.36$), the worst level's ratio of
median Morozov and oracle RMSEs ($1.269\!\to\!1.265$; the
median-of-per-cell-ratios statistic that
Table~\ref{tab:A}'s M$/$o column reports moves
$1.257\!\to\!1.186$ on the same cells) and the noise exponent
($0.614\!\to\!0.601$). We therefore state the Case-A stability result as
an ordering with a ratio band and a portability ledger, and we do not
quote any single seed-CV entry as reproducible.

\begin{table}[t]\centering\footnotesize
\caption{Cross-backend portability. The corpus of this paper (GPU
session) and an earlier corpus computed on CPU share $25$ noisy
$24^2$ cells at matched noise seed and level: same forcing, same
$\mathrm{std}(f)=678.4431$, same $\eta=16.2826$ at rel.\ $10^{-3}$, same
$\lambda_2$ grid, same $120$ epochs. Only the backend differs. Column
two is the per-cell ratio of the two corpora's test RMSE for that rule,
so $1.00$ would be exact reproduction. This is one of only two tables in
the paper that mix corpora, because the corpus is its variable. The data
are the same on both machines -- the per-$\lambda_2$ misfit ratio has
median $0.998$ -- so the spread is the solve, not the draw. Over the
$26$ shared cells (the $25$ noisy ones plus $\sigma=0$) only $6.6\%$
of the $182$ matched (cell, $\lambda_2$) RMSE pairs agree to $1\%$ and
only $35.7\%$ to $5\%$; Morozov and the oracle are the two rules whose
\emph{selected model} nonetheless survives the move.}
\label{tab:Aport}
\setlength{\tabcolsep}{6pt}
\begin{tabular}{@{}lccc@{}}
\toprule
rule & median ratio & range over the 25 cells & worst disagreement\\
\midrule
released & 1.180 & $[0.373,\,12.34]$ & $12.3\times$\\
ML-II & 0.998 & $[0.221,\,30.38]$ & $30.4\times$\\
GCV & 0.988 & $[0.101,\,12.34]$ & $12.3\times$\\
L-curve & 0.980 & $[0.221,\,4.950]$ & $4.9\times$\\
\textbf{Morozov} & \textbf{0.994} & $\mathbf{[0.796,\,1.357]}$ & $\mathbf{1.36\times}$\\
oracle & 0.994 & $[0.772,\,1.351]$ & $1.35\times$\\
\bottomrule
\end{tabular}
\end{table}

\begin{figure}[t]\centering
\includegraphics[width=\textwidth]{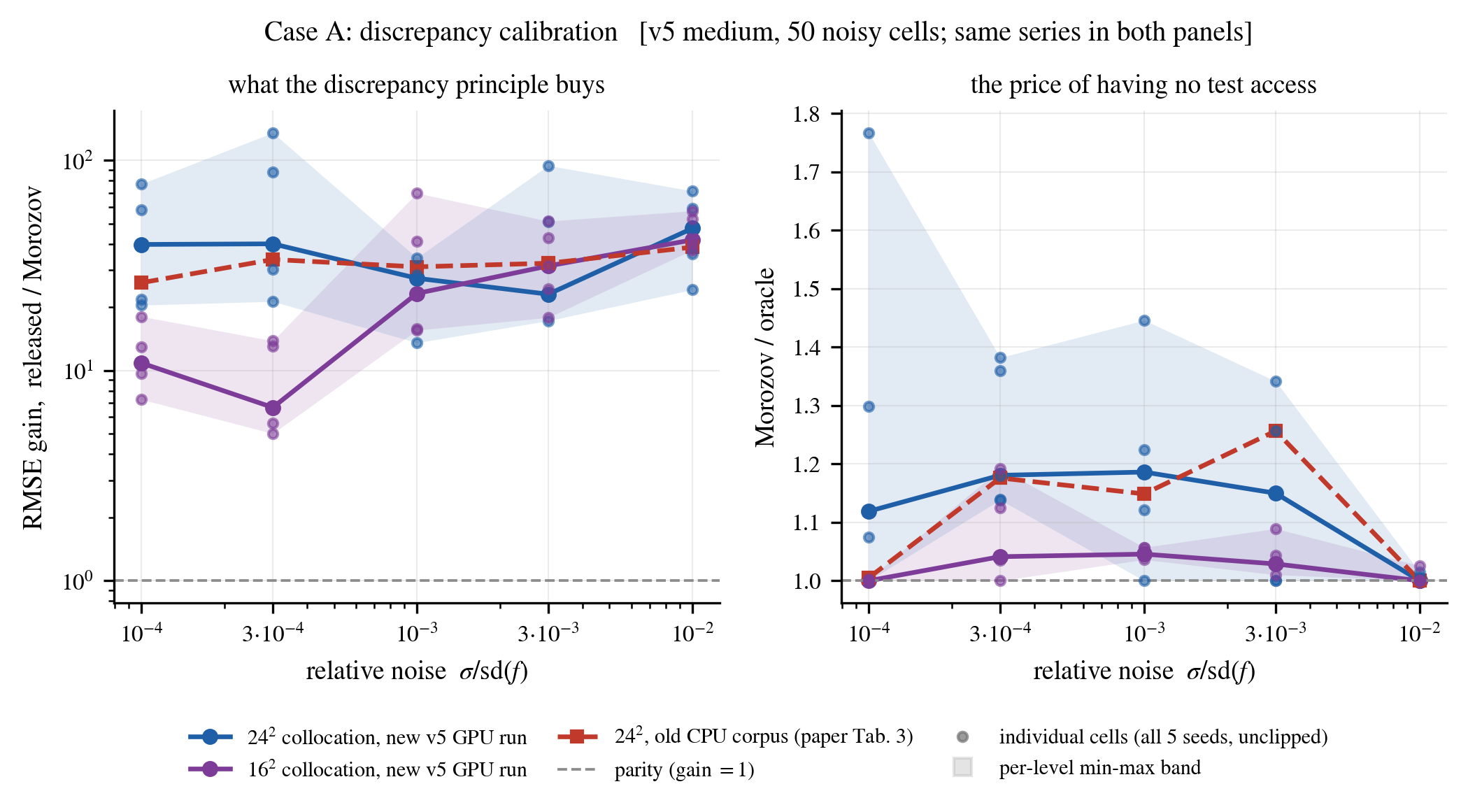}
\caption{Left: released$/$Morozov gain against noise, log--log, per
grid, with the per-level min--max band and every one of the $50$ cells
drawn unclipped; the earlier CPU corpus is overlaid dashed. Right:
Morozov$/$oracle on a linear axis, whose range $1.0$--$1.8$ spans no
decade. The two corpora agree in aggregate and disagree cell by cell;
Table~\ref{tab:Aport} quantifies which is which.}
\label{fig:Again}
\end{figure}

\subsection{Certificates and the inflation ladder}

\textbf{Certificates} (Table~\ref{tab:Acert}, Fig.~\ref{fig:A}).
Coverage of the true solution is $1.0$ at both delivered noise levels
and at every rung of the inflation ladder, with no numerical slack
needed: $0$ violations in $432$ evaluation points per level, and the
binding point still holds $31$--$33\%$ of the median certificate in
reserve (Rem.~\ref{rem:slack}). With $144$ points per cell an empirical
coverage resolves nothing finer than $1/144=0.0069$, so the honest
statement is a $95\%$ Clopper--Pearson lower bound of $0.993$ on the
pooled $0/432$. The delivered global certificate coincides with the
information radius, but as a five-digit invariant rather than an
identity: over all $864$ points the ratio $\mathrm{cert}/E_\star$ has
median $0.999989$, IQR $[0.999964,0.999996]$, worst value $0.997273$,
and never exceeds $1$; $99.9\%$ of points lie within $10^{-3}$ of the
identity, $87.5\%$ within $10^{-4}$ and only $45.6\%$ within $10^{-5}$
(Fig.~\ref{fig:Aattain}). The residual shortfall is the $97$-point $\mu$
grid, whose upper end $\mu=10^{12}$ is selected at every $\sigma=0$
point, and not a failure of Thm.~\ref{thm:radius}c; we therefore use
Thm.~\ref{thm:radius}c as an implementation invariant at five
significant digits and not as an exact equality.

At the delivered $\rho=1.05\rho_{\mathrm{oc}}$ the $y$-conditional band
is a further $104.6\times$ tighter than the global certificate at rel.\
$10^{-3}$ ($[100.0,108.2]$ over three seeds) and $3.10\times$ tighter at
$\sigma=0$ -- but the $\sigma=0$ figure carries no information. At
$\eta=0$ the band is exactly
$\rho_{\mathrm{oc}}\sqrt{m^2-1}\,P(x)$ and the tightening is the known
function $(\hat\eps/\rho_{\mathrm{oc}})/\sqrt{m^2-1}$ of the inflation
alone, which the measurement reproduces to $0.6$--$1.0\%$ and which
drops below $1$ for $m>\sqrt2$: at $\sigma=0$ the conditional band is
\emph{wider} than the global certificate at $m=1.5$ ($0.889$) and at
$m=2.0$ ($0.574$). The premium the conditional band pays over the actual
error is stable but heavy-tailed: at rel.\ $10^{-3}$ the median
pointwise width$/$error ratio is $21.5\times$ ($[19.7,22.0]$ over three
seeds) at $\rho=1.05\rho_{\mathrm{oc}}$ and ranges over
$13.6$--$34.4\times$ across the whole ladder
$\rho\in[1.01,2.0]\rho_{\mathrm{oc}}$, with a right tail reaching
$3.8\!\cdot\!10^{3}$ at points where the recovery is accidentally exact
(Fig.~\ref{fig:Awer}), so only medians are quoted and the quoted range
is a range of medians, not of points. \emph{Sensitivity:} the ball
inflation is not load-bearing -- over that ladder coverage stays
$1.000$, no rung is refused, and the median width varies by $2.37\times$
across the full range (per-seed $2.14$, $2.37$, $2.40$; measured at
rel.\ $10^{-3}$, Table~\ref{tab:Aladder},
Fig.~\ref{fig:Aladder}). We quote the two tightening
stages separately and do not multiply them.

The two radii that could anchor the ladder are not two estimates of one
number. The bordered-Occam radius $\rho_{\mathrm{oc}}$ minimizes
$\norm u$ subject to the boundary rows \emph{exact} and the collocation
residual within $\eta$; $\hat\eps$ minimizes it subject to
$\norm{\Lam u-y}\le\eta$ over \emph{all} rows. The second feasible set
strictly contains the first, so $\hat\eps\le\rho_{\mathrm{oc}}$ exactly
as observed, with a ratio stable to $0.3\%$ across seeds
($0.7926$, $0.7941$, $0.7917$). At rel.\ $10^{-3}$ they differ by
$26.17\%$; at $\sigma=0$, where the two sets coincide, they agree to
$0.063\%$. Two consequences we state rather than hide: any ladder rung
finer than the gap that applies at that noise level is inside the
specification uncertainty of its own radius, and the global certificate
is certified over the ball $\norm u\le18.96$ while the conditional
ladder is anchored at $23.91$, a $26\%$ larger ball -- Lem.~\ref{lem:occam}
applies to both, but they are not the same model ball and we do not
present them as such. The band also shrinks, rather than vanishes, where
the observations are exact: on the ring of evaluation points $0.05$ from
the Dirichlet boundary the conditional half-width at rel.\ $10^{-3}$ is
$0.66\times$ its interior median while the global certificate is
$2.8\times$ its interior median there, and at $\sigma=0$ the two
profiles agree and both are largest near the boundary. The effect is one
of the $\eta$ ball, not of exactness alone, and the evaluation grid
never reaches the boundary, so ``vanishes on the boundary'' is not a
statement this corpus can make.

The natural Bayesian comparator -- the credible band of the
misfit-matched GP on the same linearization with ML-II scale -- is
two-sided. At $\sigma=0$ it covers ($1.000$) at $3.3\times$ smaller
width than the conditional OR band. At rel.\ $10^{-3}$ it undercovers
its own posterior-mean error ($0.833$, i.e.\ $120/144$) at $14.8\times$
the OR width, and the failure is spatially deterministic rather than
stochastic: it fails at exactly the same $24$ points in all three seeds,
every one of them on the two rings nearest the Dirichlet boundary, where
the OR construction imposes the boundary rows exactly and its band is
$0.016$--$0.024$ wide against errors of $0.0002$--$0.004$
(Fig.~\ref{fig:Amaps}). Average-case bands are sharper exactly when the
model is right, and brittle the moment it is not. Conformal calibration
is \emph{not} applicable here (collocation points are not exchangeable
draws) -- precisely the regime where the worst-case ball is the only
guarantee available. Two scope limits: this block has two noise levels,
not three, so nothing here speaks to the high-noise end of the
certificate ladder; and with three seeds per level the smallest
attainable two-sided $p$ of any seed-level test is $0.25$, so every
across-seed statement above is descriptive.

\begin{table}[t]\centering\footnotesize
\caption{Case A certificates, per noise level: 3 seeds, 144 evaluation
points per seed (432 pooled). Coverage is of the true solution; widths are
medians over points, then medians over seeds; brackets are the seed range.
$m=\rho/\rho_{\mathrm{oc}}$. The GP band fails at the same 24 points in all three seeds. $\rho_{\mathrm{oc}}$ is the boundary-exact Occam radius used for the
conditional band, $\hat\eps$ the all-rows radius used for the global
certificate. At $\sigma=0$ the three seeds share one dataset -- the generator
leaves the data untouched when $\sigma=0$ -- so they differ only through solver
non-determinism and the effective $n$ there is one dataset; at rel.\ $10^{-3}$ they
are three independent noise draws. Level rel.\ $10^{-2}$ was planned and is absent:
the session ended inside this loop.}\label{tab:Acert}
\begin{tabular}{@{}lrr@{}}
\toprule
& $\sigma=0$ & rel.\ $10^{-3}$\\
\midrule
cells (seeds) $\times$ points & $3\times144=432$ & $3\times144=432$\\
$\rho_{\mathrm{oc}}$ (source) & $25.0335$ (closed form) & $23.914$ (bordered Occam)\\
$\hat\eps$ (all rows within $\eta$) & $25.0494$ & $18.957$\\
$|\hat\eps-\rho_{\mathrm{oc}}|/\hat\eps$ & $0.063\%$ & $26.17\%$\\
\midrule
raw coverage (no slack) & $1.0000$ ($0/432$) & $1.0000$ ($0/432$)\\
$\max_x(\mathrm{err}-\mathrm{cert})$ & $-5.10\cdot10^{-3}$ & $-8.84\cdot10^{-1}$\\
$\delta_{\mathrm{num}}$ (\% of median cert) & $3.73\cdot10^{-6}$ ($0.024\%$) & $2.82\cdot10^{-6}$ ($10^{-4}\%$)\\
attainment $\mathrm{cert}/E_\star$, median & $0.999994$ & $0.999987$\\
\quad IQR over points & $[0.999985,0.999998]$ & $[0.999926,0.999991]$\\
\quad worst point & $0.997273$ & $0.999848$\\
\midrule
median global certificate & $1.528\cdot10^{-2}$ & $2.876$\\
median conditional half, $m{=}1.05$ & $4.923\cdot10^{-3}$ & $2.748\cdot10^{-2}$\\
conditional coverage (all 5 rungs) & $1.000$ & $1.000$\\
refused rungs & $0/15$ & $0/15$\\
tightening (global / conditional) & $3.10\times$ & $104.6\times$\\
width/error at $m{=}1.05$ & $14.4\times$ & $21.5\times$ $[19.7,22.0]$\\
\midrule
GP credible band: coverage & $1.000$ & $0.833$ ($120/144$)\\
GP credible band: median width & $1.481\cdot10^{-3}$ & $0.406$\\
GP width / conditional width & $0.30\times$ & $14.8\times$\\
\bottomrule
\end{tabular}
\end{table}

\begin{table}[t]\centering\footnotesize
\caption{The $\rho$-inflation ladder, $\rho=m\,\rho_{\mathrm{oc}}$. Each
entry is the median over 3 seeds of the within-cell median over 144 points;
brackets give the seed range where it is wider than the last digit shown.
Coverage of the truth is $1.000$ at every one of the 30 (level, seed, rung)
cells and no rung was refused. At $\sigma=0$ the width column obeys the
parameter-free law $\mathrm{medw}(m)\propto\sqrt{m^2-1}$ exactly
(Sec.~\ref{sec:caseA}). The $\eta=0$ law is verified to $4.1\!\cdot\!10^{-15}$
relative on all $12$ consecutive pairs and to $5.4\!\cdot\!10^{-15}$ on all $4320$
pointwise pair tests.}\label{tab:Aladder}
\begin{tabular}{@{}lrrr@{\hskip 18pt}rrr@{}}
\toprule
& \multicolumn{3}{c}{$\sigma=0$} & \multicolumn{3}{c}{rel.\ $10^{-3}$}\\
\cmidrule(r){2-4}\cmidrule(l){5-7}
$m$ & median half & width/err & tighten & median half & width/err & tighten\\
\midrule
$1.01$ & $2.180\cdot10^{-3}$ & $6.39$  & $7.01$   & $1.960\cdot10^{-2}$ & $13.7$ & $146.7$\\
$1.05$ & $4.923\cdot10^{-3}$ & $14.42$ & $3.10$   & $2.748\cdot10^{-2}$ & $21.5$ & $104.6$\\
$1.1$  & $7.047\cdot10^{-3}$ & $20.64$ & $2.17$   & $3.031\cdot10^{-2}$ & $23.1$ & $94.9$\\
$1.5$  & $1.719\cdot10^{-2}$ & $50.36$ & $0.889$  & $4.142\cdot10^{-2}$ & $28.1$ & $69.4$\\
$2.0$  & $2.664\cdot10^{-2}$ & $78.01$ & $0.574$  & $4.705\cdot10^{-2}$ & $33.5$ & $61.1$\\
\midrule
$m{=}2.0$ / $m{=}1.01$ & $12.22\times$ & & & $2.37\times$ $[2.14,2.40]$ & &\\
\bottomrule
\end{tabular}
\end{table}

\begin{figure}[t]\centering
\includegraphics[width=\textwidth]{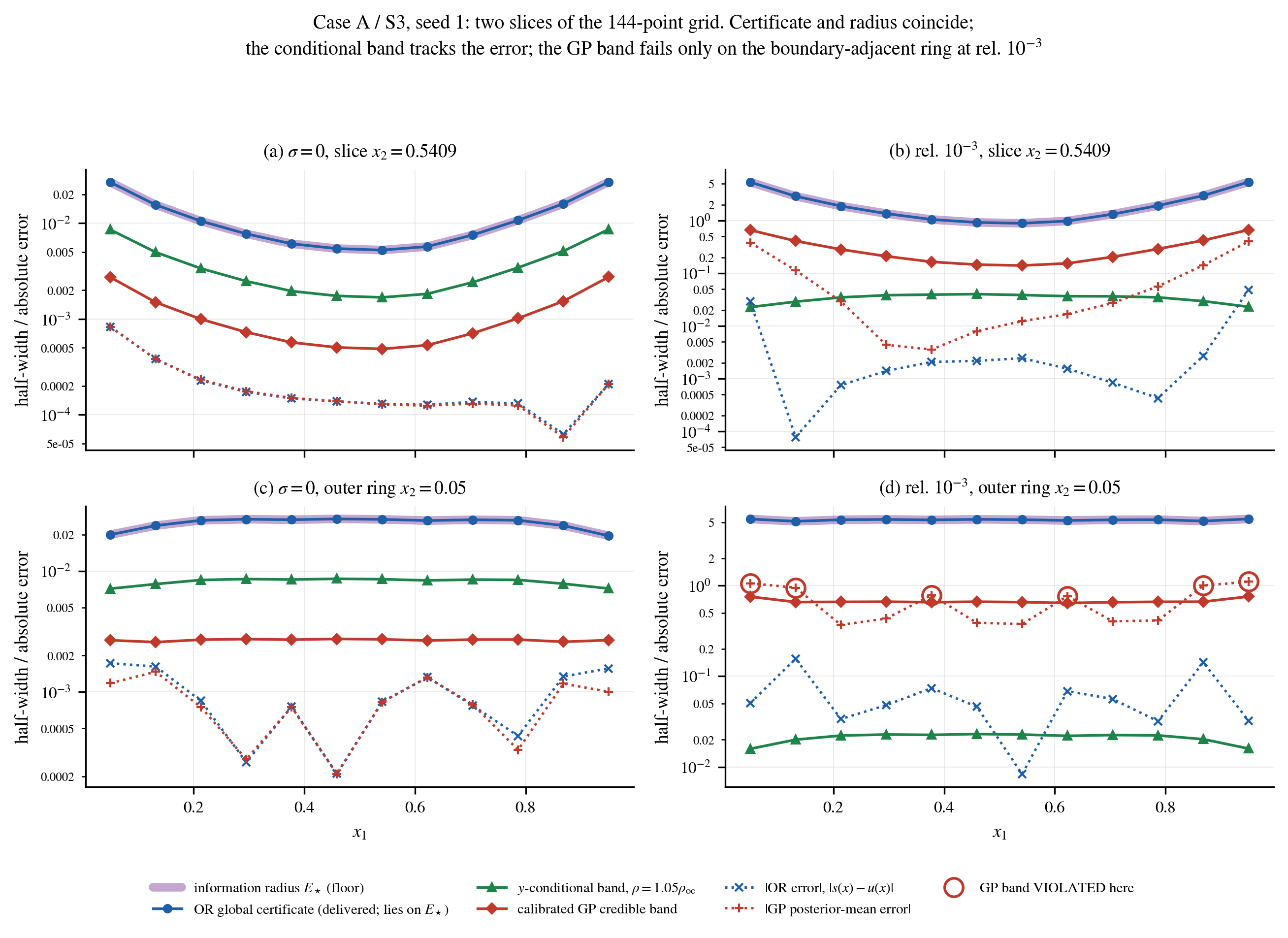}
\caption{Case A certificates, seed 1: two slices of the evaluation grid,
through the middle ($x_2=0.5409$) and along the boundary-adjacent ring
($x_2=0.05$), each at $\sigma=0$ and at rel.\ $10^{-3}$. The delivered
global certificate lies on the information radius; the $y$-conditional
band tracks the actual error; the calibrated GP band is narrower at
$\sigma=0$ and both wider and invalid at rel.\ $10^{-3}$, where it is
violated at the circled points. The fixed-$\mu$ two-term baseline of an
earlier revision is not stored in this run's per-point arrays and is not
drawn.}
\label{fig:A}
\end{figure}

\begin{figure}[t]\centering
\includegraphics[width=\textwidth]{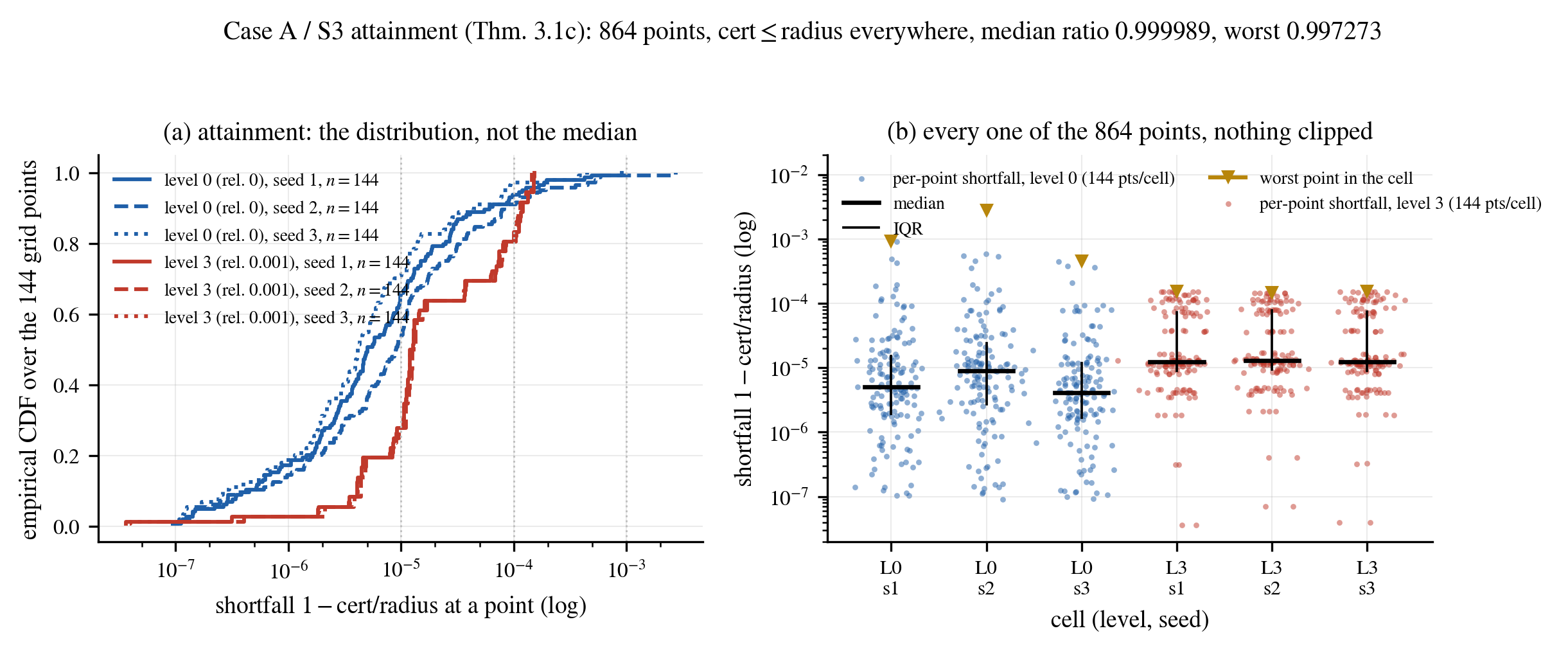}
\caption{Attainment as a distribution, not a median. Left: empirical CDF
of the pointwise shortfall $1-\mathrm{cert}/E_\star$ per (level, seed),
log axis. Right: all $864$ points, with median, IQR and the worst point
of each cell marked; nothing is clipped. The ratio never exceeds $1$, its
median is $0.999989$, and its worst value is $0.997273$.}
\label{fig:Aattain}
\end{figure}

\begin{figure}[t]\centering
\includegraphics[width=\textwidth]{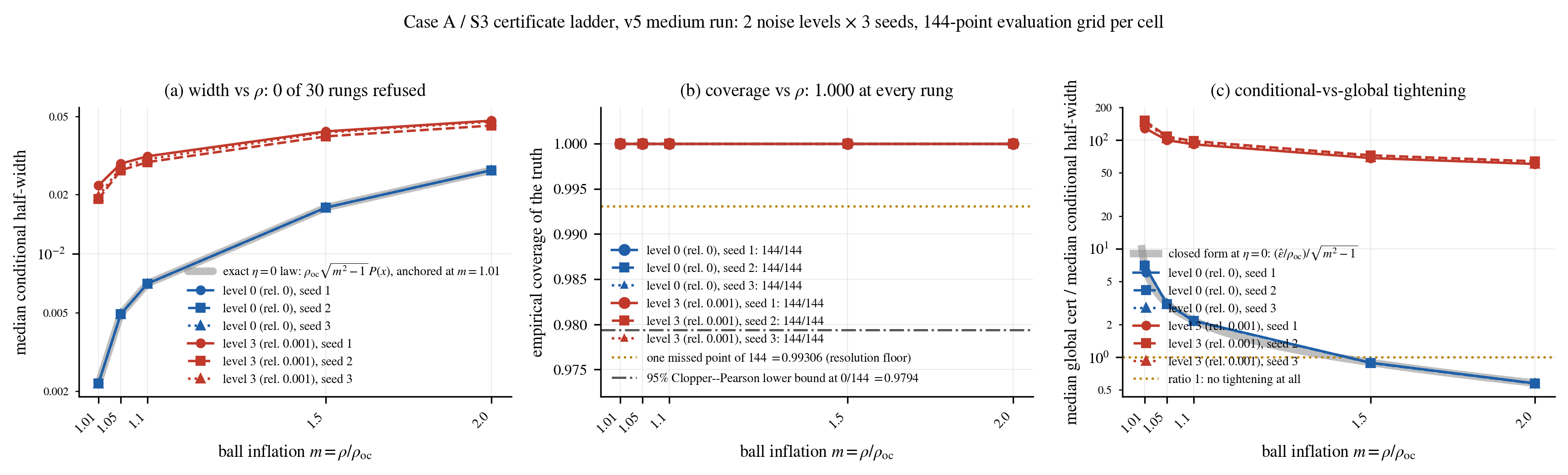}
\caption{The $\rho$-inflation ladder. (a) median conditional half-width
against $m=\rho/\rho_{\mathrm{oc}}$ with the exact $\eta=0$ law
overlaid; (b) coverage of the truth, $1.000$ at every rung, against the
resolution floor of a $144$-point grid; (c) the conditional-vs-global
tightening, which at $\sigma=0$ follows
$(\hat\eps/\rho_{\mathrm{oc}})/\sqrt{m^2-1}$ and crosses $1$ at
$m=\sqrt2$. Six series: two noise levels by three seeds; no rung was
refused.}
\label{fig:Aladder}
\end{figure}

\begin{figure}[t]\centering
\includegraphics[width=\textwidth]{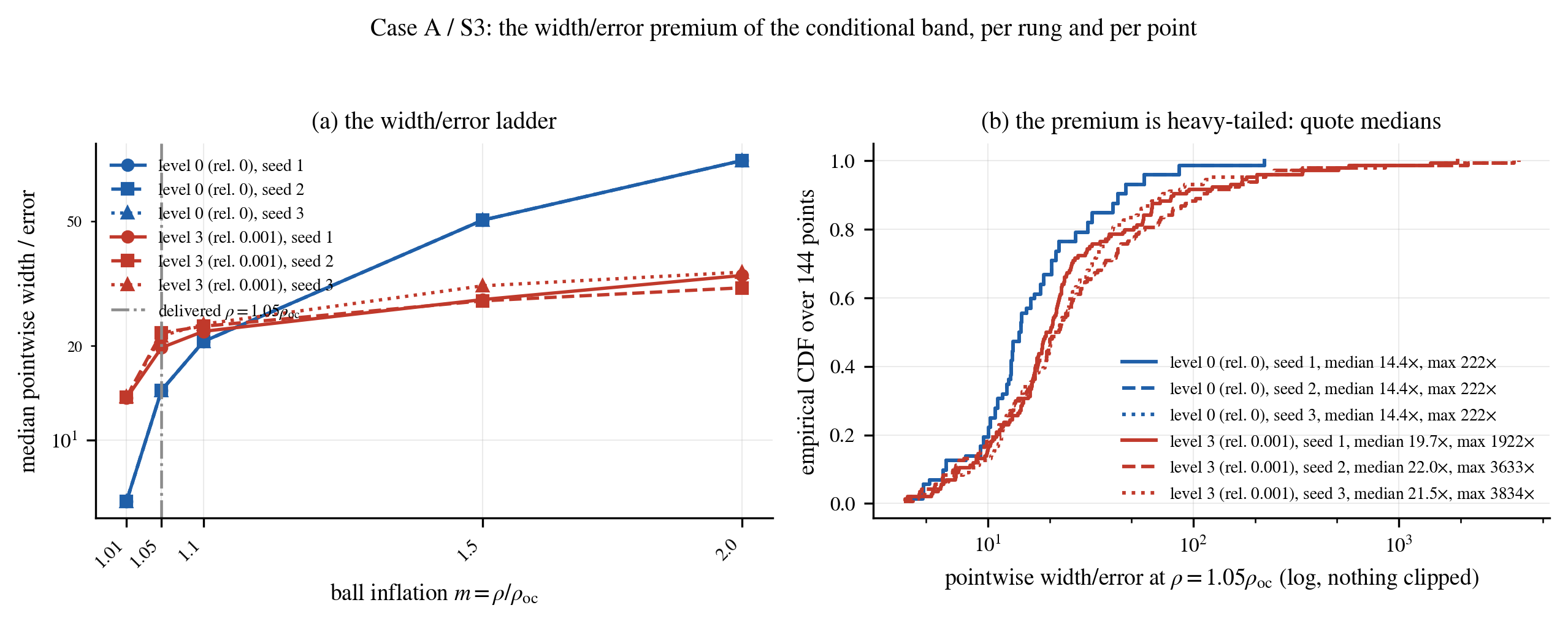}
\caption{The width$/$error premium of the conditional band. Left: the
median premium per rung, six series. Right: its pointwise distribution
at the delivered $\rho=1.05\rho_{\mathrm{oc}}$, whose right tail reaches
$3.8\!\cdot\!10^{3}$ where the recovery is accidentally exact -- which is
why only medians are quoted.}
\label{fig:Awer}
\end{figure}

\begin{figure}[t]\centering
\includegraphics[width=\textwidth]{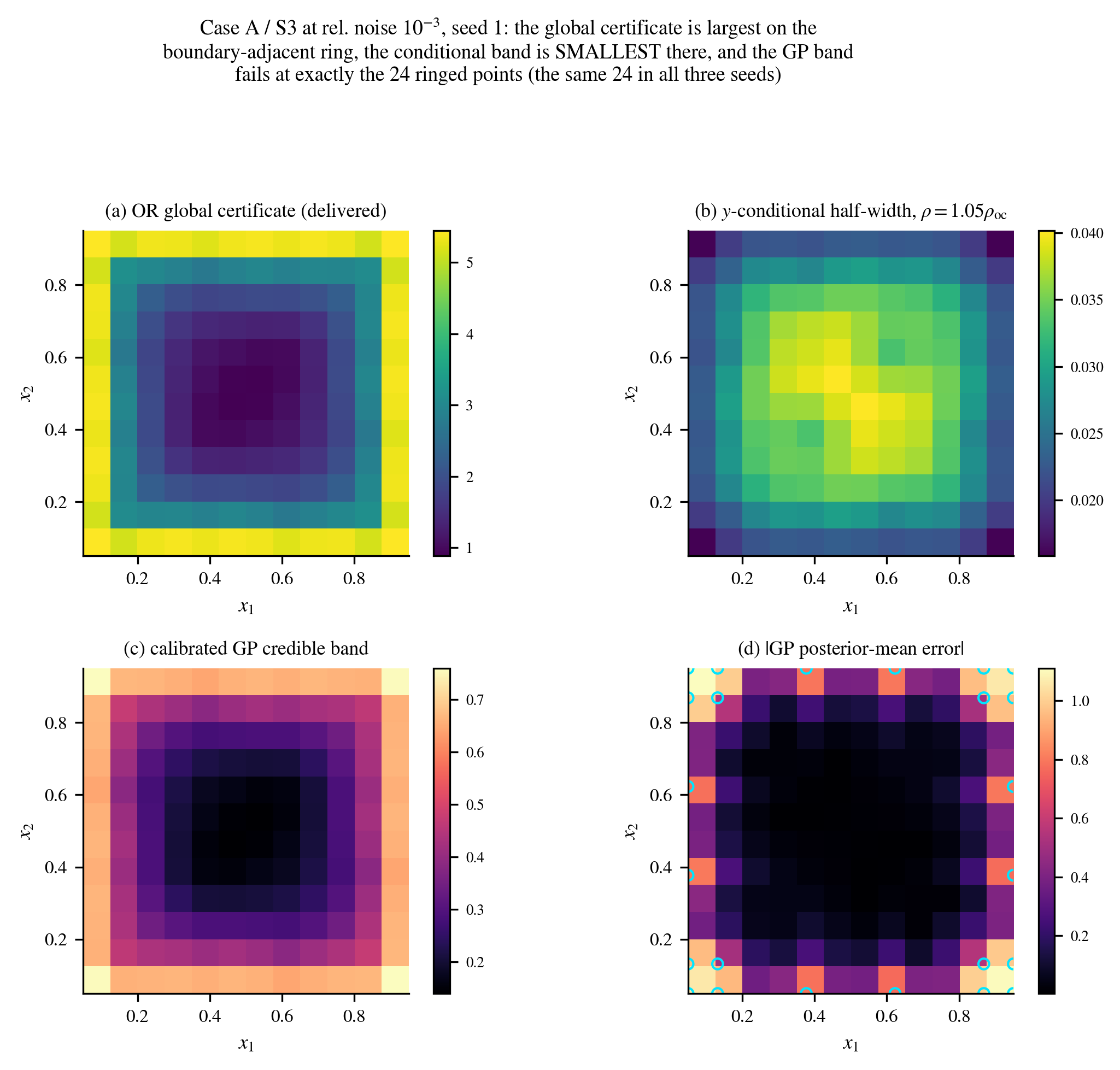}
\caption{Case A at rel.\ $10^{-3}$, seed 1, on the $12\times12$
evaluation grid: the global certificate, the $y$-conditional half-width,
the GP band, and the GP posterior-mean error with its $24$ failures
ringed. The global certificate and the GP band are largest on the
boundary-adjacent ring; the conditional band is smallest there. The same
$24$ points fail in all three seeds.}
\label{fig:Amaps}
\end{figure}

\clearpage

\section{Case B: full results (DAK)}\label{app:caseB}

\subsection{Dataset audit and conformal levels}

DAK \citep{DAK2025} composes a feature extractor, a variational linear
embedding and GP activations with a learned noise layer $\hat\sigma$.
\textbf{Protocol:} five tabular targets (\texttt{diabetes},
\texttt{concrete}, \texttt{energy}, \texttt{yacht},
\texttt{california}), each pinned by \texttt{data\_id} and passed
through a hard-failing target audit before any training
(Table~\ref{tab:Baudit}); two regimes, i.i.d.\ ($80/20$) and covariate
shift (a rank split along a random unit direction, $70/30$), plus one
legacy axis-mode anchor cell; seeds $\{0,1\}$; label noise
$\mathrm{rel}\in\{0,.25\}$ everywhere and $\{0,.1,.25,.5\}$ on
\texttt{diabetes}; nine arms per cell. The session delivered $39$ of the
$73$ planned cells before its wall-clock cap; the work order is
seed-major, so the loss is seed $2$ and the higher noise levels of seed
$1$, never a whole dataset. Within a (dataset, regime, seed) the noise
levels share one split and one test set -- we verified bit-identical
test targets in all $14$ such groups -- so the $19$ random-direction shift cells (the \texttt{shiftcol} anchor
brings the shift axis to $20$ cells in total) rest on
$10$ distinct splits and the $19$ i.i.d.\ cells on $10$; every cell-level
$p$ below is therefore optimistic, and we give the split-blocked version
wherever it changes a conclusion. Baselines: DAK as released but with
the KL term zeroed (as released the KL is live; the \texttt{return\_kl}
switch is broken and double division by the batch size leaves it at
${\approx}0.2\times$ a standard minibatch-ELBO KL) and KL-repaired;
split-conformal and temperature scaling (the latter is exactly
\emph{normalized} split-CP with scores $|r|/\hat\sigma(x)$) on a
fit-split model (the network never sees the calibration points); a
\emph{trivial} baseline, split-conformal around a constant predictor,
which uses no features and no training and which we report throughout
because it is the baseline this study's earlier revision did not have;
the OR head (R1) on the \emph{same} fit-split model with the \emph{same}
calibration indices -- a like-for-like head-to-head -- plus its no-floor
and pure-interval ablations; the water-filling prior (R2) with the
model's \emph{own} $\hat\sigma$, without access to the injected noise
level. (The combined R1{+}2 arm of an earlier revision was dropped from
the v5 arm set and is not tabulated here.) Bands are scored by coverage and width and by the
Winkler interval score at $95\%$ \citep{GneitingRaftery2007}. The
calibration split is $25\%$ of train, so the recorded conformal level
$n_c/(n_c{+}1)$ runs from $0.9815$ (\texttt{yacht} shift) to $0.9967$
(\texttt{california} i.i.d.)\ while the level the code actually achieves,
the order statistic $(k{+}1)/(n_c{+}1)$, sits at $0.950$--$0.963$;
Table~\ref{tab:Blevels} gives both per dataset and regime, and the
difference between them is a quantile convention, not a defect. Coverage
is a mean over cells (a bounded quantity); width and interval score are
median [IQR] over cells, and the width statistic is the \emph{within-cell
median} half-width unless stated -- the mean is reported separately in
Table~\ref{tab:Btail} because it is the statistic the remaining tail
lives in. Targets are standardized, so a half-width of $1$ is one
training standard deviation.

\begin{table}[t]\centering\footnotesize
\caption{Case B dataset audit (\texttt{S0}), run before any training and hard-failing.
Every source is pinned by \texttt{data\_id}; a categorical or $\le 20$-level target is
refused by the loader, the audit additionally requires $\ge 50$ distinct values and a
RandomForest 5-fold CV $R^2>0.05$. All five datasets pass, so none is excluded. The
previous corpus regressed OpenML \texttt{data\_id=1472} column \texttt{V8} (Glazing Area
Distribution, 6 levels, CV $R^2=-0.3211$) as \texttt{energy}; \texttt{data\_id=44960}
carries the real \texttt{heating\_load}.}
\label{tab:Baudit}
\setlength{\tabcolsep}{3pt}
\begin{tabular}{@{}l r r l l r r@{}}
\toprule
dataset & $N$ & $d$ & source & target & distinct $y$ & RF CV $R^2$\\
\midrule
\texttt{diabetes} & 442 & 10 & sklearn \texttt{load\_diabetes} & \texttt{target} & 214 & 0.4193\\
\texttt{concrete} & 1030 & 8 & OpenML \texttt{44959} & \texttt{strength} & 938 & 0.3428\\
\texttt{energy} & 768 & 8 & OpenML \texttt{44960} & \texttt{heating\_load} & 587 & 0.9651\\
\texttt{yacht} & 308 & 6 & OpenML \texttt{42370} & \texttt{Residuary.resistance} & 258 & 0.9954\\
\texttt{california} & 1500 & 8 & sklearn \texttt{california\_housing} & \texttt{target} & 1101 & 0.7401\\
\bottomrule
\end{tabular}
\end{table}

\begin{table}[t]\centering\footnotesize
\caption{Case B conformal calibration, recomputed from the splits this run used.
$n_c=\max(8,\lfloor 0.25\,n_{\mathrm{tr}}\rfloor)$; the recorded \emph{level} is
$n_c/(n_c{+}1)$ and the \emph{achieved} level is $(k{+}1)/(n_c{+}1)$ for
$k=\min(n_c{-}1,\lceil 0.95(n_c{+}1)\rceil{-}1)$, the order statistic the code takes.
The two are not equal, and the difference is the quantile convention, not a bug.
The paper's $n_c{=}88$, $88/89$ figures are exact for \texttt{diabetes} i.i.d.;
the shift regime trains on $70\%$, giving $n_c{=}77$ and $77/78$.}
\label{tab:Blevels}
\setlength{\tabcolsep}{5pt}
\begin{tabular}{@{}l l r r r c r c@{}}
\toprule
dataset & regime & $n_{\mathrm{tr}}$ & $n_{\mathrm{te}}$ & $n_c$ & level $n_c/(n_c{+}1)$ & $k$ & achieved $(k{+}1)/(n_c{+}1)$\\
\midrule
\texttt{california} & iid & 1200 & 300 & 300 & $300/301=0.9967$ & 285 & $286/301=0.9502$\\
\texttt{california} & shift & 1050 & 450 & 262 & $262/263=0.9962$ & 249 & $250/263=0.9506$\\
\texttt{concrete} & iid & 824 & 206 & 206 & $206/207=0.9952$ & 196 & $197/207=0.9517$\\
\texttt{concrete} & shift & 721 & 309 & 180 & $180/181=0.9945$ & 171 & $172/181=0.9503$\\
\texttt{diabetes} & iid & 353 & 89 & 88 & $88/89=0.9888$ & 84 & $85/89=0.9551$\\
\texttt{diabetes} & shift & 309 & 133 & 77 & $77/78=0.9872$ & 74 & $75/78=0.9615$\\
\texttt{diabetes} & shiftcol & 309 & 133 & 77 & $77/78=0.9872$ & 74 & $75/78=0.9615$\\
\texttt{energy} & iid & 614 & 154 & 153 & $153/154=0.9935$ & 146 & $147/154=0.9545$\\
\texttt{energy} & shift & 537 & 231 & 134 & $134/135=0.9926$ & 128 & $129/135=0.9556$\\
\texttt{yacht} & iid & 246 & 62 & 61 & $61/62=0.9839$ & 58 & $59/62=0.9516$\\
\texttt{yacht} & shift & 215 & 93 & 53 & $53/54=0.9815$ & 51 & $52/54=0.9630$\\
\bottomrule
\end{tabular}
\end{table}

\subsection{The nine-arm ledger, i.i.d. and shift}

\textbf{In-distribution results} (Table~\ref{tab:Barms},
Fig.~\ref{fig:Bp}, Fig.~\ref{fig:Bis})
-- \emph{the concessions first.} (i) Both DAK versions miscalibrate on
\texttt{diabetes} ($0.57$--$0.80$ at nominal $95\%$); pooled over the
five datasets their mean coverage is $0.861$ (released) and $0.873$
(KL-repaired), so the miscalibration is real but dataset-dependent, not
uniform, and it is intrinsic rather than a consequence of the KL bug.
(ii) \textbf{Split-conformal is the better in-distribution band}, and by
a larger margin than an earlier revision reported: it beats the OR head
on interval score in $18/19$ paired cells (median difference $-2.573$,
bootstrap $95\%$ CI $[-3.36,-1.20]$, $p=7.6\!\cdot\!10^{-6}$ at an
attainable floor of $3.8\!\cdot\!10^{-6}$) and is narrower in $19/19$
(median $-1.143$). (iii) \textbf{The water-filling prior yields no
accuracy or scoring gain when its noise hint is the model's own
$\hat\sigma$}: on \texttt{diabetes}, as a mean over the six cells, its
interval score is $6.95$
i.i.d.\ against $7.10$ (released) and $7.03$ (KL-repaired), and $8.13$
under shift against $8.02$ and $6.77$, at coverage $0.685$ and $0.703$.
On the median convention of Table~\ref{tab:Barms} the released-arm
comparison flips ($6.77$ against $7.24$ and $6.89$ i.i.d., $7.76$
against $6.92$ and $7.14$ under shift), so the concession is that R2
buys nothing detectable, not that it is uniformly worse.
(iv) The OR head is valid ($1.000$ mean coverage i.i.d., $0.994$ under
shift) but wide, and its $\kappa$-grid saturates in $29$ of $37$
delivered cells.

\textbf{Two ablation rows of the earlier table were artifacts and are
withdrawn.} With the feature scale repaired and the exact $\eta=0$ slice
in place, the no-floor OR band is \emph{valid}, not broken: mean
coverage $0.986$ i.i.d.\ and $0.961$ under shift, against $0.51$--$0.62$
before, at median half-widths $1.46$ and $1.34$. It beats the delivered
band on interval score in both regimes, by a median $1.853$
i.i.d.\ ($p=6.5\!\cdot\!10^{-4}$) and $1.255$ under shift
($p=9.8\!\cdot\!10^{-4}$), and it is the best OR variant under shift
(paired against the pure interval: median $-0.722$, better on $11/18$
cells, $p=0.038$); in-distribution the two are indistinguishable
(median $+0.360$, better on only $7/19$, $p=0.86$), which is why
Table~\ref{tab:Barms} prints the pure interval's $3.42$ below the
no-floor band's $3.54$ there. The pure interval is likewise no longer
degenerate: coverage $0.852$ at median half-width $0.65$ i.i.d., against
$0.02$--$0.06$ at $0.11$--$0.43$ before. The earlier numbers were
produced by a dual that returned negative half-widths, clipped to zero,
on a tube that the corrected code refuses outright.

\begin{table}[t]\centering\footnotesize
\caption{Case B, v5 corrected corpus (39 cells: 5 datasets $\times$ \{i.i.d., shift\}
$\times$ 2 seeds $\times$ noise grid, plus one \texttt{shiftcol} anchor cell).
Coverage is a \emph{mean} over cells (bounded quantity); width and interval score are
\emph{median [IQR]} over cells, and the width column is the within-cell \emph{median}
half-width, not the mean --- the mean is reported separately in Table~\ref{tab:Btail}
because it is the statistic the remaining tail lives in. $n$ is the number of cells the
arm delivered on: the OR arms refuse on two cells (empty consistent set) and
\texttt{or2} was run on \texttt{diabetes} only. Targets are standardized, so a
half-width of $1$ is one training standard deviation.}
\label{tab:Barms}
\setlength{\tabcolsep}{4pt}
\begin{tabular}{@{}l r c c c c@{}}
\toprule
arm & $n$ & coverage & median half-width & interval score & penalty share \\
 & & (mean) & median [IQR] & median [IQR] & of IS (median) \\
\midrule
\multicolumn{6}{@{}l}{\emph{i.i.d.}}\\
DAK (released) & 19 & 0.861 & 0.68 [0.44, 0.86] & 2.49 [1.83, 6.29] & 0.539 \\
DAK (KL repaired) & 19 & 0.873 & 0.68 [0.45, 0.90] & 2.59 [1.62, 5.40] & 0.465 \\
split-CP & 19 & 0.964 & 0.91 [0.61, 1.71] & 2.58 [1.80, 3.74] & 0.231 \\
norm.\ split-CP (temp.) & 19 & 0.961 & 0.90 [0.62, 1.77] & 2.60 [1.82, 3.77] & 0.201 \\
\textbf{triv} (constant pred.) & 19 & 0.963 & 2.02 [1.92, 2.43] & 4.39 [4.07, 5.24] & 0.079 \\
OR head (R1) & 19 & 1.000 & 2.06 [1.18, 3.29] & 5.79 [2.47, 7.02] & 0.000 \\
\quad -- no floor & 19 & 0.986 & 1.46 [0.91, 1.67] & 3.54 [1.99, 4.15] & 0.024 \\
\quad -- pure interval & 19 & 0.852 & 0.65 [0.50, 0.76] & 3.42 [1.47, 6.34] & 0.321 \\
water-fill (R2) & 6 & 0.685 & 0.90 [0.81, 0.95] & 6.77 [5.31, 7.85] & 0.724 \\
\midrule
\multicolumn{6}{@{}l}{\emph{covariate shift (random-direction rank split, 70/30)}}\\
DAK (released) & 19 & 0.749 & 0.72 [0.51, 0.88] & 4.71 [3.55, 6.44] & 0.732 \\
DAK (KL repaired) & 19 & 0.801 & 0.73 [0.52, 0.85] & 3.77 [2.51, 6.25] & 0.633 \\
split-CP & 19 & 0.892 & 1.06 [0.61, 1.90] & 3.82 [3.03, 4.74] & 0.254 \\
norm.\ split-CP (temp.) & 19 & 0.894 & 1.03 [0.61, 1.91] & 3.69 [2.98, 4.94] & 0.235 \\
\textbf{triv} (constant pred.) & 19 & 0.975 & 2.11 [1.96, 2.26] & 4.35 [4.13, 4.92] & 0.026 \\
OR head (R1) & 18 & 0.994 & 2.37 [1.33, 3.34] & 6.85 [3.57, 16.46] & 0.000 \\
\quad -- no floor & 18 & 0.961 & 1.34 [1.16, 1.80] & 5.60 [3.34, 13.65] & 0.030 \\
\quad -- pure interval & 18 & 0.852 & 0.82 [0.62, 1.23] & 6.49 [4.17, 15.57] & 0.146 \\
water-fill (R2) & 6 & 0.703 & 0.89 [0.86, 0.95] & 7.76 [6.50, 8.50] & 0.774 \\
\bottomrule
\end{tabular}
\end{table}

\begin{figure}[t]\centering
\includegraphics[width=\textwidth]{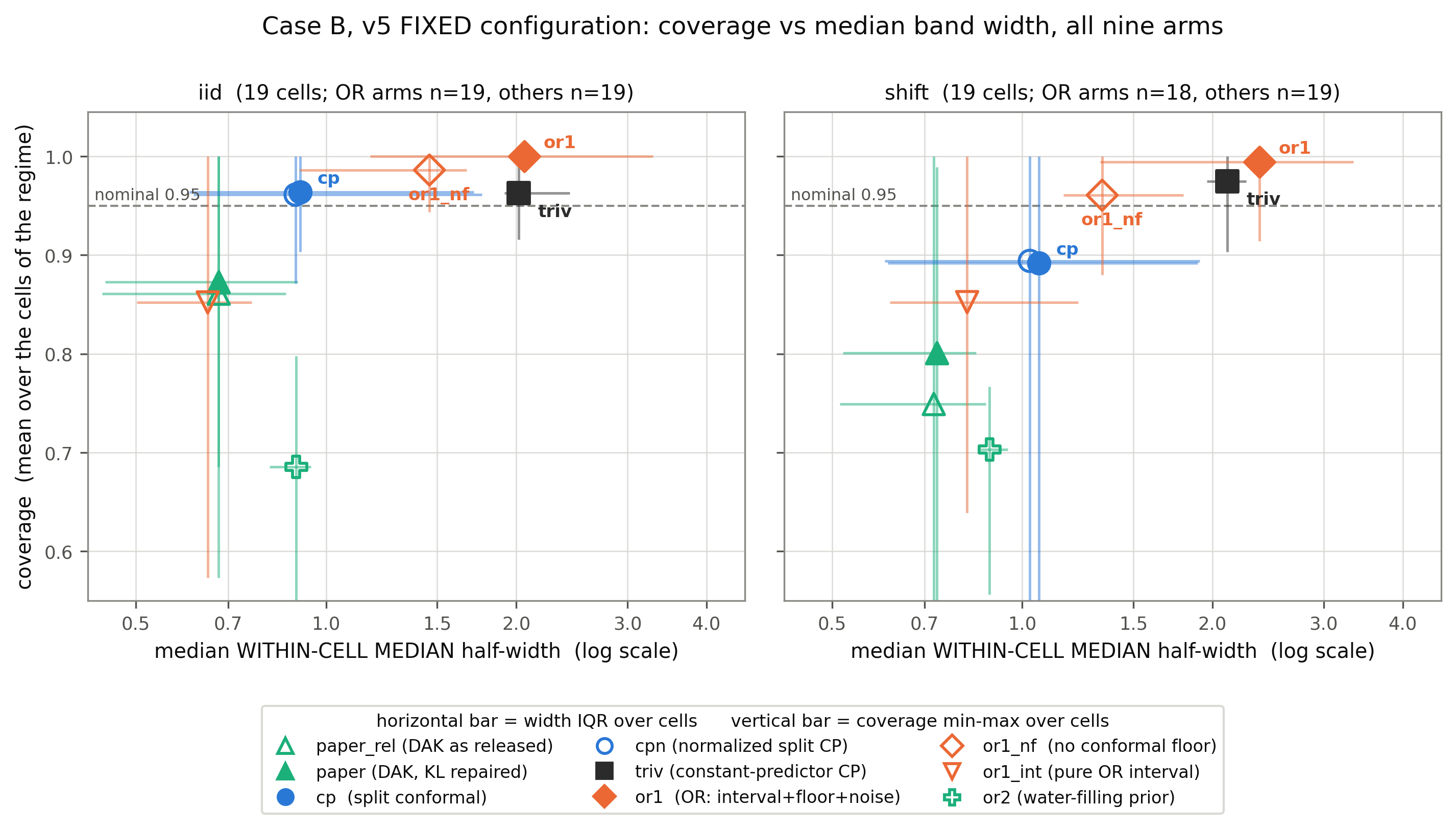}
\caption{Case B width--coverage Pareto over all $38$ i.i.d.\ and shift
cells (the OR arms deliver on $37$ of them; each panel prints its own
counts). Split-CP and
normalized split-CP sit at the smallest valid widths in-distribution; the
constant-predictor band and the OR head both over-cover at larger width;
the miscalibrated DAK arms cluster bottom-left. The pure-interval
ablation is now on-scale (coverage $0.852$), unlike in an earlier
revision.}
\label{fig:Bp}
\end{figure}

\begin{figure}[t]\centering
\includegraphics[width=\textwidth]{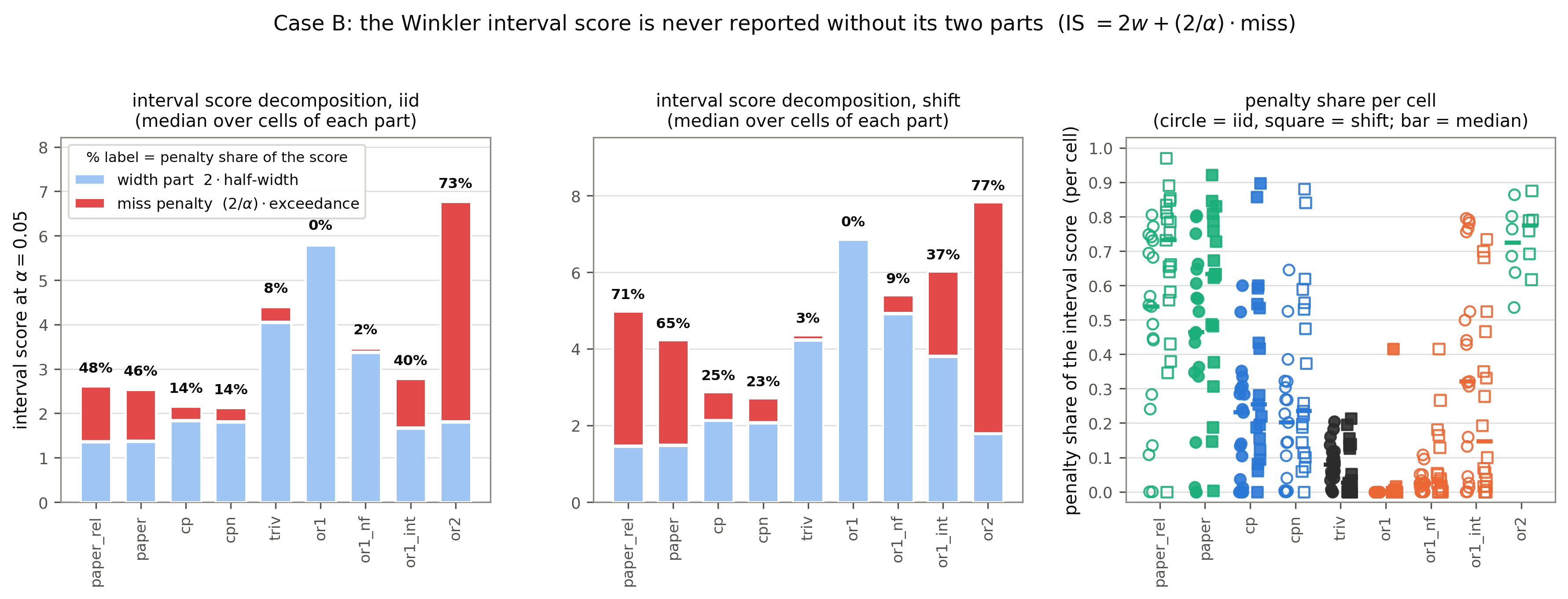}
\caption{The Winkler interval score is never reported without its two
parts, $\mathrm{IS}=2w+(2/\alpha)\cdot\text{miss}$. Left and centre: the
median over cells of the width part and of the miss penalty, stacked per
arm, i.i.d.\ and shift, with the penalty's percentage of the score printed
above each bar. Right: the penalty share of every individual cell, all
nine arms, circles i.i.d.\ and squares shift, with the per-arm median as a
bar. The OR head's score is pure width -- its penalty share is $0.000$ in
both regimes -- and the per-arm medians drawn in the right panel are the
last column of Table~\ref{tab:Barms}.}
\label{fig:Bis}
\end{figure}

\subsection{The width budget and the conformal floor}

\textbf{The OR geometry is not decorative, and the conformal floor is a
net loss} (Table~\ref{tab:Bfloor}, Fig.~\ref{fig:Bdecomp}). The width
decomposition now reads: in-distribution the band is $34\%$ OR interval,
$41\%$ conformal floor, $29\%$ noise allowance; under shift it is $69\%$
OR interval, $12\%$ floor, $18\%$ noise allowance. On \texttt{diabetes}
i.i.d.\ at $\mathrm{rel}=0.1$ (mean of the two seeds) the band $3.45$
splits into OR interval $0.86$,
conformal floor $1.63$, noise allowance $0.96$ -- the OR interval is
$25\%$ of the band, not the $5\%$ that the dead-column feature scale and
the clipped dual produced. What each component buys is now separable and
the answer is uncomfortable for the delivered band: the noise allowance
buys $+0.087$ coverage for $+0.66$ median width and no interval-score
gain ($p=0.86$ i.i.d.), while the conformal floor buys only $+0.010$
coverage for $+0.96$ median width and $+1.85$ interval score. On the
proper score the floor is a loss in both regimes, and we report the
no-floor band alongside the delivered one throughout rather than only as
an ablation.

\begin{table}[t]\centering\footnotesize
\caption{Case B width budget of the OR head, and what each component buys.
$w_{\mathrm{int}}$ is the tight OR interval, $w_{\mathrm{add}}=q$ the conformal floor
(constant in a cell), $w_{\mathrm{obs}}=1.96\hat\sigma$ the noise allowance. Paired
differences are median [bootstrap 95\% CI of the median], Wilcoxon $p$, on the cells
where all three variants delivered; the attainable $p$-floor is $2^{1-n}$
($3.8\times10^{-6}$ at $n{=}19$, $7.6\times10^{-6}$ at $n{=}18$). In the
lower block ``$w$ buys'' is the paired change from dropping that component:
$\Delta$width is the within-cell median half-width and $\Delta$IS the
interval score.}
\label{tab:Bfloor}
\setlength{\tabcolsep}{3pt}
\begin{tabular}{@{}l cc@{}}
\toprule
 & i.i.d.\ ($n{=}19$) & shift ($n{=}18$)\\
\midrule
$w_{\mathrm{int}}$ share of the band & 0.338 [0.281, 0.542] & 0.686 [0.543, 0.798]\\
$w_{\mathrm{add}}$ share of the band & 0.412 [0.180, 0.447] & 0.117 [0.024, 0.209]\\
$w_{\mathrm{obs}}$ share of the band & 0.292 [0.266, 0.352] & 0.182 [0.092, 0.213]\\
\midrule
$w_{\mathrm{int}}$ (absolute) & 0.828 [0.604, 0.960] & 1.898 [0.981, 5.747]\\
$w_{\mathrm{add}}$ (absolute) & 0.961 [0.291, 1.558] & 0.836 [0.230, 1.630]\\
$w_{\mathrm{obs}}$ (absolute) & 0.662 [0.435, 0.922] & 0.620 [0.417, 0.674]\\
\midrule
$w_{\mathrm{obs}}$ buys: $\Delta$coverage & $+0.087$ [+0.02, +0.19], $p=2.9\!\cdot\!10^{-4}$ & $+0.113$ [+0.04, +0.16], $p=6.5\!\cdot\!10^{-4}$\\
$w_{\mathrm{obs}}$ buys: $\Delta$width & $+0.662$ [+0.45, +0.91], $p=3.8\!\cdot\!10^{-6}$ & $+0.620$ [+0.43, +0.67], $p=7.6\!\cdot\!10^{-6}$\\
$w_{\mathrm{obs}}$ buys: $\Delta$IS & $+0.360$ [-0.19, +0.55], $p=0.86$ & $-0.722$ [-1.94, +0.36], $p=0.0385$\\
$w_{\mathrm{add}}$ buys: $\Delta$coverage & $+0.010$ [+0.00, +0.02], $p=0.00221$ & $+0.023$ [+0.00, +0.06], $p=0.00221$\\
$w_{\mathrm{add}}$ buys: $\Delta$width & $+0.961$ [+0.40, +1.54], $p=6.5\!\cdot\!10^{-4}$ & $+0.836$ [+0.26, +1.62], $p=9.8\!\cdot\!10^{-4}$\\
$w_{\mathrm{add}}$ buys: $\Delta$IS & $+1.853$ [+0.80, +2.95], $p=6.5\!\cdot\!10^{-4}$ & $+1.255$ [+0.36, +2.22], $p=9.8\!\cdot\!10^{-4}$\\
\bottomrule
\end{tabular}
\end{table}

\begin{figure}[t]\centering
\includegraphics[width=\textwidth]{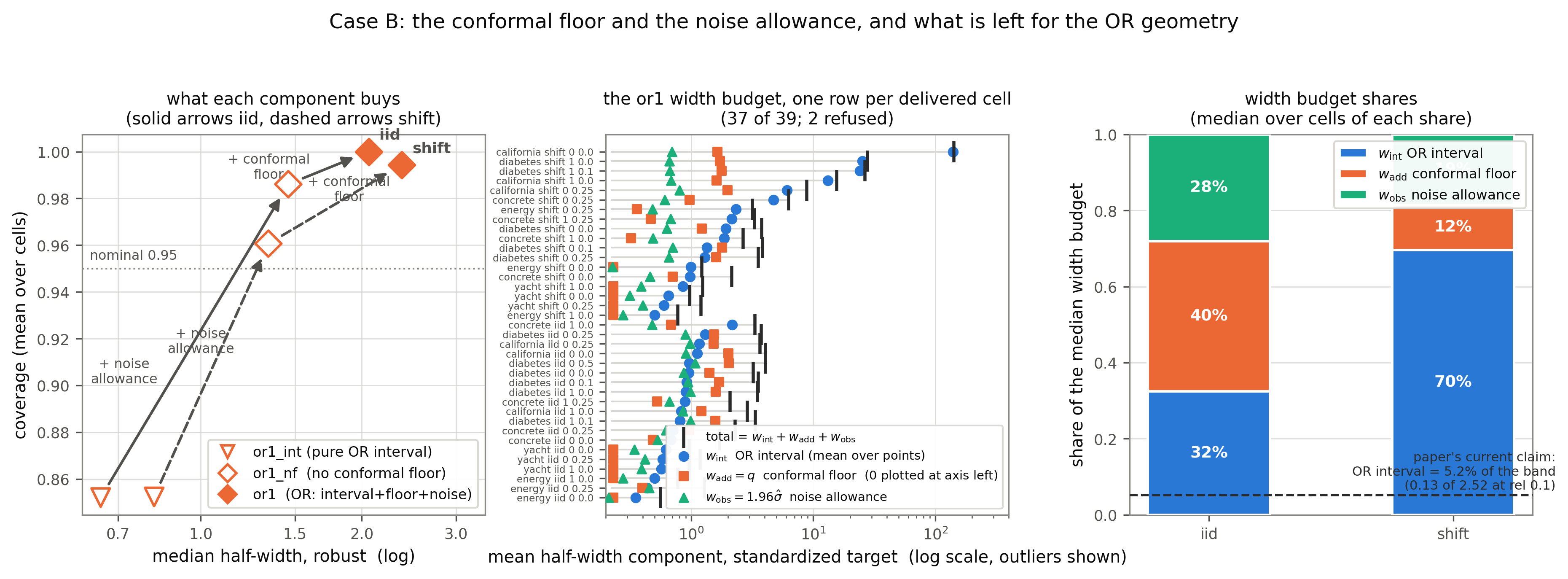}
\caption{The OR band's width budget, and what each component buys.
Left: the three OR variants in (median half-width, coverage) space, with
one arrow per component -- adding the noise allowance
$w_{\mathrm{obs}}$, then the conformal floor $w_{\mathrm{add}}$ -- drawn
solid for i.i.d.\ and dashed for shift, against the nominal $0.95$.
Centre: the absolute budget $w_{\mathrm{int}}+w_{\mathrm{add}}+
w_{\mathrm{obs}}$ stacked for every one of the $37$ delivered cells
($2$ of the $39$ refused), sorted within regime. Right: the same budget
as shares, median over cells. The OR interval is a third of the band
in-distribution and two-thirds under shift.}
\label{fig:Bdecomp}
\end{figure}

\subsection{Coverage under shift, the trivial baseline, and the centre counterfactual}

\textbf{Where the ball geometry earns its keep, predicted before
training.} The $\kappa$-grid saturates in $29/37$ delivered cells, but
not at random: Spearman$(\text{RF CV }R^2,\kappa_\star)=-0.619$,
$p=4.5\!\cdot\!10^{-5}$ over the $37$ cells. The eight cells where it
does not saturate -- $\kappa_\star\in\{1.3,1.7\}$ with a conformal floor
of \emph{exactly} zero, meaning the OR ball already contains every
calibration residual -- are all \texttt{energy} and \texttt{yacht} at
zero injected noise, the two targets with CV $R^2\ge0.965$. The same
split governs whether the OR band's \emph{shape} buys anything: holding
the centre and the mean width fixed and asking whether the OR geometry
covers more than a flat band of the same width would, the skill is
positive on $5/6$ \texttt{energy} and $6/6$ \texttt{yacht} shift and
i.i.d.\ cells against $1/12$ and $3/13$ elsewhere (Fisher exact
$p=0.0039$ and $0.0031$), with a best case of $+0.719$ coverage at
$0.41\times$ the trivial band's width. This is the tabular analogue of
the 1D gap task, and the practical point is that a training-free audit
statistic predicts it.

\textbf{Under shift the ledger changes, but not the way an earlier
revision said it did} (Table~\ref{tab:Bshift}, Fig.~\ref{fig:Bshift},
Fig.~\ref{fig:Btriv}). On the tabular corpus the one thing the OR head
wins universally is coverage: pooled it holds $0.9969$
$[0.9947,0.9984]$ of test points against $0.890$ for split-CP and
$0.891$ for normalized split-CP, and it beats split-CP on coverage on
$5/5$ datasets. It loses the proper score. Against split-CP the OR head
is worse on interval score in $16$ of $18$ delivered shift cells (median
$+2.695$, $p=8.4\!\cdot\!10^{-4}$); blocked on the $10$ distinct splits
it loses $9/10$ ($p=0.0098$). Against normalized split-CP the figures
are $16/18$, median $+2.719$, $p=4.2\!\cdot\!10^{-4}$. Against the
feature-free constant band it loses $12/18$ (median $+2.677$,
$p=0.038$). Repeating all of it on the median-based pointwise interval
score, which is immune to the one catastrophic cell, does not rescue it
($+2.508$ against split-CP in the same sign convention, better in
$1/18$, $p=1.5\!\cdot\!10^{-5}$).
The coverage it does win over the constant band is worth $+0.0075$ in
the median cell and costs $+1.374$ of \emph{mean} half-width (the
median-half-width cost is $+0.366$, $p=0.52$). We therefore
withdraw ``the OR shape covers everything at smaller in-distribution
width'' as a statement about this corpus.

\textbf{The regime is not ``shift''; it is ``shift on a learnable
target''} (Table~\ref{tab:Bshift}). The per-dataset ledger splits
cleanly along the training-free audit statistic. On \texttt{energy} (CV
$R^2=0.965$, and the harshest shift in the corpus at
$d_{\mathrm{M}}=3.04$) split-CP's coverage collapses to $0.784$ while
the OR band holds $1.000$ -- but it does \emph{not} score better there:
its $3.44$, and the no-floor band's $3.20$, both lose to split-CP's
$2.05$ and to normalized split-CP's $1.73$, and what the OR band wins on
\texttt{energy} is coverage alone. Both OR variants do beat the constant
band's $4.03$, but neither is the best of the four arms
Table~\ref{tab:Bshift} tabulates: split-CP is. On \texttt{yacht} (CV $R^2=0.995$)
split-CP covers $0.896$ and the OR band covers $0.971$ at a better
interval score than every conformal arm ($2.73$ and $2.62$ against
$3.24$ and $6.39$), though the two DAK arms themselves score $2.38$ and
$2.41$ there. On the
three datasets with CV $R^2\le0.74$ split-CP keeps $0.83$--$0.96$
coverage and the OR band's mean interval score is $2$--$33\times$ worse
(on within-cell medians, $1.7$--$9.6\times$). That
is the honest scope: the OR shape carries the out-of-distribution signal
where the target is genuinely learnable, and where it is not, a
conformal constant is both cheaper and better.

\textbf{And ``conformal constants are structurally blind'' is the wrong
diagnosis here: the failure is in the centre, not the width.} A
four-way counterfactual over the $20$ shift cells -- the $19$
random-direction cells plus the \texttt{shiftcol} anchor -- gives median coverage
$0.921$ for split-CP (model centre, $q_{\mathrm{cp}}$), $1.000$ for the
model centre with the \emph{constant} band's wider quantile
$q_{\mathrm{triv}}$, $0.830$ for the constant centre with
$q_{\mathrm{cp}}$, and $0.9925$ for the constant band itself; the median
quantiles are $q_{\mathrm{cp}}=1.159$ against $q_{\mathrm{triv}}=2.127$.
Handing the same shift-blind centre a wider constant -- obtainable from
the calibration set alone -- restores coverage to $1.000$. No shape is
required. The OR head does not repair the centre either: its RMSE under
shift is significantly worse than the KL-repaired posterior mean's
(median $+0.104$, better in $4/18$, $p=0.0034$) and no better than
split-CP's ($+0.040$, $7/18$, $p=0.18$). We also report the smallest
constant inflation of split-CP that matches the OR head's coverage in
each shift cell: it is $1.68\times$ in the median, and at that inflation
the shape-free band reaches the OR band's coverage at $1/1.79$ of its
\emph{mean} half-width and is narrower in $16/18$ cells (on the
within-cell median half-width the same comparison is $1/1.20$ and
$14/18$).

\textbf{Two further scope corrections we owe the reader.} First, the
random-direction shift used for $19$ of the $20$ shift cells does not
stress a shift-blind band in the direction one would expect: paired
across the $19$ matched cells, the constant band's coverage \emph{rises}
under shift (median $+0.0216$, higher in $14/19$) while split-CP's falls
($-0.0451$, $p=0.0046$), because the rank split narrows the target
distribution ($y$ standard-deviation ratio below $1$ in $9$ of $10$
shift splits, down to $0.494$; only \texttt{yacht} seed $0$ exceeds
$1$). A regime in which the maximally blind
baseline gets easier is a weak test of blindness, and the coverage
collapses that do occur are dataset-specific rather than generic.
Second, the axis-mode \texttt{shiftcol} cell is a single cell on which
the OR arm refused, so it contributes zero OR observations and is not a
third regime; its headline severity ($0.0075$ of test points inside the
training box) is a definitional artifact of measuring an axis split with
an all-coordinates box, and on the two mode-neutral scales it sits
\emph{inside} the random-direction range. We report it as an anchor cell
and draw nothing from it.

\textbf{The trivial baseline, in the open} (Fig.~\ref{fig:Btriv}).
Corpus-wide and in-distribution the constant-predictor conformal band is
\emph{not} better than the OR head on interval score (median difference
$-0.425$, $11/19$, $p=0.71$) or on median width ($-0.041$, $10/19$,
$p=0.74$), although it does cover less ($-0.034$, $18/19$,
$p=1.9\!\cdot\!10^{-4}$). The earlier claim that the trivial band beats
the OR head is a \texttt{diabetes} effect: it holds $6/6$ there
($-2.799$, $p=0.031$ at the attainable floor) and $3/3$ on
\texttt{california}, and fails on \texttt{concrete}, \texttt{energy} and
\texttt{yacht}. Under shift it does beat the OR head on interval score
in $12/18$ cells ($-2.677$, $p=0.038$), and although it beats split-CP
on coverage in $15/19$ shift cells (median $+0.0693$, $p=0.0016$) it
\emph{loses} to it on interval score in $13/19$ (median $+1.036$,
$p=0.045$) -- which is what Table~\ref{tab:Barms}'s shift block already
shows, \texttt{triv} at $4.35$ against split-CP's $3.82$. We report the
split rather than a pooled headline in either direction. The normalized conformal baseline, meanwhile, is not a
defence of the conformal side: its half-width has a within-cell
coefficient of variation of $0.0144$ (median over cells, IQR
$[0.0096,0.0226]$), it differs from the constant band by a median
$2.0\%$ of $q_{\mathrm{cp}}$, the two make the same in/out decision at
$7369$ of $7395$ test points ($99.65\%$), and pairwise it is
indistinguishable from split-CP on \emph{mean} half-width ($p=0.62$
i.i.d., $p=0.83$ shift; on the within-cell median half-width, $p=0.80$
and $p=0.62$) and on interval score ($p=0.62$, $p=0.52$). It is not
numerically identical -- the centres are, the widths are not -- and the
honest statement is that the learner's own $\hat\sigma(x)$ is so nearly
homoscedastic that normalizing by it changes nothing.

\begin{figure}[t]\centering
\includegraphics[width=\textwidth]{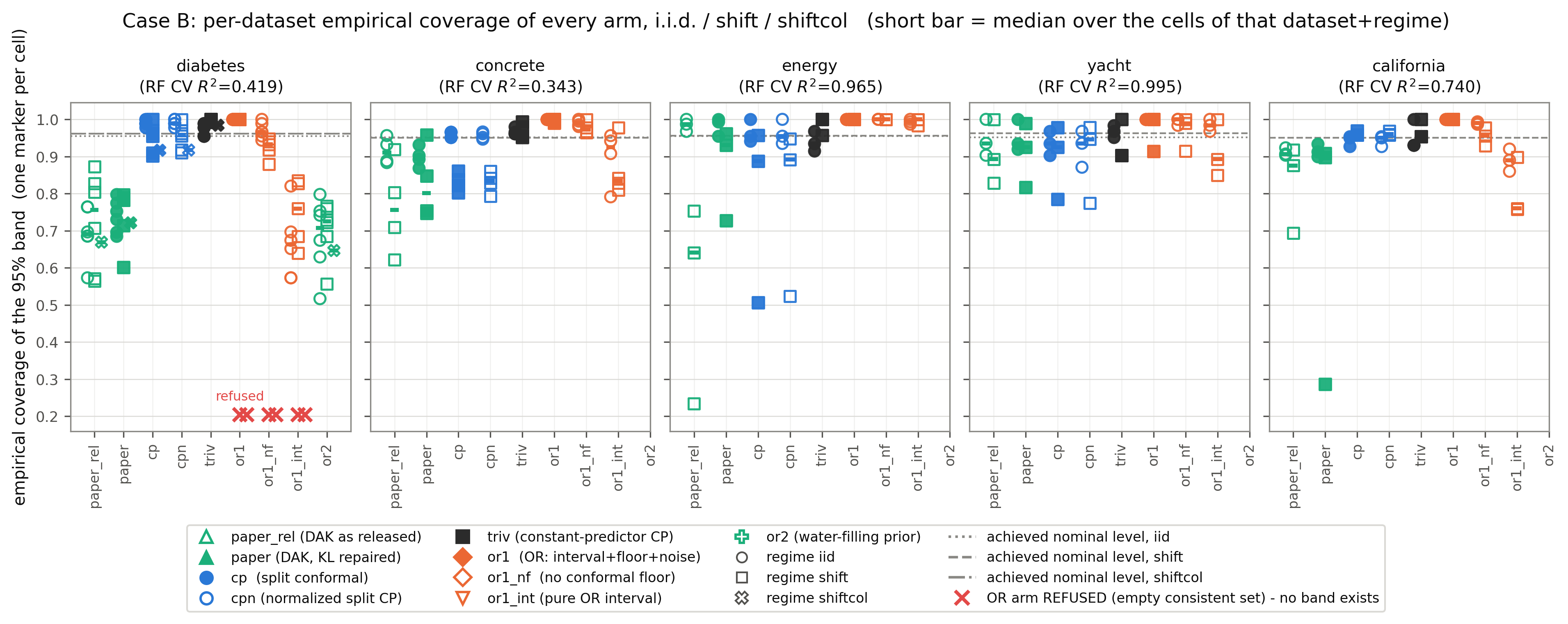}
\caption{Case B coverage per arm and per dataset, all three regimes drawn
(circle i.i.d., square shift, cross \texttt{shiftcol}) with a separate
nominal-level guide for each; every cell is shown, and a cross at the
foot marks the two cells on which the OR arm refused. Coverage is the one
column the OR head wins on all five datasets \emph{under shift};
Table~\ref{tab:Bshift} shows what it costs.}
\label{fig:Bshift}
\end{figure}

\begin{figure}[t]\centering
\includegraphics[width=\textwidth]{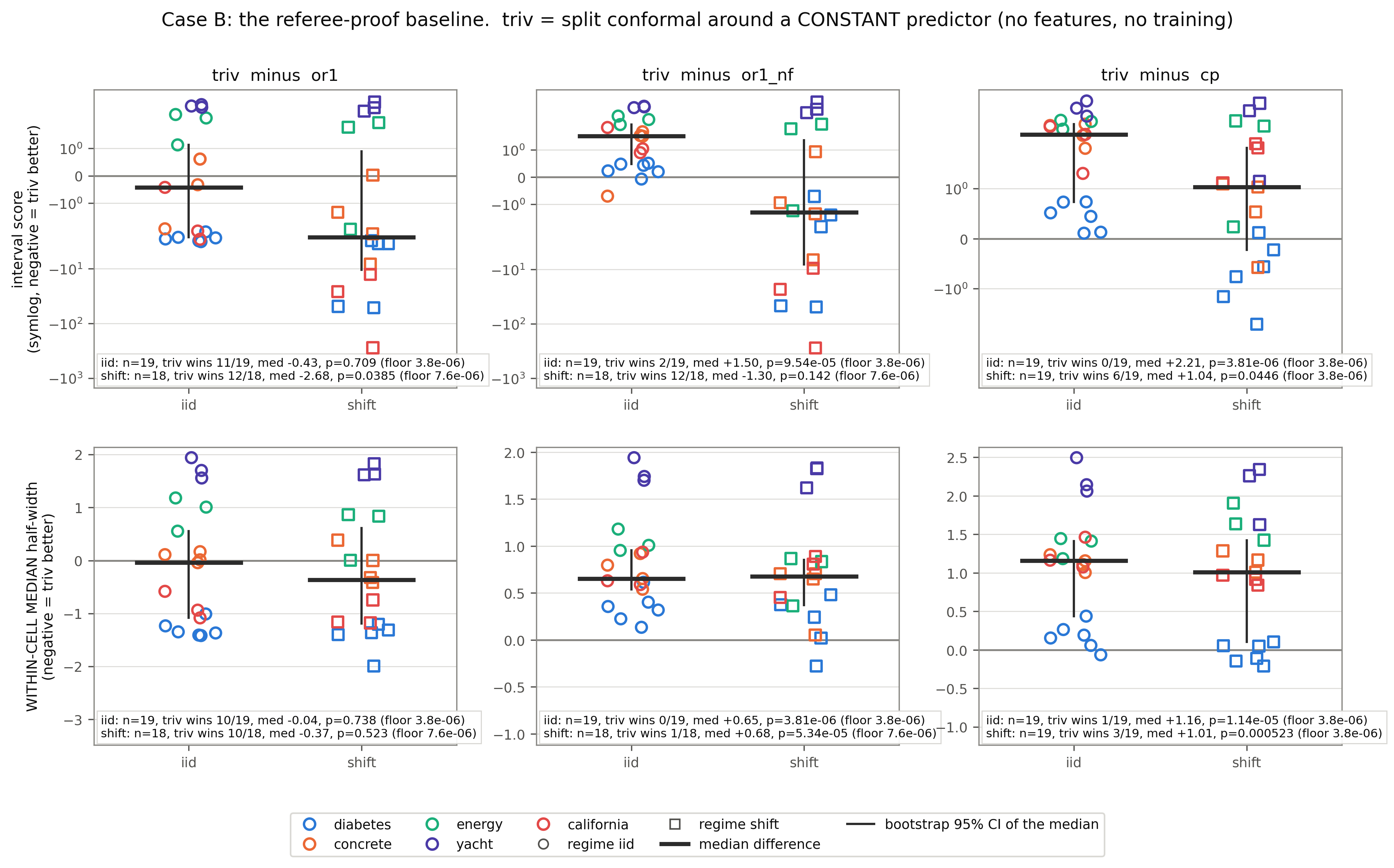}
\caption{The referee-proof baseline. Paired differences of the
constant-predictor conformal band against the OR head, its no-floor
ablation and split-CP, on interval score and on median half-width, in
both regimes; every cell is drawn and coloured by dataset, with the
median difference and its bootstrap $95\%$ CI. A feature-free constant
band beats the delivered OR band on interval score under shift
($12/18$, $p=0.038$) and is indistinguishable from it
in-distribution.}
\label{fig:Btriv}
\end{figure}

\subsection{The per-point tail and the refusals}

\textbf{The tail, and the refusals.} The corrected configuration removes
the width defect but not the worst-case semantics that produced its
symptom. The blow-up census is $0/37$ cells by the study's own criterion
(within-cell median half-width above $20$ on a standardized target),
against $13/38$ before (Fisher exact $p=7.1\!\cdot\!10^{-5}$), and the
mechanism is measurably gone: the median test$/$train feature-norm ratio
is $1.006$ and the median null-space share of the test functional is at
machine zero, where the shipped clamp left a largest-in-cell null-space
share whose median over the $30$ ablation cells is $0.69$ and which is
exactly $1.000$ in $14$ of them (S5 below). What remains is a
per-\emph{point} tail: $95$ of
$7129$ test points ($1.33\%$) in $13$ of $37$ cells carry an OR interval
above $20\,\mathrm{sd}(y_{\mathrm{fit}})$, with a maximum of $61896$ on
a standardized target (Table~\ref{tab:Btail}). It is a shift
phenomenon -- $92$ of those points and $10$ of those cells are under
shift -- and it is carried by the minority of test functionals that
retain a partial null-space component: cells with
$\mathrm{ns}_{\max}>10^{-6}$ have an over-cap point $8/10$ times against
$5/27$ elsewhere (odds ratio $17.6$, $p=1.1\!\cdot\!10^{-3}$;
Spearman$(\mathrm{ns}_{\max},\text{mean half-width})=0.495$,
$p=1.8\!\cdot\!10^{-3}$). Dropping dead columns cannot remove this,
because a rank covariate shift pushes the test functional out of the
training span by construction. The guarded-delivery ledger, in full:
$0$ non-finite half-widths, $0$ negative raw half-widths, $0$
centre-outside firings and $\max(\text{half}/E_\star)=0.965<1$ in every
cell, so Thm.~\ref{thm:local}(d) is never violated -- but $55$ of the
$7129$ delivered test points, spread over $16$ of the $37$ cells, had a
relative primal--dual gap above the $10^{-6}$ tolerance, and on those
points the guard substituted the certified global value (the estimate the
global certificate belongs to, with that certificate as its half-width)
for the conditional dual. The largest recorded relative gap is $0.55$
(\texttt{diabetes/shift/1/0.0}, absolute $169.3$), the next $0.53$
(\texttt{concrete/shift/1/0.25}, absolute $20.0$); $55$ delivered
half-widths are therefore guarded substitutions rather than tight duals,
the same order of magnitude as the $95$ over-cap points above, and we
report them as such. With that caveat the tail is the
exact worst-case certificate doing what it is defined to do at a handful
of out-of-support points, and it is the reason the mean half-width and
the interval score must be reported next to the median. On two cells --
\texttt{diabetes/shift/0/0.5} and the axis-mode \texttt{shiftcol} cell -- the
OR head \emph{refused}: the noise budget implied by the network's own
$\hat\sigma$ was smaller than the best achievable fit-split misfit
(margins $-0.487$ and $-1.553$), the consistent set is empty and no band
exists. That is the correct behaviour and it is why the previous
configuration could report a zero-width band there; a refusal makes the
arm absent, never zero, and every paired comparison above is run
pairwise-complete.

\begin{table}[t]\centering\footnotesize
\caption{Case B blow-up census under the corrected configuration, i.e.\ with
\texttt{FEATNORM=drop} and \texttt{OR\_DELIVERY=fit} and
\texttt{OR\_ETA\_INFLATE=refuse}
on the exact $\eta{=}0$ slice. A \emph{blow-up} is a within-cell median half-width above
$20$ on a standardized target. The census is $0/37$; what remains is a per-\emph{point}
tail in $13/37$ cells, listed here. \texttt{n$>$cap} counts test points whose OR
interval exceeds $20\,\mathrm{sd}(y_{\mathrm{fit}})$; $\mathrm{ns}_{\max}$ is the
largest null-space share of a test functional in the cell.}
\label{tab:Btail}
\setlength{\tabcolsep}{4.5pt}
\begin{tabular}{@{}l r r r r r r r@{}}
\toprule
cell & $n_{\mathrm{test}}$ & n$>$cap & median half & mean half & max half & $\mathrm{ns}_{\max}$ & $\kappa_\star$\\
\midrule
\texttt{california/shift/0/0.0} & 450 & 8 & 0.62 & 139.8 & 61896 & 0.500 & 40 \\
\texttt{diabetes/shift/1/0.0} & 133 & 13 & 1.52 & 25.4 & 830 & 0.370 & 40 \\
\texttt{diabetes/shift/1/0.1} & 133 & 10 & 1.07 & 24.1 & 783 & 0.297 & 40 \\
\texttt{california/shift/1/0.0} & 450 & 36 & 1.28 & 13.2 & 391 & 0.194 & 40 \\
\texttt{california/shift/0/0.25} & 450 & 13 & 0.56 & 6.1 & 516 & 0.337 & 40 \\
\texttt{concrete/shift/0/0.25} & 309 & 8 & 0.62 & 4.7 & 245 & 0.308 & 40 \\
\texttt{concrete/iid/1/0.0} & 206 & 1 & 0.66 & 2.2 & 287 & 0.098 & 40 \\
\texttt{concrete/shift/1/0.25} & 309 & 1 & 1.41 & 2.1 & 37 & 0.000 & 40 \\
\texttt{diabetes/shift/0/0.0} & 133 & 1 & 1.31 & 1.9 & 21 & 0.000 & 40 \\
\texttt{concrete/shift/1/0.0} & 309 & 1 & 0.78 & 1.9 & 184 & 0.055 & 40 \\
\texttt{diabetes/shift/0/0.1} & 133 & 1 & 0.86 & 1.3 & 26 & 0.000 & 40 \\
\texttt{california/iid/0/0.25} & 300 & 1 & 0.83 & 1.2 & 20 & 0.000 & 40 \\
\texttt{concrete/iid/0/0.25} & 206 & 1 & 0.48 & 0.7 & 27 & 0.000 & 40 \\
\midrule
\multicolumn{8}{@{}l}{all 37 delivered cells: 95/7129 test points over the cap (1.33\%), 13/37 cells}\\
\bottomrule
\end{tabular}
\end{table}

\subsection{S5: the 2x2 width-defect ablation}

\paragraph{Which defect produced the blow-ups (S5).}
Two implementation defects were candidates: the feature-scale clamp
(\texttt{Phi.std(0).clamp\_min(1e-6)}, which divides a column whose training
standard deviation is $\sim\!10^{-14}$ by $10^{-6}$ instead of dropping it) and
delivery from all training rows rather than the fit split. The 2x2 ablation
runs both factors on one trained model per cell over 30 shift cells, so the
comparison is fully paired and isolates the head (Table~\ref{tab:Babl},
Table~\ref{tab:Battr}, Fig.~\ref{fig:Babl}). The answer is unambiguous
and it is not the answer the shipped diagnostics suggested.

\emph{The clamp is the width defect, and it acts on the tail, not on the
centre of the width distribution.} Holding delivery at the fit split, the
clamp multiplies the within-cell \emph{mean} half-width by
$235$ (95\% bootstrap CI of the median
[1.13, $2.44\!\cdot\!10^{3}$], Wilcoxon
$p=9.8\!\cdot\!10^{-7}$, $n=26$ paired cells, attainable floor
$3\!\cdot\!10^{-8}$) and the within-cell \emph{maximum} half-width by
$2681$ ($p=7\!\cdot\!10^{-5}$), while the within-cell
\emph{median} half-width moves by a factor of only
$1.022$ [1.006, 1.041]
-- statistically detectable ($p=2.5\!\cdot\!10^{-5}$) and scientifically
negligible. The interval score, which is the loss a user pays, rises by a
factor $111$ ($p=2.1\!\cdot\!10^{-6}$). Only the clamp variants
ever produce a half-width above $10^6$: 10 of 39 delivered clamp cells against
0 of 37 delivered drop cells, and every discordant pair goes the same way.
Tested on the delivered pairs only -- the population those counts come
from -- McNemar's exact test gives $7/0$ discordant cells and
$p=0.0156$ at the fit split and $2/0$ and $p=0.50$ on all rows, the
latter because all-rows delivery refuses on $18$--$19$ of $30$ cells and
leaves only $11$ paired deliveries. Counting instead every one of the
$30$ cells, including the refused ones whose recorded \texttt{half\_max}
is the ${\sim}10^{9}$ stand-in that Sec.~\ref{sec:caseB} insists must not
be read as a width, the discordant counts are $9/0$ and $15/0$ and the
$p$-values $0.0039$ and $6.1\!\cdot\!10^{-5}$, both at the attainable
floor. The direction is the same in every version; only the delivered-pair
test is a statement about a delivered width.

\emph{The effect is strongly dataset-dependent, and the wide CI says so.}
Per dataset, at the fit split, the clamp multiplies the mean half-width by
$2307$ on concrete ($6/6$ cells, $p=0.031$ at the attainable floor), $2946$ on
energy ($6/6$, $p=0.031$), $1.13$ on diabetes ($8/9$, $p=0.027$) and $1.27$ on
yacht ($4/5$, $p=0.125$). The direction is the same everywhere -- the clamp is
never narrower -- but the magnitude spans three orders of magnitude across
datasets, which is why the pooled bootstrap CI runs from $1.13$ to
$2.4\!\cdot\!10^{3}$ and why a single pooled factor should be quoted with its
CI, never alone.

\emph{The mechanism is null-space dimension, measured directly.} Dropping the
dead columns leaves $\mathrm{nullity}(\Phi_{\rm fit})$ at a median of
4 [3, 5]; clamping them keeps them in the ambient space
and the nullity becomes 20
[18, 23] of 56 columns at
the same numerical rank. On that null space the tight interval is bounded only
by the ball radius, so a test functional that lands in it inherits the whole
radius. In 14 of 30
clamp cells at least one test functional lies \emph{entirely} in
$\mathrm{null}(\Phi_{\rm fit})$ (recorded null-space share $>0.999$; 15 of 30
exceed $0.5$), against 0 of 30 drop cells (Clopper--Pearson 95\%
$[0.000,0.116]$), whose largest observed share is $0.42$. Within each variant
the null-space share and the $\log_{10}$ maximum half-width move together
(Spearman $\rho=+0.84$, $p=4.7\cdot10^{-9}$ at clamp / fit split). The drop
rule removes the catastrophic leakage, not all of it: 11 of 30 drop cells still
show a share above $0.05$, which is why the production head's largest
half-width still reaches $26$ $[4.1,253]$ per cell and 15 of its 26 delivered
cells flag at least one point above the $20\,\mathrm{sd}(y_{\rm fit})$ cap.

\emph{Delivering on all rows is a validity defect, not a width defect.} On
the same cells the delivery axis moves the mean half-width by a factor
$0.945$ [0.621, 0.984]
and the maximum by $0.993$; the interaction is
$1.053$ [0.993, 1.278]
($p=0.083$) and is not detectable on any scale. What all-rows
delivery does instead is empty the consistent set. Adding the held-out
calibration rows to the conditioning set raises the least-squares misfit faster
than the tube radius $\eta\propto\sqrt{n}$ grows: on the 15 cells that
refuse under all-rows delivery but not under the fit split, the fit-split ratio
$\eta/\min\|\Phi\theta-y\|$ is $1.086$ $[1.055,1.179]$ while the all-rows
ratio is $0.777$ $[0.692,0.932]$. The refusal rate goes from 3--4 of 30 cells
at the fit split to 18--19 of 30 on all rows (McNemar exact
$p=6.1\!\cdot\!10^{-5}$, 15 discordant cells, all in the same
direction). The shipped code never saw this, because it inflated $\eta$ to
$1.02\times$ the achieved misfit, which makes the tube feasible by
construction.

\emph{A refusal with a tiny half-width is the signature, not a narrow band.}
The shipped configuration (clamp, all rows) refuses on
18 of 30 cells, and on
those cells its recorded median half-width is
$3.97\!\cdot\!10^{-6}$ $[1.21\!\cdot\!10^{-6}, 6.78\!\cdot\!10^{-6}]$ while its recorded \emph{maximum}
half-width is $1.6\!\cdot\!10^{9}$ and its delivered mean band width
is $1.36\!\cdot\!10^{7}$ at a recorded coverage of
$0.970$. Both numbers come from the same object: on an
empty tube the head hands back the exact equality slice at
$\eta=\min\|\Phi\theta-y\|$ with the radius lifted to
$\|\Phi^{+}y\|(1+10^{-9})$, whose half-width is proportional to the
null-space component of the test functional. Under the clamp
$\|\Phi^{+}y\|$ has median $1.9\cdot10^{10}$, so the stand-in is
$10^{-6}$ where the functional is in the row space and $10^{12}$ where it
is not. Under the drop rule the same refusals occur with
$\|\Phi^{+}y\|=0.83$ and the stand-in is uniformly $\sim\!10^{-14}$.
Averaging either into a width is the laundering the refusal test exists to
prevent; the correct report is the refusal rate.

\begin{table}[t]\centering\small
\caption{\textbf{Case B / S5: the 2x2 width-defect ablation.} Thirty
shift cells (4 datasets $\times$ 3 seeds $\times$ 2--4 noise levels), four
head variants per cell built on \emph{one} trained model, so the comparison
isolates the head. \textsc{featnorm} $\in$ \{clamp $10^{-6}$ as shipped, drop
dead columns\}; \textsc{delivery} $\in$ \{all rows as shipped, fit split only\}.
A \emph{refused} delivery has an empty consistent set and is not a width
measurement: refusals are reported as a rate and excluded from the width
quantiles. Half-widths are $w_{\mathrm{int}}$ on a standardized target;
``blow-up'' is the notebook's own criterion at threshold $20$.}
\label{tab:Babl}
\setlength{\tabcolsep}{3pt}
\begin{tabular}{@{}llrrrrrrr@{}}
\toprule
 & & refused & \multicolumn{3}{c}{median [IQR] over the delivered cells} & blow-up & $ns_{\max}$ & median\\
\cmidrule(lr){4-6}
featnorm & delivery & /30 & median $w$ & mean $w$ & max $w$ & (mean) & $>0.999$ & IS\\
\midrule
clamp $10^{-6}$\,(shipped) & all rows & 18 & 0.641 & 266 & $1.17\!\cdot\!10^{4}$ & 0.58 & 15/30 & 533\\
clamp $10^{-6}$ & fit split & 3 & 0.863 & 87.3 & $8.06\!\cdot\!10^{3}$ & 0.56 & 14/30 & 903\\
drop dead & all rows & 19 & 0.588 & 0.841 & 2.53 & 0.09 & 0/30 & 2.91\\
drop dead\,(v5) & fit split & 4 & 0.818 & 1.9 & 26.2 & 0.12 & 0/30 & 6.85\\
\bottomrule
\end{tabular}
\end{table}

\begin{table}[t]\centering\small
\caption{\textbf{Case B / S5: attribution of the width defect.} Each row is a
paired comparison on $\log_{10}$ of the named scale, over the cells where both
variants delivered; the ratio is $10^{\text{median difference}}$ with the 95\%
bootstrap CI of the median, the Wilcoxon signed-rank $p$, the number of paired
cells, and the smallest $p$ attainable at that $n$. \textsc{featnorm} rows
are clamp$/$drop at the delivery named after the comma; \textsc{delivery}
rows are all rows$/$fit split at the feature rule named after the comma;
\textsc{interaction} is (clamp all$/$fit)$/$(drop all$/$fit). The clamp defect does not
live in the median half-width; it lives in the mean, the maximum and the
interval score. The delivery axis is a factor ${\approx}1$ on every width
scale: its cost is validity, not width.}
\label{tab:Battr}
\setlength{\tabcolsep}{4pt}
\begin{tabular}{@{}llrlrrr@{}}
\toprule
scale & effect & ratio & 95\% CI of the median & $p$ & $n$ & floor\\
\midrule
median $w_{\mathrm{int}}$ & FEATNORM, fit split & $1.02$ & $[1.01,\,1.04]$ & $2.5\!\cdot\!10^{-5}$ & 26 & $3\!\cdot\!10^{-8}$\\
 & FEATNORM, all rows & $1.02$ & $[1,\,1.09]$ & $0.042$ & 11 & $9.8\!\cdot\!10^{-4}$\\
 & DELIVERY, drop dead & $0.817$ & $[0.592,\,0.978]$ & $9.8\!\cdot\!10^{-4}$ & 11 & $9.8\!\cdot\!10^{-4}$\\
 & DELIVERY, clamp & $0.893$ & $[0.634,\,0.978]$ & $0.027$ & 12 & $4.9\!\cdot\!10^{-4}$\\
 & INTERACTION & $1$ & $[1,\,1.02]$ & $0.32$ & 11 & $9.8\!\cdot\!10^{-4}$\\
\addlinespace
mean $w_{\mathrm{int}}$ & FEATNORM, fit split & $235$ & $[1.13,\,2442]$ & $9.8\!\cdot\!10^{-7}$ & 26 & $3\!\cdot\!10^{-8}$\\
 & FEATNORM, all rows & $133$ & $[1.46,\,3743]$ & $2.9\!\cdot\!10^{-3}$ & 11 & $9.8\!\cdot\!10^{-4}$\\
 & DELIVERY, drop dead & $0.945$ & $[0.621,\,0.984]$ & $0.042$ & 11 & $9.8\!\cdot\!10^{-4}$\\
 & DELIVERY, clamp & $0.988$ & $[0.849,\,1]$ & $0.042$ & 12 & $4.9\!\cdot\!10^{-4}$\\
 & INTERACTION & $1.05$ & $[0.993,\,1.28]$ & $0.083$ & 11 & $9.8\!\cdot\!10^{-4}$\\
\addlinespace
max $w_{\mathrm{int}}$ & FEATNORM, fit split & $2681$ & $[1,\,7082]$ & $7\!\cdot\!10^{-5}$ & 26 & $3\!\cdot\!10^{-8}$\\
 & FEATNORM, all rows & $2950$ & $[1,\,7832]$ & $0.014$ & 11 & $9.8\!\cdot\!10^{-4}$\\
 & DELIVERY, drop dead & $0.993$ & $[0.898,\,1.05]$ & $0.41$ & 11 & $9.8\!\cdot\!10^{-4}$\\
 & DELIVERY, clamp & $0.993$ & $[0.849,\,1]$ & $0.064$ & 12 & $4.9\!\cdot\!10^{-4}$\\
 & INTERACTION & $0.998$ & $[0.911,\,1]$ & $0.37$ & 11 & $9.8\!\cdot\!10^{-4}$\\
\addlinespace
interval score & FEATNORM, fit split & $111$ & $[1.07,\,1915]$ & $2.1\!\cdot\!10^{-6}$ & 26 & $3\!\cdot\!10^{-8}$\\
 & FEATNORM, all rows & $53.5$ & $[1.13,\,3014]$ & $2.9\!\cdot\!10^{-3}$ & 11 & $9.8\!\cdot\!10^{-4}$\\
 & DELIVERY, drop dead & $0.955$ & $[0.816,\,0.989]$ & $0.032$ & 11 & $9.8\!\cdot\!10^{-4}$\\
 & DELIVERY, clamp & $0.991$ & $[0.902,\,1]$ & $0.052$ & 12 & $4.9\!\cdot\!10^{-4}$\\
 & INTERACTION & $1$ & $[0.995,\,1.13]$ & $0.15$ & 11 & $9.8\!\cdot\!10^{-4}$\\
\addlinespace
\bottomrule
\end{tabular}
\end{table}

\begin{figure}[t]\centering
\includegraphics[width=\textwidth]{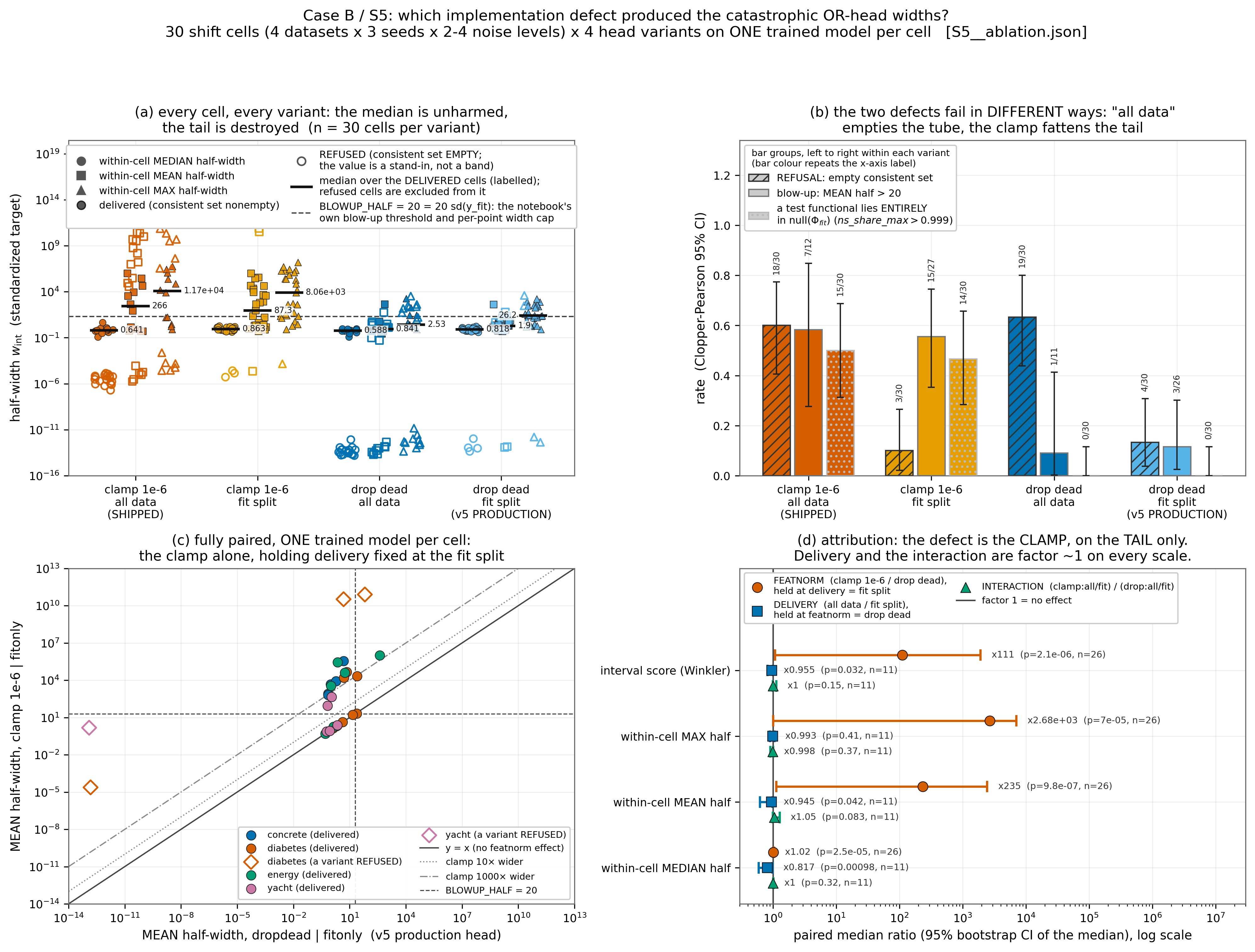}
\caption{The $2\times2$ width-defect ablation on $30$ shift cells, four
head variants per cell built on one trained model. The feature-scale
clamp is the width defect and it acts on the tail: it multiplies the
within-cell mean half-width by $235$ and the maximum by $2681$ while
moving the median by $1.02$. Delivering on all training rows is a
\emph{validity} defect, not a width defect: it moves the mean half-width
by a factor $0.95$, the maximum by $0.99$ and the within-cell median by
$0.82$ -- never by an order of magnitude on any scale, against the
clamp's three -- and instead empties the consistent set, taking
the refusal rate from $3$--$4$ of $30$ to $18$--$19$ of $30$.}
\label{fig:Babl}
\end{figure}

\subsection{S6: the 1D design-gap task at ten seeds}

\paragraph{The gap task at ten seeds (S6).} The shift evidence previously
rested on a single run (Table~\ref{tab:Bgap}, Fig.~\ref{fig:B},
Fig.~\ref{fig:Bregions}). Ten seeds re-draw the noise and the $37/13$ fit/
calibration split on a fixed design; the OR arms were recorded on all ten (no
refusal), the achieved split-conformal level is $13/14=0.9286$ on all ten, and
the conformal add-on chosen at $\kappa_\star$ is exactly zero on all ten,
which means the OR ball already contains every calibration residual and the
conformal repair costs nothing. $\kappa_\star$ is not $1.3$: it takes the
values $\{1.3,1.7,1.7,1.7,2.5,2.5,2.5,4.0,6.0,20.0\}$, median
$2.5$ $[1.7,3.6]$, unsaturated on $10/10$ seeds. The reported $1.3$ is the
minimum of that set.

\emph{The coverage claim survives; the width claim does not.} The OR band's
$1.00/1.00/1.00$ survives as a median: coverage of the truth is $1.000$ in all
three regions, at or above the nominal level on $10/10$ (gap), $10/10$
(in-distribution) and $8/10$ (extrapolation) seeds, with worst-seed values
$0.932$, $1.000$ and $0.750$. No other arm reaches the nominal level in
extrapolation on a single seed; in the design gap only split-CP and
normalized split-CP do, on $2$ of $10$ seeds each
(Table~\ref{tab:Bgap}). The widths do not survive. The reported
$1.08/0.49/2.08$ becomes
$2.80$ $[1.79, 3.75]$, $0.61$ $[0.57, 0.67]$ and
$5.16$ $[4.45, 11.06]$ -- the
in-distribution $0.49$ lies below the ten-seed minimum of $0.553$. The
constant split-CP half-width is not $0.84$ either: it is
$0.772$ $[0.578, 0.951]$, range $0.410$--$1.239$.

\emph{The $1.7\times$ does not hold as stated.} The in-distribution width
ratio split-CP / OR is $1.21$ $[0.92,1.49]$, range
$0.67$--$2.15$, with the OR band narrower on
only $6/10$ seeds (Wilcoxon $p=0.16$,
attainable floor $0.00195$): the effect is not separable from zero at this
seed depth. Two restatements do hold, both at the attainable floor. Against
the OR certificate itself, $w_{\rm int}$, the factor is
$2.67$ $[2.16,3.28]$ on
$10/10$ seeds ($p=0.002$); and
against the only baseline that actually attains the nominal level
in-distribution -- split conformal around a constant predictor, which covers
$0.977$ $[0.963,0.985]$ on $9/10$ seeds -- the OR band is
$2.99$ $[2.85,3.12]$ times
narrower on $10/10$ seeds ($p=0.002$).
The gap between the two versions of the claim is the observation term: the
delivered band is $w_{\rm int}+1.96\hat\sigma$, the learner's noise
estimate overshoots the known $\sigma=0.1$ by a factor
$1.76$ $[1.63, 1.84]$, and $1.96\hat\sigma=0.345$ is larger than
the certificate it is added to ($w_{\rm int}=0.249$ in-distribution).
Substituting the known $\sigma$ brings the factor back to
$1.61$ $[1.20,1.90]$, OR narrower on $8/10$ seeds ($p=0.020$), with the
paper's $1.71$ inside that interquartile range.
We therefore drop the $1.7\times$ claim in the form it was made and keep the two
restatements that hold at the attainable floor. The $1.7\times$ is therefore a
property of the OR geometry \emph{plus an oracle noise level}, not of the
delivered band.

\emph{What replaces it: the interval score.} On the decision-relevant loss
the OR band dominates every baseline in-distribution and in extrapolation. Its
miss-penalty is zero on the median seed in all three regions, and it pays any
penalty at all on $0/10$ seeds in-distribution, $1/10$ in the gap and $2/10$ in
extrapolation, against $8$--$10$ of $10$ for every conformal arm in every
region (the two lowest counts are split-CP's and normalized split-CP's
$8/10$ in the design gap; every other arm-by-region combination is $9$
or $10$ of $10$).
In-distribution the interval score is $1.23$ $[1.13,1.35]$ for the OR band
against $3.69$ (constant-predictor split-CP), $6.94$ (split-CP), $6.39$
(normalized split-CP) and $8.82$ (DAK), each $10/10$ seeds at
$p=0.00195$. In extrapolation it is $11.4$ $[10.0,22.1]$ against $30.7$, $45.0$, $45.2$ and $48.9$, each $9/10$ seeds ($p=0.0059$ against split-CP). In the
design gap alone the comparison is a tie ($5.61$ against $8.03$ for split-CP, $6/10$ seeds, $p=0.77$): there the OR band buys its coverage by being
$3.6\times$ wider than split-CP's constant band as a ratio of the two
ten-seed medians of Table~\ref{tab:Bgap} ($2.804$ against $0.772$; the
median of the ten per-seed ratios is $4.2\times$), and that width costs
about what the conformal band's misses cost. The paper should say so; ``wide
exactly where ignorance lives'' is the right description of the shape, but in
the gap it is not yet a win on loss.

\begin{table}[t]\centering\small
\caption{\textbf{Case B / S6: the 1D design-gap task at ten seeds.} Median
[IQR] over the ten seeds; every arm, every region. Nominal split-conformal
level $13/14=0.9286$, achieved on 10/10 seeds. The OR arms were recorded on
10/10 seeds (no refusals); \texttt{or1\_nf} is omitted because the conformal
add-on $w_{\mathrm{add}}$ is exactly zero on every seed, which makes it
identical to \texttt{or1}. ``$\geq$nom.'' counts seeds whose coverage of the
truth reaches the nominal level. Interval score at $\alpha=0.05$ against the
noiseless truth.}
\label{tab:Bgap}
\setlength{\tabcolsep}{4pt}
\begin{tabular}{@{}lrrrr@{}}
\toprule
arm & coverage & $\geq$nom. & mean half-width & interval score\\
\midrule
\multicolumn{5}{l}{\emph{region: design gap $|x|<1$}}\\
  DAK $1.96\hat\sigma(x)$ & $0.480$ [0.395, 0.588] & 0/10 & $0.596$ [0.589, 0.633] & $7.04$ [4.23, 18.6]\\
  split-CP (constant) & $0.473$ [0.436, 0.689] & 2/10 & $0.772$ [0.578, 0.951] & $8.03$ [3.51, 12.9]\\
  normalized split-CP & $0.473$ [0.426, 0.689] & 2/10 & $0.772$ [0.634, 0.967] & $7.45$ [3.46, 12.5]\\
  constant-predictor split-CP & $0.635$ [0.625, 0.662] & 0/10 & $1.809$ [1.762, 1.844] & $7.29$ [6.31, 7.62]\\
  OR band \texttt{or1} & $1.000$ [1.000, 1.000] & 10/10 & $2.804$ [1.786, 3.747] & $5.61$ [3.57, 7.49]\\
  OR certificate \texttt{or1\_int} & $1.000$ [1.000, 1.000] & 9/10 & $2.469$ [1.465, 3.372] & $5.37$ [4.14, 7.49]\\
\multicolumn{5}{l}{\emph{region: in-distribution $1\leq|x|\leq3$}}\\
  DAK $1.96\hat\sigma(x)$ & $0.530$ [0.493, 0.628] & 0/10 & $0.600$ [0.596, 0.631] & $8.82$ [6.83, 11.9]\\
  split-CP (constant) & $0.620$ [0.497, 0.765] & 2/10 & $0.772$ [0.578, 0.951] & $6.94$ [3.95, 11.9]\\
  normalized split-CP & $0.643$ [0.517, 0.753] & 1/10 & $0.788$ [0.622, 0.977] & $6.39$ [3.94, 11.3]\\
  constant-predictor split-CP & $0.977$ [0.963, 0.985] & 9/10 & $1.809$ [1.762, 1.844] & $3.69$ [3.62, 3.77]\\
  OR band \texttt{or1} & $1.000$ [1.000, 1.000] & 10/10 & $0.614$ [0.566, 0.674] & $1.23$ [1.13, 1.35]\\
  OR certificate \texttt{or1\_int} & $1.000$ [1.000, 1.000] & 10/10 & $0.249$ [0.244, 0.328] & $0.499$ [0.488, 0.656]\\
\multicolumn{5}{l}{\emph{region: extrapolation $|x|>3$}}\\
  DAK $1.96\hat\sigma(x)$ & $0.421$ [0.398, 0.444] & 0/10 & $0.609$ [0.595, 0.641] & $48.9$ [45, 49.8]\\
  split-CP (constant) & $0.454$ [0.411, 0.497] & 0/10 & $0.772$ [0.578, 0.951] & $45$ [41.8, 49]\\
  normalized split-CP & $0.467$ [0.424, 0.487] & 0/10 & $0.789$ [0.614, 0.985] & $45.2$ [41.8, 47.8]\\
  constant-predictor split-CP & $0.395$ [0.332, 0.447] & 0/10 & $1.809$ [1.762, 1.844] & $30.7$ [28.5, 32.5]\\
  OR band \texttt{or1} & $1.000$ [1.000, 1.000] & 8/10 & $5.163$ [4.453, 11.059] & $11.4$ [10, 22.1]\\
  OR certificate \texttt{or1\_int} & $1.000$ [0.954, 1.000] & 8/10 & $4.818$ [4.097, 10.682] & $12.8$ [9.47, 21.8]\\
\bottomrule
\end{tabular}
\end{table}

\begin{table}[t]\centering\small
\caption{\textbf{Case B / S6: which single-seed number survives ten seeds.}
Left, the single-seed value from the earlier corpus; right, the
ten-seed distribution of the same quantity recomputed from the corrected v5
corpus. ``in IQR'' means the single-seed value lies inside the interquartile
range of the ten seeds; ``in range'' means it lies inside the observed
min--max but outside the IQR; ``outside'' means it lies outside the observed
range altogether. This is the second of the paper's two deliberately cross-corpus
tables: the left column is one run of the earlier corpus and the right
column is this run.}
\label{tab:Bgapsurv}
\begin{tabular}{llrrrl}
\toprule
quantity & region & single seed & ten-seed median [IQR] & min--max & verdict\\
\midrule
split-CP coverage & gap & $0.58$ & $0.473$ [0.436, 0.689] & 0.243--1.000 & in IQR\\
split-CP coverage & ind & $0.70$ & $0.620$ [0.497, 0.765] & 0.293--1.000 & in IQR\\
split-CP coverage & ext & $0.43$ & $0.454$ [0.411, 0.497] & 0.395--0.500 & in IQR\\
norm. split-CP coverage & gap & $0.61$ & $0.473$ [0.426, 0.689] & 0.270--1.000 & in IQR\\
norm. split-CP coverage & ind & $0.72$ & $0.643$ [0.517, 0.753] & 0.313--1.000 & in IQR\\
norm. split-CP coverage & ext & $0.42$ & $0.467$ [0.424, 0.487] & 0.408--0.539 & in range\\
DAK coverage & gap & $0.46$ & $0.480$ [0.395, 0.588] & 0.243--0.676 & in IQR\\
DAK coverage & ind & $0.48$ & $0.530$ [0.493, 0.628] & 0.453--0.653 & in range\\
DAK coverage & ext & $0.30$ & $0.421$ [0.398, 0.444] & 0.289--0.474 & in range\\
OR coverage & gap & $1.00$ & $1.000$ [1.000, 1.000] & 0.932--1.000 & in IQR\\
OR coverage & ind & $1.00$ & $1.000$ [1.000, 1.000] & 1.000--1.000 & in IQR\\
OR coverage & ext & $1.00$ & $1.000$ [1.000, 1.000] & 0.750--1.000 & in IQR\\
split-CP width $q$ & gap & $0.84$ & $0.772$ [0.578, 0.951] & 0.410--1.239 & in IQR\\
split-CP width $q$ & ind & $0.84$ & $0.772$ [0.578, 0.951] & 0.410--1.239 & in IQR\\
split-CP width $q$ & ext & $0.84$ & $0.772$ [0.578, 0.951] & 0.410--1.239 & in IQR\\
OR width & gap & $1.08$ & $2.804$ [1.786, 3.747] & 1.000--11.757 & in range\\
OR width & ind & $0.49$ & $0.614$ [0.566, 0.674] & 0.553--0.849 & \textbf{outside}\\
OR width & ext & $2.08$ & $5.163$ [4.453, 11.059] & 2.025--33.718 & in range\\
\bottomrule
\end{tabular}
\end{table}

\begin{figure}[t]\centering
\includegraphics[width=\textwidth]{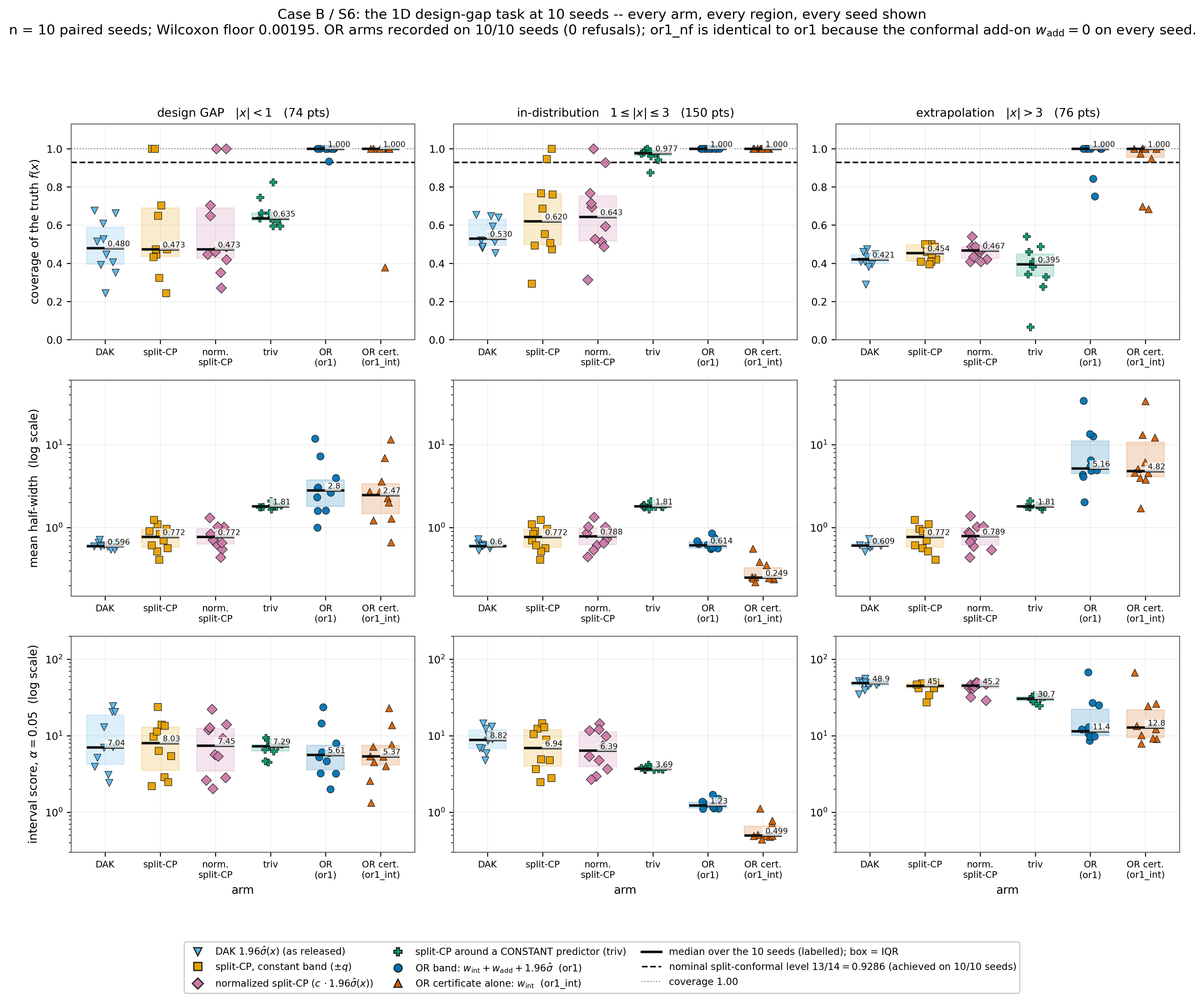}
\caption{Case B gap task by region over ten seeds: coverage of the truth
and interval score for every arm in the design gap, in-distribution and
in extrapolation. Every seed is plotted. The OR band is the only arm that
reaches the nominal level in extrapolation on any seed; in the gap
split-CP and normalized split-CP reach it on $2$ of $10$ seeds each,
drawn at coverage $1.000$ in the top-left panel. The OR band
dominates on interval score in-distribution and in extrapolation while
tying in the gap.}
\label{fig:Bregions}
\end{figure}

\clearpage

\section{Case C: full results (ALAS-BO)}\label{app:caseC}

\subsection{Regret, per benchmark and paired}

ALAS-BO \citep{ALAS2026} couples a learnable $\alpha$-stable mixture
kernel with UCB ($\beta{=}0.2$) or EI \citep{Srinivas2010,%
ChowdhuryGopalan2017}. Deviations (all arms): pinned
\texttt{botorch 0.18.1}; mixtures $Q{=}3$ -- a $Q{=}7$ fidelity check on
Branin (2 seeds, earlier corpus) preserves method ordering (EI
$0.17$--$0.21$, OR-LCB $0.13$--$0.39$). Benchmarks: Branin-2D,
Hartmann-6, Griewank-5D, Ackley-5D and Levy-10D; the run record pins
the dimension of Levy and Ackley explicitly, and Griewank's $d=5$ is
confirmed empirically -- its median nearest-neighbour distance at
$t=0$ is bitwise identical to Ackley's in all $32$ shared (arm, seed)
comparisons. The corpus is $167$ trajectories from a single session,
all of length $19$, $6$ or $7$ seeds per (benchmark, arm): $7$ for
Branin, Griewank and Hartmann-6 throughout, $6$ for Levy and for the
three Ackley $\kappa$ arms. Every paired test below pairs only on
shared seeds. All $80$ (benchmark, arm, seed) trajectories this
corpus shares with an earlier CPU corpus are \emph{bitwise identical}
to it at all $19$ iterations, so the new seeds extend rather than
overturn the earlier entries -- the one place in this paper where a
cross-machine reproduction is exact. Writing $m(x)$ for the posterior mean (the BO
symbol $\mu$ is avoided to prevent a clash with the dual parameter),
\textbf{OR-LCB} picks $\argmin_x[m(x)-c_\star(x)]$ with $c_\star$ the
information radius ($\eps=\rho_{\mathrm{oc}}$, $\eta=\hat\sigma\sqrt t$);
the direct floor-law test \textbf{$\kappa{\times}$floor} uses
$\kappa c_\star(x)$, $\kappa\in\{2,5\}$.

\textbf{Results} (Table~\ref{tab:C}, Table~\ref{tab:CpairEI},
Fig.~\ref{fig:C}). Paired against
EI on shared seeds, OR-LCB's final-regret median differences are
$+0.063$ (Branin, $n{=}7$, $p{=}0.078$), $+0.211$ (Hartmann-6, $n{=}7$,
$p{=}0.063$), $+5.07$ (Griewank-5D, $n{=}7$, $p{=}0.219$), $+0.595$
(Ackley-5D, $n{=}6$, $p{=}0.313$) and $0.000$ (Levy-10D, $n{=}6$,
$p{=}1$, five of six seeds tied exactly); pooling the per-benchmark
normalized paired $\log_{10}$ ratios over all $33$ complete blocks,
OR-LCB is worse than EI by $+0.075$ dex ($p{=}0.0011$). EI is the
stronger default. OR-LCB still avoids UCB's Branin blow-up (worst $0.34$
against $3.88$), though EI's worst there ($0.19$) is smaller still. At
these seed counts a per-benchmark test can for the first time reach
significance -- the attainable minimum two-sided Wilcoxon $p$ is
$0.0156$ at $n{=}7$ and $0.0313$ at $n{=}6$, against $0.0625$ at the
five seeds of an earlier revision, where no per-benchmark test could
clear $0.05$ by construction -- and exactly one final-regret contrast
does: $2{\times}$floor is worse than EI on Branin (median $+0.122$,
$[+0.026,+0.282]$, six of seven seeds, $p{=}0.031$). A Friedman test
over the $33$ blocks separates the arms ($\chi^2{=}19.33$,
$p{=}6.8\!\cdot\!10^{-4}$), with mean ranks EI $2.41$, UCB $2.64$,
$5{\times}$floor $3.05$, $\kappa{=}1$ $3.42$, $2{\times}$floor $3.49$.

\textbf{A second ordering} (Table~\ref{tab:CaucEI}). Ranking by AUC -- mean regret over the
$19$-point trajectory -- disagrees with the final-regret ranking on three
of five benchmarks. On Ackley-5D UCB is first by final regret ($5.58$)
and third by AUC; on Griewank-5D EI is first by final ($8.31$) and last
by AUC ($32.30$); on Branin UCB is first by final ($0.134$) and last by
AUC ($2.01$), while $5{\times}$floor is fourth by final and first by AUC
($0.945$). The Friedman ordering is unchanged in direction
($\chi^2{=}25.89$, $p{=}3.3\!\cdot\!10^{-5}$; EI $2.09$, UCB $2.88$,
$5{\times}$floor $3.09$, $2{\times}$floor $3.44$, $\kappa{=}1$ $3.50$),
and on AUC three per-benchmark contrasts clear $0.05$, all on Branin and
all against EI: UCB and $2{\times}$floor at $p{=}0.0156$, the attainable
floor, losing on seven of seven seeds, and OR-LCB at $p{=}0.031$ on six
of seven. Where a marginal median and a paired difference disagree the
paired difference is the statement to trust: on Branin AUC,
$5{\times}$floor has the lower marginal median ($0.945$ against EI's
$1.120$) yet loses to EI on six of seven paired seeds ($+0.150$,
$p{=}0.109$).

\begin{table}[t]\centering\footnotesize
\caption{Case C: median (worst) final simple regret at the same $23$-evaluation
budget ($n_{\mathrm{init}}{=}5$ then $18$ BO steps), from one GPU session
(\texttt{botorch 0.18.1}, $Q{=}3$). Seed counts are uneven and are given in the
$n$ column: the Ackley $\kappa$ arms and all Levy arms have $6$ seeds, everything
else has $7$; every paired test pairs only on shared seeds. All $80$
(benchmark, arm, seed) trajectories this corpus shares with the earlier CPU corpus
are \emph{bitwise identical} to it, so the added seeds extend rather than
overturn the earlier entries. or\_full and or\_ei were not in this run; their
$5$-seed values are unchanged and reported in App.~\ref{app:exp}.}
\label{tab:C}
\setlength{\tabcolsep}{4pt}
\begin{tabular}{@{}lcccccc@{}}
\toprule
& $n$ & EI & UCB & OR-LCB ($\kappa{=}1$) & $2{\times}$floor & $5{\times}$floor\\
\midrule
Branin & 7 & 0.14 (0.19) & 0.13 (3.88) & 0.20 (0.34) & 0.29 (0.58) & 0.26 (0.89)\\
Hartmann-6 & 7 & 0.70 (1.65) & 0.96 (2.05) & 0.96 (2.49) & 0.96 (2.49) & 0.96 (2.49)\\
Griewank & 7 & 8.31 (20.79) & 9.71 (25.49) & 12.40 (22.71) & 10.98 (19.98) & 9.92 (19.98)\\
Ackley-5D & 7/6 & 6.59 (8.95) & 5.58 (9.60) & 7.92 (9.60) & 7.92 (9.60) & 7.34 (9.60)\\
Levy-10D & 6 & 4.47 (9.17) & 4.47 (9.17) & 4.47 (9.17) & 4.47 (9.17) & 4.47 (9.17)\\
\bottomrule
\end{tabular}
\end{table}

\begin{table}[t]\centering\scriptsize
\caption{Case C, final simple regret paired against EI. Seed counts are uneven, so each row pairs only on the seeds both arms have. ``Better/worse/tied'' counts seeds on which the row arm ends below/above/exactly at EI; an exact tie is a bitwise-identical trajectory. ``Floor'' is the smallest two-sided Wilcoxon $p$ attainable at that $n$ ($2^{1-n}$): at the $5$ seeds of the earlier revision it was $0.0625$, so no per-benchmark test could reach $0.05$ by construction. The pooled rows test the per-benchmark normalized paired $\log_{10}$ ratio against zero over all $33$ complete blocks.}
\label{tab:CpairEI}
\setlength{\tabcolsep}{4pt}
\begin{tabular}{@{}llrrrrrr@{}}
\toprule
Benchmark & arm & $n$ & median diff. & bootstrap $95\%$ CI of the median & better/worse/tied & Wilcoxon $p$ & floor\\
\midrule
Levy-10D & UCB & $6$ & $+0.0000$ & $[-0.0606,\,+0.0000]$ & $1/0/5$ & $1$ & $0.03125$\\
Levy-10D & OR-LCB & $6$ & $+0.0000$ & $[-0.0606,\,+0.0000]$ & $1/0/5$ & $1$ & $0.03125$\\
Levy-10D & $2{\times}$floor & $6$ & $+0.0000$ & $[-0.0606,\,+0.0000]$ & $1/0/5$ & $1$ & $0.03125$\\
Levy-10D & $5{\times}$floor & $6$ & $+0.0000$ & $[-0.0606,\,+0.0000]$ & $1/0/5$ & $1$ & $0.03125$\\
Hartmann-6 & UCB & $7$ & $+0.2105$ & $[+0.0000,\,+0.3871]$ & $0/5/2$ & $0.0625$ & $0.01562$\\
Hartmann-6 & OR-LCB & $7$ & $+0.2105$ & $[+0.0000,\,+0.3871]$ & $0/5/2$ & $0.0625$ & $0.01562$\\
Hartmann-6 & $2{\times}$floor & $7$ & $+0.2105$ & $[+0.0000,\,+0.3871]$ & $0/5/2$ & $0.0625$ & $0.01562$\\
Hartmann-6 & $5{\times}$floor & $7$ & $+0.2105$ & $[+0.0000,\,+0.3871]$ & $0/5/2$ & $0.0625$ & $0.01562$\\
Ackley-5D & UCB & $6$ & $+0.3485$ & $[-2.6163,\,+1.7060]$ & $2/3/1$ & $1$ & $0.03125$\\
Ackley-5D & OR-LCB & $6$ & $+0.5954$ & $[-0.4959,\,+2.7887]$ & $1/4/1$ & $0.3125$ & $0.03125$\\
Ackley-5D & $2{\times}$floor & $6$ & $+0.5954$ & $[-0.4959,\,+2.7887]$ & $1/4/1$ & $0.3125$ & $0.03125$\\
Ackley-5D & $5{\times}$floor & $6$ & $+0.1823$ & $[-0.8769,\,+1.2199]$ & $1/3/2$ & $0.875$ & $0.03125$\\
Griewank & UCB & $7$ & $-1.2086$ & $[-2.0878,\,+6.8006]$ & $4/2/1$ & $1$ & $0.01562$\\
Griewank & OR-LCB & $7$ & $+5.0683$ & $[-2.0878,\,+9.2188]$ & $2/5/0$ & $0.2188$ & $0.01562$\\
Griewank & $2{\times}$floor & $7$ & $+1.6075$ & $[-2.0878,\,+9.2188]$ & $2/4/1$ & $0.4375$ & $0.01562$\\
Griewank & $5{\times}$floor & $7$ & $+0.1839$ & $[-2.3604,\,+5.3799]$ & $3/4/0$ & $0.8125$ & $0.01562$\\
Branin & UCB & $7$ & $-0.0060$ & $[-0.0197,\,+0.0006]$ & $5/2/0$ & $0.4688$ & $0.01562$\\
Branin & OR-LCB & $7$ & $+0.0629$ & $[+0.0063,\,+0.1500]$ & $1/6/0$ & $0.07812$ & $0.01562$\\
Branin & $2{\times}$floor & $7$ & $+0.1218$ & $[+0.0259,\,+0.2819]$ & $1/6/0$ & $0.03125$ & $0.01562$\\
Branin & $5{\times}$floor & $7$ & $+0.0986$ & $[-0.0308,\,+0.2734]$ & $2/5/0$ & $0.07812$ & $0.01562$\\
\midrule
pooled log-ratio & UCB & $33$ & $+0.0000$ & $[-0.0011,\,+0.0016]$ & $12/12/9$ & $0.6071$ & $2.33{\times}10^{-10}$\\
pooled log-ratio & OR-LCB & $33$ & $+0.0750$ & $[+0.0000,\,+0.1735]$ & $5/20/8$ & $0.001079$ & $2.33{\times}10^{-10}$\\
pooled log-ratio & $2{\times}$floor & $33$ & $+0.0718$ & $[+0.0000,\,+0.1625]$ & $5/19/9$ & $7.48{\times}10^{-4}$ & $2.33{\times}10^{-10}$\\
pooled log-ratio & $5{\times}$floor & $33$ & $+0.0167$ & $[+0.0000,\,+0.0851]$ & $7/17/9$ & $0.02584$ & $2.33{\times}10^{-10}$\\
\bottomrule
\end{tabular}
\end{table}

\begin{table}[t]\centering\scriptsize
\caption{Case C, final simple regret paired against OR-LCB ($\kappa{=}1$), same conventions as Tab.~\ref{tab:CpairEI}. The tie counts are the $\kappa$-inertness made visible: $22/33$ blocks tie exactly between $\kappa{=}1$ and $\kappa{=}2$, $17/33$ between $\kappa{=}1$ and $\kappa{=}5$, and $20/33$ between $\kappa{=}2$ and $\kappa{=}5$.}
\label{tab:CpairOR}
\setlength{\tabcolsep}{4pt}
\begin{tabular}{@{}llrrrrrr@{}}
\toprule
Benchmark & arm & $n$ & median diff. & bootstrap $95\%$ CI of the median & better/worse/tied & Wilcoxon $p$ & floor\\
\midrule
Levy-10D & EI & $6$ & $+0.0000$ & $[+0.0000,\,+0.0606]$ & $0/1/5$ & $1$ & $0.03125$\\
Levy-10D & UCB & $6$ & $+0.0000$ & $[+0.0000,\,+0.0000]$ & $0/0/6$ & $1$ & $0.03125$\\
Levy-10D & $2{\times}$floor & $6$ & $+0.0000$ & $[+0.0000,\,+0.0000]$ & $0/0/6$ & $1$ & $0.03125$\\
Levy-10D & $5{\times}$floor & $6$ & $+0.0000$ & $[+0.0000,\,+0.0000]$ & $0/0/6$ & $1$ & $0.03125$\\
Hartmann-6 & EI & $7$ & $-0.2105$ & $[-0.3871,\,+0.0000]$ & $5/0/2$ & $0.0625$ & $0.01562$\\
Hartmann-6 & UCB & $7$ & $+0.0000$ & $[+0.0000,\,+0.0000]$ & $1/0/6$ & $1$ & $0.01562$\\
Hartmann-6 & $2{\times}$floor & $7$ & $+0.0000$ & $[+0.0000,\,+0.0000]$ & $0/0/7$ & $1$ & $0.01562$\\
Hartmann-6 & $5{\times}$floor & $7$ & $+0.0000$ & $[+0.0000,\,+0.0000]$ & $0/0/7$ & $1$ & $0.01562$\\
Ackley-5D & EI & $6$ & $-0.5954$ & $[-2.7887,\,+0.4959]$ & $4/1/1$ & $0.3125$ & $0.03125$\\
Ackley-5D & UCB & $6$ & $-0.7418$ & $[-2.7081,\,+0.0000]$ & $3/0/3$ & $0.25$ & $0.03125$\\
Ackley-5D & $2{\times}$floor & $6$ & $+0.0000$ & $[+0.0000,\,+0.0000]$ & $0/0/6$ & $1$ & $0.03125$\\
Ackley-5D & $5{\times}$floor & $6$ & $-0.3810$ & $[-1.9819,\,+0.0000]$ & $3/0/3$ & $0.25$ & $0.03125$\\
Griewank & EI & $7$ & $-5.0683$ & $[-9.2188,\,+2.0878]$ & $5/2/0$ & $0.2188$ & $0.01562$\\
Griewank & UCB & $7$ & $+0.0000$ & $[-6.2769,\,+0.9336]$ & $3/2/2$ & $0.3125$ & $0.01562$\\
Griewank & $2{\times}$floor & $7$ & $-0.4871$ & $[-2.7244,\,+0.0000]$ & $4/0/3$ & $0.125$ & $0.01562$\\
Griewank & $5{\times}$floor & $7$ & $-2.4762$ & $[-7.4287,\,+0.0000]$ & $5/1/1$ & $0.09375$ & $0.01562$\\
Branin & EI & $7$ & $-0.0629$ & $[-0.1500,\,-0.0063]$ & $6/1/0$ & $0.07812$ & $0.01562$\\
Branin & UCB & $7$ & $-0.0623$ & $[-0.1697,\,-0.0067]$ & $6/1/0$ & $0.2969$ & $0.01562$\\
Branin & $2{\times}$floor & $7$ & $+0.1154$ & $[+0.0084,\,+0.2309]$ & $1/6/0$ & $0.1094$ & $0.01562$\\
Branin & $5{\times}$floor & $7$ & $+0.0840$ & $[-0.0184,\,+0.2105]$ & $2/5/0$ & $0.2969$ & $0.01562$\\
\midrule
pooled log-ratio & EI & $33$ & $-0.0750$ & $[-0.1735,\,+0.0000]$ & $20/5/8$ & $0.001079$ & $2.33{\times}10^{-10}$\\
pooled log-ratio & UCB & $33$ & $+0.0000$ & $[-0.0830,\,+0.0000]$ & $13/3/17$ & $0.01509$ & $2.33{\times}10^{-10}$\\
pooled log-ratio & $2{\times}$floor & $33$ & $+0.0000$ & $[+0.0000,\,+0.0000]$ & $5/6/22$ & $0.4236$ & $2.33{\times}10^{-10}$\\
pooled log-ratio & $5{\times}$floor & $33$ & $+0.0000$ & $[+0.0000,\,+0.0000]$ & $10/6/17$ & $0.5695$ & $2.33{\times}10^{-10}$\\
\bottomrule
\end{tabular}
\end{table}

\begin{table}[t]\centering\scriptsize
\caption{Case C, AUC (mean regret over the $19$-point trajectory) paired against EI; conventions as in Tab.~\ref{tab:CpairEI}. AUC is reported as a second ordering because final regret ignores the path; the two orderings disagree on three of five benchmarks.}
\label{tab:CaucEI}
\setlength{\tabcolsep}{4pt}
\begin{tabular}{@{}llrrrrrr@{}}
\toprule
Benchmark & arm & $n$ & median diff. & bootstrap $95\%$ CI of the median & better/worse/tied & Wilcoxon $p$ & floor\\
\midrule
Levy-10D & UCB & $6$ & $+0.0000$ & $[+0.0000,\,+0.5668]$ & $0/1/5$ & $1$ & $0.03125$\\
Levy-10D & OR-LCB & $6$ & $+0.0000$ & $[+0.0000,\,+0.5668]$ & $0/1/5$ & $1$ & $0.03125$\\
Levy-10D & $2{\times}$floor & $6$ & $+0.0000$ & $[+0.0000,\,+0.5668]$ & $0/1/5$ & $1$ & $0.03125$\\
Levy-10D & $5{\times}$floor & $6$ & $+0.0000$ & $[+0.0000,\,+0.5668]$ & $0/1/5$ & $1$ & $0.03125$\\
Hartmann-6 & UCB & $7$ & $+0.0935$ & $[+0.0000,\,+0.2253]$ & $1/5/1$ & $0.0625$ & $0.01562$\\
Hartmann-6 & OR-LCB & $7$ & $+0.0935$ & $[+0.0000,\,+0.2253]$ & $1/5/1$ & $0.0625$ & $0.01562$\\
Hartmann-6 & $2{\times}$floor & $7$ & $+0.0935$ & $[+0.0000,\,+0.2253]$ & $1/5/1$ & $0.0625$ & $0.01562$\\
Hartmann-6 & $5{\times}$floor & $7$ & $+0.0935$ & $[+0.0000,\,+0.2253]$ & $1/5/1$ & $0.0625$ & $0.01562$\\
Ackley-5D & UCB & $6$ & $-0.1948$ & $[-0.8289,\,+1.2636]$ & $4/2/0$ & $0.8438$ & $0.03125$\\
Ackley-5D & OR-LCB & $6$ & $+0.7537$ & $[-0.4963,\,+1.6545]$ & $1/5/0$ & $0.3125$ & $0.03125$\\
Ackley-5D & $2{\times}$floor & $6$ & $+0.7537$ & $[-0.4963,\,+1.6545]$ & $1/5/0$ & $0.3125$ & $0.03125$\\
Ackley-5D & $5{\times}$floor & $6$ & $+0.5927$ & $[-0.5422,\,+1.2413]$ & $1/5/0$ & $0.3125$ & $0.03125$\\
Griewank & UCB & $7$ & $-0.7294$ & $[-2.7192,\,+0.8308]$ & $4/3/0$ & $0.5781$ & $0.01562$\\
Griewank & OR-LCB & $7$ & $+0.8427$ & $[-2.1666,\,+7.9078]$ & $2/5/0$ & $0.4688$ & $0.01562$\\
Griewank & $2{\times}$floor & $7$ & $+0.5759$ & $[-2.1666,\,+7.6769]$ & $2/5/0$ & $0.4688$ & $0.01562$\\
Griewank & $5{\times}$floor & $7$ & $+0.1303$ & $[-3.5321,\,+7.6769]$ & $3/4/0$ & $0.6875$ & $0.01562$\\
Branin & UCB & $7$ & $+0.8898$ & $[+0.0565,\,+1.7596]$ & $0/7/0$ & $0.01562$ & $0.01562$\\
Branin & OR-LCB & $7$ & $+0.2480$ & $[+0.1060,\,+1.2761]$ & $1/6/0$ & $0.03125$ & $0.01562$\\
Branin & $2{\times}$floor & $7$ & $+0.2762$ & $[+0.1918,\,+0.6979]$ & $0/7/0$ & $0.01562$ & $0.01562$\\
Branin & $5{\times}$floor & $7$ & $+0.1499$ & $[+0.0818,\,+0.8243]$ & $1/6/0$ & $0.1094$ & $0.01562$\\
\midrule
pooled log-ratio & UCB & $33$ & $+0.0058$ & $[+0.0000,\,+0.0601]$ & $9/18/6$ & $0.01426$ & $2.33{\times}10^{-10}$\\
pooled log-ratio & OR-LCB & $33$ & $+0.0255$ & $[+0.0000,\,+0.0733]$ & $5/22/6$ & $9.97{\times}10^{-4}$ & $2.33{\times}10^{-10}$\\
pooled log-ratio & $2{\times}$floor & $33$ & $+0.0532$ & $[+0.0017,\,+0.0876]$ & $4/23/6$ & $3.44{\times}10^{-4}$ & $2.33{\times}10^{-10}$\\
pooled log-ratio & $5{\times}$floor & $33$ & $+0.0246$ & $[+0.0000,\,+0.0627]$ & $6/21/6$ & $0.00425$ & $2.33{\times}10^{-10}$\\
\bottomrule
\end{tabular}
\end{table}

\begin{figure}[t]\centering
\includegraphics[width=\textwidth]{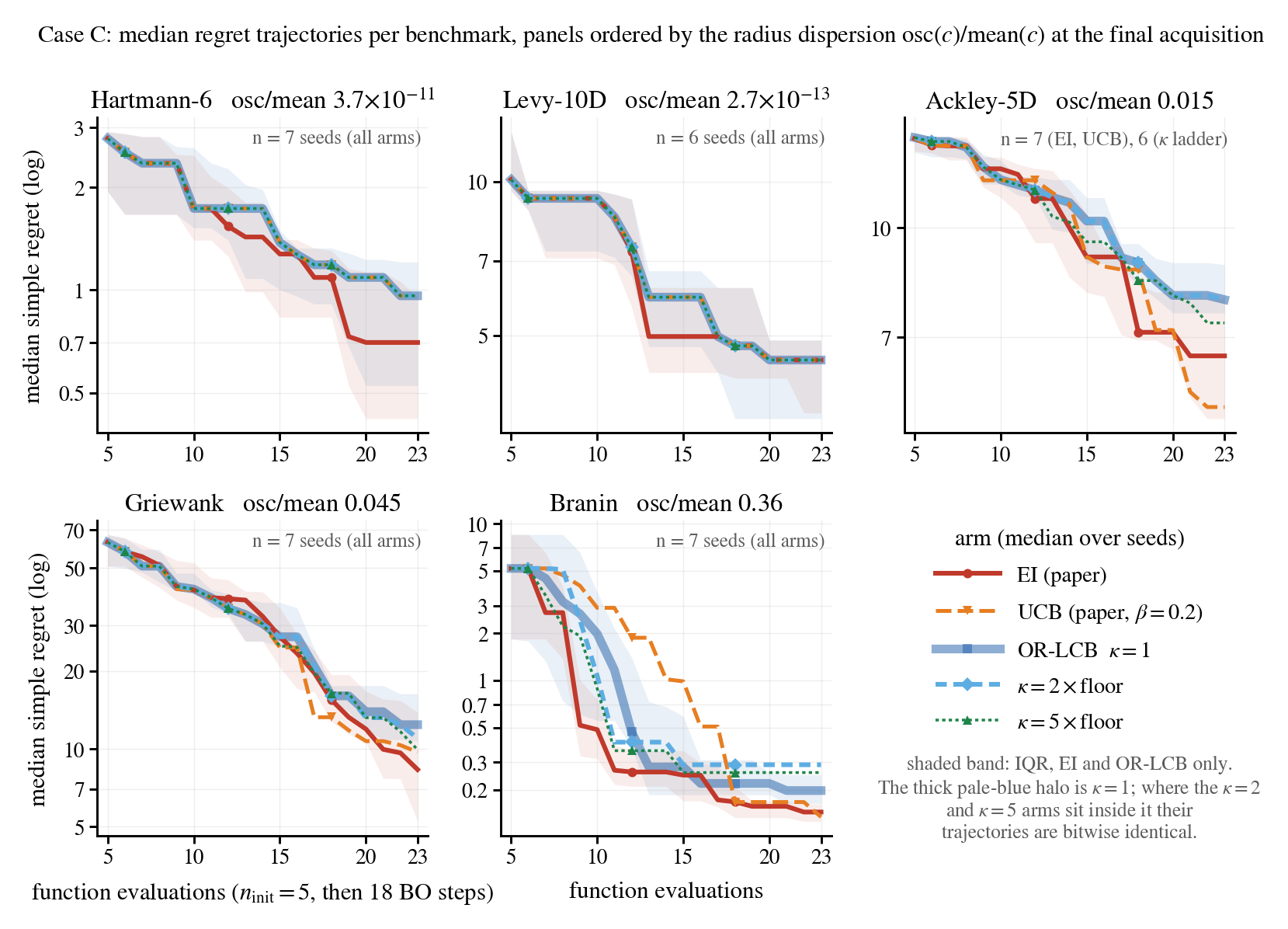}
\caption{Case C median regret per benchmark, panels ordered by radius
dispersion at the final acquisition -- the two saturated benchmarks sit
below the certifier's $\mu$-grid floor, so their order between themselves
is arbitrary (Rem.~\ref{rem:gridfloor}); all arms share the initial design.
Line widths are chosen so that a bitwise-identical overlay is visible:
the thick pale halo is $\kappa{=}1$, and where the $\kappa{=}2$ (dashed)
and $\kappa{=}5$ (dotted) arms sit inside it the trajectories agree
element for element. Seed counts are printed in each panel.}
\label{fig:C}
\end{figure}

\begin{figure}[t]\centering
\includegraphics[width=\textwidth]{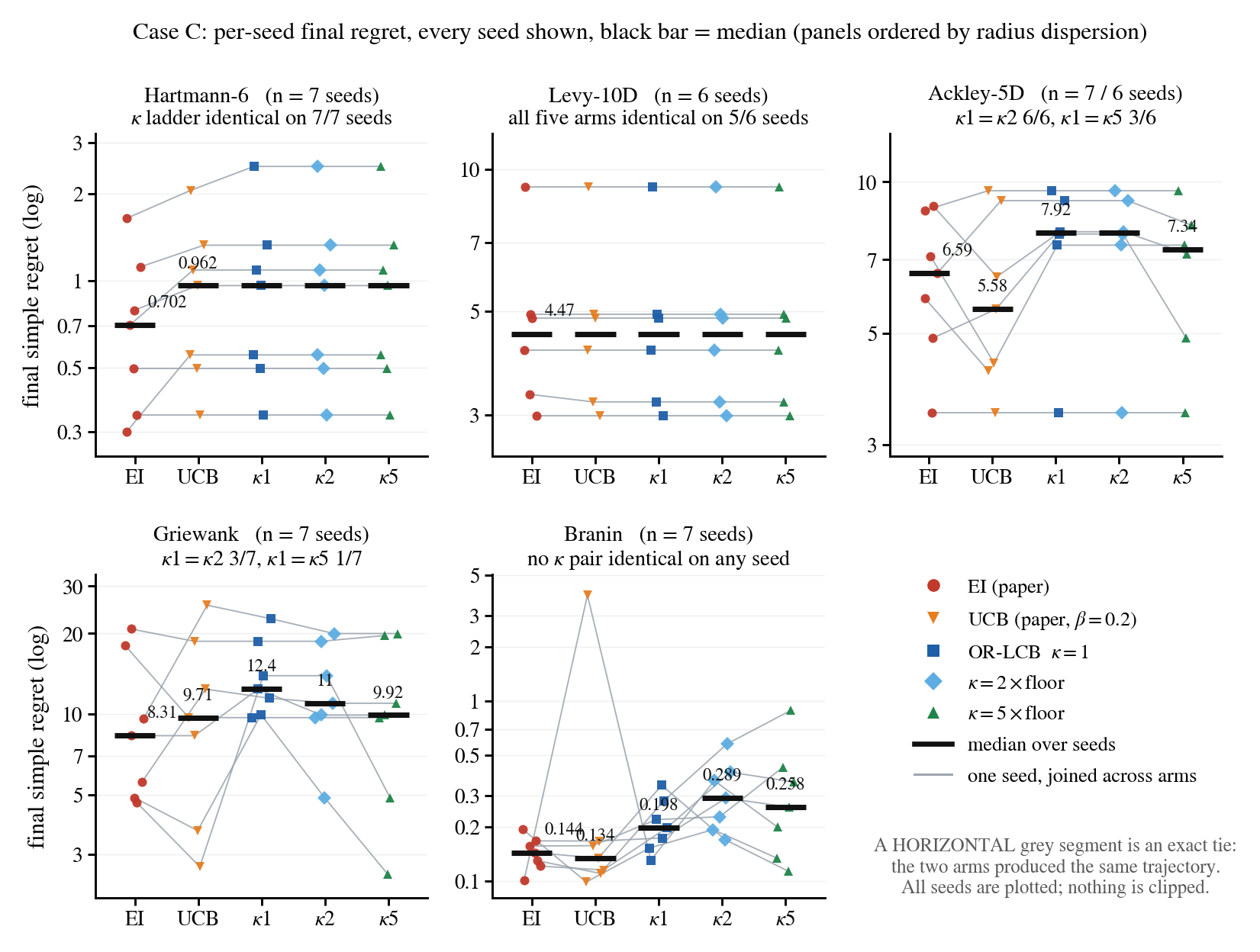}
\caption{Case C per-seed final regret; every seed is plotted, nothing is
clipped, the black bar is the median. Each grey polyline is one seed
followed across the five arms, so a horizontal segment is an exact tie in
\emph{final} regret; on this corpus every such tie is also a
bitwise-identical trajectory (Table~\ref{tab:CpairOR}), which the plot
itself cannot show. The ties fill the $\kappa$
columns on Hartmann-6 and Levy-10D and vanish on Branin.}
\label{fig:Cfin}
\end{figure}

\subsection{The kappa ladder and its grading}

\textbf{The $\kappa$ ladder, benchmark by benchmark}
(Table~\ref{tab:Ckappa}, Table~\ref{tab:CpairOR},
Fig.~\ref{fig:Cfin}). $\kappa$ is now run at
$\{1,2,5\}$ on all five benchmarks rather than at $\{1,2\}$ on two of
them, and the two rungs behave differently. \textbf{$\kappa{=}2$ is
inert wherever $\kappa{=}1$ is}: it selects the same candidate at every
step on $6/6$ Levy-10D, $7/7$ Hartmann-6 and $6/6$ Ackley-5D seeds, and
it departs only on Griewank-5D ($1/7$ inert) and Branin-2D ($0/7$).
\textbf{$\kappa{=}5$ is not.} It stays identical on Levy-10D ($6/6$) and
almost so on Hartmann-6 ($6/7$; the one departure is at $t{=}15$ and the
running minimum absorbs it, which is why the regret trajectory records
$7/7$), but it breaks away from $\kappa{=}1$ on $3$ of $6$ Ackley-5D
seeds and on $7$ of $7$ Griewank-5D seeds. The Ackley departures are
genuine, not numerical: they occur at $t{=}12$, $7$ and $10$ at
oscillation-to-gap ratios $2.30$, $1.76$ and $2.38$, with oscillations
of order $10^{-2}$, twelve orders of magnitude above machine epsilon.
And where $\kappa{=}5$ acts it \emph{helps}: on Ackley-5D it improves
the final regret on all three divergent seeds ($-0.76$, $-0.97$,
$-2.99$) and ties on the other three, moving the median $7.92\to7.34$;
on Griewank-5D it is better on five seeds, worse on one and tied on one,
moving the median $12.40\to9.92$ against $10.98$ at $\kappa{=}2$. No
paired contrast clears $0.05$ (Griewank $-2.48$, $p{=}0.094$;
Ackley-5D $-0.38$, $p{=}0.25$; Hartmann-6 and Levy-10D are ties
throughout; Branin $+0.084$, $p{=}0.297$; pooled median $0.000$,
$p{=}0.57$ with $17$ of $33$ blocks exact ties), but $\kappa{=}5$ never
worsens a median, improves two, and has the better Friedman mean rank of
the three rungs. The claim of an earlier revision that ``$\kappa{=}5$
does not help'' rested on a two-seed Griewank median of $14.85$; those
two trajectories reproduce bitwise here, and the median over all seven
seeds is $9.92$. We withdraw it. The defensible statement is: the
inertness of the certified width is a property of the \emph{size} of the
scaling as much as of the objective, $\kappa{=}2$ is inert on three of
five benchmarks and $\kappa{=}5$ on two, and on the two benchmarks where
$\kappa{=}5$ does act it moves the median the right way without
reaching significance at six or seven seeds.

\textbf{The grading, and what it does and does not resolve}
(Fig.~\ref{fig:Cmech}). Ordering the benchmarks by the radius dispersion
$\mathrm{osc}(\hat c)/\mathrm{mean}(\hat c)$, the action-level inertness
of the full ladder is monotone non-increasing -- Levy-10D $6/6$,
Hartmann-6 $6/7$, Ackley-5D $3/6$, Griewank-5D $0/7$, Branin-2D $0/7$ --
and within every benchmark it is monotone non-increasing in
$|\Delta\kappa|$. That is the shape of a graded law rather than a binary
property of two benchmarks, and we report it as such. It is not,
however, a resolved five-level ordering. Pairwise Fisher exact tests
with Holm correction over the ten benchmark pairs resolve exactly two
groups, $\{$Branin, Griewank$\}$ and $\{$Hartmann-6, Levy$\}$ (four
comparisons at Holm $p=0.006$--$0.037$); Ackley-5D is indistinguishable
from both (Holm $p\ge0.42$), and the two comparisons the word ``graded''
would most need -- Branin against Griewank at the bottom and Hartmann-6
against Levy at the top -- are both $p=1.00$. Three further limits.
First, dispersion and \emph{dimension} are confounded across these five
objectives (exact-permutation Spearman against inertness: dispersion
$-0.97$, $p=0.033$; dimension $+0.92$, $p=0.067$), and the only
dimension-controlled comparison available -- Ackley against Griewank,
which share $d=5$ and, we verified, bitwise-identical initial designs
and candidate clouds -- is Fisher $p=0.070$. We cannot separate the two
explanations from this corpus. Second, the ordering is a budget
snapshot: recomputing inertness at a truncated budget of six iterations
gives Griewank $7/7$ and Ackley $6/6$ and destroys the ordering
entirely. The variable it actually tracks is the iteration at which the
design escapes the flat-radius regime (median first $t$ with
$\mathrm{osc}/\mathrm{mean}>10^{-6}$: Branin $0$, Griewank $5$, Ackley
$6$, Hartmann-6 $9$, Levy never within $18$), and the first $\kappa$
divergence always follows that escape. Escape time is a
budget-by-dimension quantity, not a property of the objective. Third,
the dispersion statistic does not explain \emph{which} seeds diverge
inside a benchmark: on Ackley the seed with the largest bound violation
over its run ($4\,\mathrm{osc}/\Delta=26.9$) is inert while a seed at
$4.7$ is not, and the same inversion occurs on Hartmann-6 ($222$ inert,
$1.4$ divergent). $\mathrm{osc}(\hat c)/\mathrm{mean}(\hat c)$ is a
between-benchmark descriptor and we use it as one.

\textbf{The inert end is where every acquisition is inert, and we say so
with the corpus's own control.} Levy-10D and Hartmann-6 are not
demonstrations that the OR radius specifically is flat; they are
configurations in which the acquisition function does not matter at all.
On Levy-10D five of six seeds produce one action path across all five
arms -- EI, UCB and $\kappa\in\{1,2,5\}$ alike -- and on Hartmann-6
UCB's action path coincides with OR-LCB's on $5$ of $7$ seeds and its
regret trajectory on $6$ of $7$. The UCB-versus-OR-LCB action-identity
rate across the five benchmarks ($0/7$, $2/7$, $2/6$, $5/7$, $6/6$)
reproduces the same ordering with no $\kappa$ anywhere in it. The two
top rungs of Table~\ref{tab:Ckappa} are therefore saturated nulls -- as
this paper already said of Levy, and now says of Hartmann-6 -- and the
informative comparisons are Branin, Griewank and Ackley. The certified
width's \emph{level} carries no spatial information either:
$\mathrm{mean}(\hat c)$ at a given $t$ agrees across all five benchmarks
to $0.003$ at $t{=}0$ and $0.56$ at $t{=}17$. Only its oscillation can
act, and escape time governs that.

The regret trajectory behind Table~\ref{tab:Ckappa} is a running minimum
and can be bitwise identical while the arms acquire
different points -- that happens in $10$ of $37$ action-divergent Hartmann-6
arm pairs ($27\%$), and correcting for it moves four cells of that table
(Hartmann-6 $2{\leftrightarrow}5$ and $1{\leftrightarrow}5$ from $7/7$ to
$6/7$, Griewank $1{\leftrightarrow}2$ from $3/7$ to $1/7$ and
$1{\leftrightarrow}5$ from $1/7$ to $0/7$) and two of the full-ladder column.
Its Clopper--Pearson $95\%$ intervals depend only on $k/n$:
$7/7\,[0.590,1]$, $6/6\,[0.541,1]$, $6/7\,[0.421,0.996]$,
$3/6\,[0.118,0.882]$, $3/7\,[0.099,0.816]$, $1/7\,[0.004,0.579]$,
$0/7\,[0,0.410]$.

The measured radius dispersions $\mathrm{osc}(\hat c)/\mathrm{mean}(\hat c)$
that order these benchmarks, and the normalized candidate-to-data distances
behind them, are reported with Lem.~\ref{lem:flat} in
App.~\ref{app:proofs}.

\begin{table}[t]\centering\footnotesize
\caption{$\kappa$-inertness, measured at the \emph{action} level: a
$\kappa$ pair is inert on a seed iff the two arms select the same
candidate at every one of the $18$ BO iterations. We measure the action
path and not the regret trajectory because the regret trajectory is a
running minimum and can be bitwise identical while the arms acquire
different points. The independent unit is the seed;
the three $\kappa$ pairs within a seed are \emph{not} independent, since
identity is an equivalence relation, so we do not pool them, and
Clopper--Pearson $95\%$ intervals depend only on $k/n$ ($7/7\,[0.590,1]$ to
$0/7\,[0,0.410]$). The dispersion column is the
median over BO steps of $\mathrm{osc}(\hat c)/\mathrm{mean}(\hat c)$; on
the top two rows the radius minimizer is pinned to the lower end of the
certifier's $\mu$ grid and the recorded value is an $O(\mu_{\mathrm{lo}})$
artifact of that grid rather than a property of the objective
(Rem.~\ref{rem:gridfloor}), which is why it is written as a bound. The escape
column is the first BO iteration at which
$\mathrm{osc}(\hat c)/\mathrm{mean}(\hat c)$ exceeds $10^{-6}$, and the first
$\kappa$ divergence always follows it.}
\label{tab:Ckappa}
\setlength{\tabcolsep}{5pt}
\begin{tabular}{@{}lccccccc@{}}
\toprule
& & \multicolumn{3}{c}{action-inert seeds} & full & escape & UCB\\
\cmidrule(lr){3-5}
Benchmark & $\mathrm{osc}(\hat c)/\mathrm{mean}(\hat c)$ & $1{\leftrightarrow}2$ & $2{\leftrightarrow}5$ & $1{\leftrightarrow}5$ & ladder & step & $=$OR-LCB\\
\midrule
Levy-10D & below the $\mu$-grid floor & $6/6$ & $6/6$ & $6/6$ & $6/6$ & never & $6/6$\\
Hartmann-6 & below the $\mu$-grid floor & $7/7$ & $6/7$ & $6/7$ & $6/7$ & $9$ & $5/7$\\
Ackley-5D & $0.0066$ & $6/6$ & $3/6$ & $3/6$ & $3/6$ & $6$ & $2/6$\\
Griewank-5D & $0.0110$ & $1/7$ & $3/7$ & $0/7$ & $0/7$ & $5$ & $2/7$\\
Branin-2D & $0.3263$ & $0/7$ & $0/7$ & $0/7$ & $0/7$ & $0$ & $0/7$\\
\bottomrule
\end{tabular}
\end{table}

\subsection{The invariance criterion tested step by step}

\begin{table}[t]\centering\footnotesize
\caption{Prop.~\ref{prop:inv} tested directly. Its criterion
$|\Delta\kappa|\,\mathrm{osc}(c)<\Delta_\kappa$ was evaluated and compared with
what the acquisition actually did, at every BO iteration of every arm and seed
(left) and over whole trajectories (right). The proposition is \emph{sufficient},
so the boldface cell -- criterion fires, argmin moves -- is the only cell that
could refute it; it is empty in both tests. The criterion is conservative rather
than tight: it fires on $65.8\%$ of steps and covers $84.7\%$ of the steps on
which invariance actually occurred.}
\label{tab:Cinv}
\setlength{\tabcolsep}{6pt}
\begin{tabular}{@{}lrrrr@{}}
\toprule
& \multicolumn{2}{c}{one step ($n{=}9{,}018$ records)}
& \multicolumn{2}{c}{whole trajectory ($n{=}99$ pairs)}\\
\cmidrule(lr){2-3}\cmidrule(lr){4-5}
& argmin coincided & argmin moved & bitwise identical & differ\\
\midrule
criterion fires & $5{,}936$ & $\mathbf{0}$ & $35$ & $\mathbf{0}$\\
criterion silent & $1{,}076$ & $2{,}006$ & $23$ & $41$\\
\bottomrule
\end{tabular}
\end{table}

\textbf{The proposition is a validated predictor, and this is the part
that survives everything} (Table~\ref{tab:Cinv}, Fig.~\ref{fig:Cinv}).
At every BO iteration of every arm and seed we recorded the criterion
$|\kappa'-\kappa|\,\mathrm{osc}(c)<\Delta_\kappa$ of
Prop.~\ref{prop:inv} alongside whether the argmin actually coincided:
$9{,}018$ (benchmark, arm, seed, step, $\kappa$-pair) records.
Prop.~\ref{prop:inv} is sufficient, so ``criterion fires, argmin moves''
is the only cell that could refute it; that count is $0$ of $9{,}018$ at
the step level, $0$ of $99$ at the trajectory level, and $0$ within every
benchmark and every $\kappa$ pair separately. Three ways it could be
spurious, checked and excluded. It is not a tie artifact: the minimum
runner-up gap anywhere in the corpus is $1.02\!\cdot\!10^{-7}$ and no
record has a gap below $10^{-9}$. It is not vacuous globally: the largest
$|\Delta\kappa|\,\mathrm{osc}(c)/\Delta_\kappa$ on a firing record is
$0.998$, with $46$ records above $0.9$ and five above $0.99$, so the
criterion came within $0.2\%$ of binding and did not fail. And it is not
an artifact of the identity measure, because it is tested on the argmin
directly rather than on the regret trajectory. What it is, is
conservative. It fires on $65.8\%$ of steps and covers $84.7\%$ of the
steps on which invariance in fact occurred; above its threshold the
argmin still coincided on $1{,}076$ of $3{,}082$ steps ($34.9\%$,
Clopper--Pearson $[0.332,0.366]$), and on Ackley specifically on $380$ of
$436$ ($87\%$). Binned by the margin ratio
$r=|\Delta\kappa|\,\mathrm{osc}(c)/\Delta_\kappa$ the observed invariance
rate is exactly $1.000$ in every bin with $r<1$ -- $5{,}936/5{,}936$,
including $337/337$ in $[0.5,1)$ -- and then falls to $0.83$, $0.80$,
$0.82$, $0.18$ and $0.023$ in $[1,1.5)$, $[1.5,2.5)$, $[2.5,10)$,
$[10,10^3)$ and above $10^3$. A usable operational rule follows: when the
criterion fires on at least $80\%$ of the $18$ BO steps the trajectory
was bitwise identical in $48/48$ cells (Clopper--Pearson
$[0.926,1.000]$); when it fires on fewer than half, in $0/28$. Two scope
limits. The criterion is vacuous exactly on the saturated benchmarks, and
vacuous by \emph{always} firing rather than never: it fires on
$1620/1620$ Levy records and on $1863/1890$ Hartmann-6 records, at
inverse margin ratios whose medians are $1.7\!\cdot\!10^{10}$ and
$1.1\!\cdot\!10^{9}$, and the margin ratio
$|\Delta\kappa|\,\mathrm{osc}(c)/\Delta_\kappa$ never even reaches
$0.5$ on Levy ($0$ of $540$ steps) against $18$ of $630$ on Hartmann-6.
Those two benchmarks therefore supply $3{,}483$ of the $5{,}936$ firings
($59\%$) and none of the information; the criterion's non-trivial
exercise is the remaining $2{,}453$ firings on Branin ($55$), Griewank
($1{,}106$) and Ackley ($1{,}292$). And the implementation anchors the runner-up gap at
the smaller $\kappa$ only, while the proposition is symmetric;
anchoring at either endpoint recovers $43$ of the $1{,}076$ misses and
still yields zero counterexamples. We report this as a floating-point
validation of a proved sufficient condition, plus a measurement of how
conservative that condition is, and not as evidence for a two-line
proof.

\textbf{What this leaves.} The certified floor anchors validity, and a
checkable margin condition says exactly when scaling it can do nothing:
that condition held $5{,}936$ times in this corpus without a single
counterexample, and it is the practical content of Rem.~\ref{rem:floor}.
Useful exploration schedules must reshape the width's geometry -- through
the nugget path $\nu(x)$ or hybrid distance terms -- rather than rescale
its magnitude, which is why an intentionally loose fixed-nugget bound of
an earlier revision reached $4.40$ on Griewank where every certified rung
reaches $9.9$--$12.4$: its advantage came from its accidental
$x$-dependence, its shape, not its scale. Two honest caveats on the
identity results themselves. Case C contains no within-session
determinism control -- no arm is re-run against itself -- so bitwise
identity is conditional on the GP fit being reproducible across calls,
which we did not independently verify; the argmin margins along inert
runs are large enough that a deterministic backend change could not flip
a pick -- over the $15$ fully inert runs the smallest runner-up gap
anywhere along a run is $1.3\!\cdot\!10^{-4}$ and the per-run minimum has
median $2.0\!\cdot\!10^{-3}$, against acquisition values of order
$1$ -- but non-determinism within a backend is not
excluded. And identity is budget-contingent: two cells of
Table~\ref{tab:Ckappa} diverge only at $t{=}15$ (Hartmann-6
$2{\leftrightarrow}5$ and $1{\leftrightarrow}5$, both on seed $2$), and a
$17$-iteration budget would move two further Griewank cells
($1{\leftrightarrow}2$ from $1/7$ to $4/7$ and $1{\leftrightarrow}5$ from
$0/7$ to $1/7$), so the table is budget-specific. The
$50$-evaluation deep-budget probe that would settle where the grading
converges was not run at this preset.

\begin{figure}[t]\centering
\includegraphics[width=\textwidth]{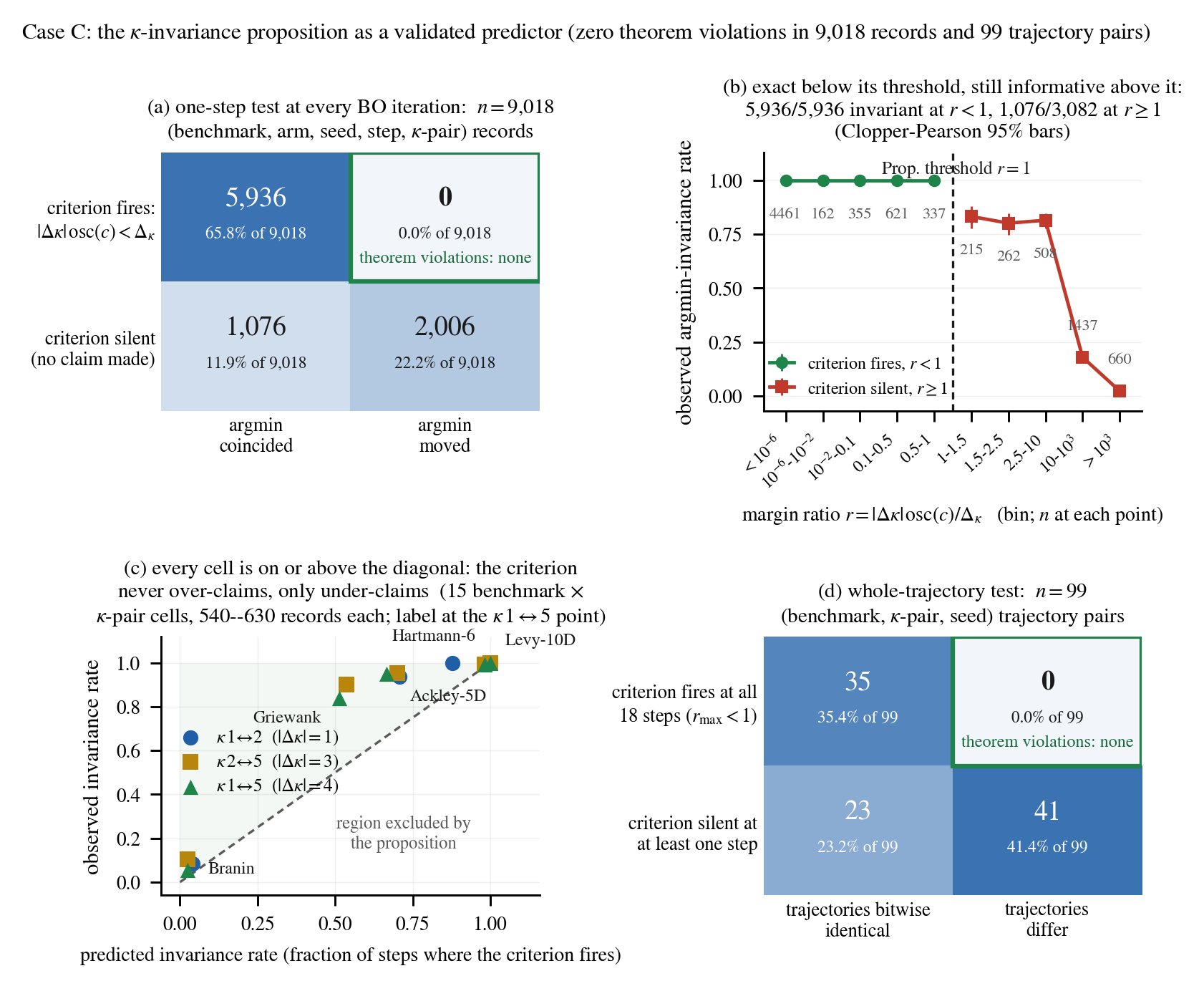}
\caption{Prop.~\ref{prop:inv} tested against what the acquisition
actually did. (a) The one-step confusion table over every BO iteration of
every arm and seed; the outlined cell is the only one that could refute a
sufficient condition and it is empty. (b) Observed argmin-invariance
against the margin ratio
$r=|\Delta\kappa|\,\mathrm{osc}(c)/\Delta_\kappa$: exactly $1.000$ in
every bin below the threshold and still $0.80$--$0.83$ just above it,
which is how conservative the bound is. (c) Every
benchmark-by-$\kappa$-pair cell lies on or above the diagonal: the
criterion never over-claims, only under-claims. (d) The same test at
whole-trajectory resolution.}
\label{fig:Cinv}
\end{figure}

\end{document}